%% file: main.tex
\documentclass{article}

\input{myheaders.tex}

\title{Iterative Preference Learning from Human Feedback: Bridging Theory and Practice for RLHF under KL-constraint}
\author{
Wei Xiong$^1$ \thanks{The first three authors contribute equally to this work. \\$^1$University of Illinois Urbana-Champaign. $^2$Salesforce AI Research. $^3$The Hong Kong University of Science and Technology. $^4$ Work done during an internship at The Hong Kong University of Science and Technology. Correspondence to wx13@illinois.edu, tongzhang@tongzhang-ml.org} \qquad 
Hanze Dong$^2$ \qquad
Chenlu Ye$^3$ \qquad
Ziqi Wang$^1$ \qquad Han Zhong$^4$ \\\\
Heng Ji$^1$ \qquad
Nan Jiang$^1$ \qquad
Tong Zhang$^1$\\
}
\date{}
\usepackage{cancel}
\newcommand{\hmm}{\textbf{Human: }}     
\newcommand{\assi}{\textbf{Assistant: }}
\usepackage[skins]{tcolorbox} 
\tcbuselibrary{breakable} 
\newtcolorbox[auto counter, number within=section, list type=subsubsection, list inside=toc]{sectionbox}[2][]{
colback=white!98!gray, colframe=black, 
colbacktitle=white!90!gray, coltitle=black, 
fonttitle=\bfseries,
title={#2}, 
list entry={Comment \thetcbcounter\quad}
}
\usepackage[utf8]{inputenc}

\begin{document}

\maketitle

\begin{abstract}
This paper studies the alignment process of generative models with Reinforcement Learning from Human Feedback (RLHF). We first identify the primary challenges of existing popular methods like offline PPO and offline DPO as lacking in strategical exploration of the environment. Then, to understand the mathematical principle of RLHF, we consider a standard mathematical formulation, the reverse-KL regularized contextual bandit for RLHF. 
Despite its widespread practical application, a rigorous theoretical analysis of this formulation remains open.
We investigate its behavior in three distinct settings---offline, online, and hybrid---and propose efficient algorithms with finite-sample theoretical guarantees. 

Moving towards practical applications, 
our framework, with a robust approximation of the information-theoretical policy improvement oracle, naturally gives rise to several novel RLHF algorithms. This includes an iterative version of the Direct Preference Optimization (DPO) algorithm for online settings, and a multi-step rejection sampling strategy for offline scenarios. Our empirical evaluations on real-world alignment experiment of large language model demonstrate that these proposed methods significantly surpass existing strong baselines, such as DPO and Rejection Sampling Optimization (RSO), showcasing the connections between solid theoretical foundations and their potent practical implementations.

\end{abstract}

\setlength{\parindent}{0pt}
\setlength{\parskip}{8pt}

\tableofcontents

\section{Introduction} \label{sec:intro}

\textit{Reinforcement Learning from Human Feedback} (RLHF) \citep{christiano2017deep, ziegler2019fine} has emerged as a powerful paradigm to align modern 
generative models like Large Language Models (LLMs) and diffusion models with human values and preferences. 
This approach has shown significant effectiveness in applications such as
 ChatGPT \citep{OpenAI2023GPT4TR}, Claude \citep{Anthropic@claude}, Bard \citep{google@bard}, and LLaMA2 \citep{touvron2023llama}, by making the built AI system helpful, harmless, honest and controllable \citep{ouyang2022training, bai2022training}. 
 
 Despite its effectiveness, RLHF's implementation often involves ad-hoc practices and extensive algorithmic tuning in the entire pipeline, including preference data collection (it is hard to select representative humans~\citep{bai2022training}, larger language models~\citep{mint2024a} or program compiler~\citep{nlfeedback2023}), preference/reward modeling (reward misspecification and misgeneralization \citep{hong2022sensitivity, gao2023scaling}), and model optimization (instability of training \citep{choshen2019weaknesses} and distribution shift issue \citep{michaud2020understanding, tien2022causal}). Meanwhile, the resulting models of RLHF typically suffer from issues like performance degeneration if we impose strong optimization pressure toward an imperfect reward function \citep{michaud2020understanding, tien2022causal, gao2023scaling}, which contains bias and approximation error from the data collection and preference modeling \citep{gao2023scaling, wang2023enable}.  \citet{casper2023open}
 also discussed many other challenges of RLHF. Thus, it is important to understand the mathematical principle of the RLHF process, as well as the connections among its different steps, which should be able to motivate future algorithmic design in principle.

In current RLHF theory, the agent's objective is to maximize an observed reward function, with the optimal policy typically being deterministic and reward-greedy \citep{agarwal2019reinforcement}. However, in practical RLHF applications, merely maximizing the reward function is often insufficient and probably results in overfitting, as the generative model must simultaneously ensure both diversity and high fidelity in its outputs. A deterministic maximizer of the reward tends to compromise on these aspects significantly. For example, the maximizer of the ``safety reward'' tends to avoid providing answers all the time, which contradicts the LLM's training objective. The situation worsens due to bias and approximation errors in reward modeling, leading to the critical problem of reward hacking, where the model often repeats superfluous, pleasing yet irrelevant words to appease the reward model \citep{michaud2020understanding, tien2022causal, casper2023open}. Thus, it is important to model diversity and high fidelity in the theoretical framework beyond the reward. Notably, the most widely used mathematical objective function for this goal can be regarded as a reverse-KL regularized contextual bandit problem \citep{ziegler2019fine, wu2021recursively, ouyang2022training, rafailov2023direct,liu2023statistical}. The KL regularized contextual bandit additionally imposes a constraint that the optimal policy cannot move too far away from the original policy (i.e. the starting checkpoint of the LLM). 
A major difference between this objective function from traditional contextual bandit \citep{langford2007epoch} is that the optimal policy is stochastic, which is closer to the practical generative models. See an intuitive illustration why such a target is appealing in Figure \ref{fig:preference}. Despite numerous proposed procedures for this formulation, a rigorous theoretical analysis remains open. This paper provides a theoretical analysis of the regularized contextual bandit problem in both offline and online settings, aiming to inform and motivate practical algorithmic designs. We summarize the contributions and take-away messages of this work as follows:
\begin{itemize}
\item We identify the challenges of existing RLHF methods in Section~\ref{sec:formu}: they require a preference dataset with uniform coverage over the entire prompt-response space to converge to the optimal policy, which is extremely hard to satisfy in practice due to the exponentially large response space;
\item To understand the mathematical principle of RLHF, we first \emph{formally} formulate the RLHF process as a reverse-KL regularized contextual bandit problem in RLHF theory in Section~\ref{sec:formu}, which more accurately reflects real-world alignment practices \citep{ouyang2022training, bai2022training, rafailov2023direct} compared to existing theoretical frameworks. Meanwhile, we deliver a comprehensive theoretical analysis in offline, online, and hybrid settings for the formulated framework, where the three settings are complementary to each other and hold their own values in practical applications.
\begin{itemize}
    \item For offline learning from a fixed dataset, we show that RLHF with pessimism (a conservative reward estimation) is sample efficient in Section~\ref{sec:offline}. Moreover, we derive the way of implementing the pessimism with both the PPO and DPO;
    \item For the online/hybrid learning, we discuss how to strategically explore and continuously enhance the policy by online interactions with the humans in Section~\ref{sec:hybrid}. Specifically, by establishing the theoretical guarantee of the proposed algorithms, we show that the RLHF can benefit from online exploration. 
\end{itemize}
\item Moving towards practical applications, in Section~\ref{sec:discuss}, we demonstrate that the proposed frameworks can be practically implemented when combined with existing planning algorithms like offline PPO, offline DPO, and InfoNCA, that optimize against a fixed reward function. In other words, our framework can be viewed as built on the top of existing methods, and serve as to boost their performance;
\item In Section~\ref{sec:exp}, we demonstrate that the proposed algorithms empirically outperform existing strong baselines like DPO \citep{rafailov2023direct} and RSO \citep{liu2023statistical} in real-world LLM experiments. In particular, with the Zephyr-SFT-7B as the initial model, the resulting aligned policy enjoys an impressive win-rate $34.79\%$ in the AlpacaEval2 benchmark that beats many larger LLMs.
\end{itemize}

\subsection{Related Work}
There is a rich literature in RLHF and we refer the interested readers to the survey papers like \citet{casper2023open} for a more comprehensive review. We focus on the papers that are most related to our work here.


\textbf{RLHF} has attracted considerable attention in the past few years, especially after its tremendous success in ChatGPT \citep{OpenAI2023GPT4TR}. We refer interested readers to \citet{wirth2017survey, casper2023open} for a detailed survey but focus on the most related works here. The standard RLHF was popularized by \citet{christiano2017deep}, which served to direct the attention of the RL community to the preference-based feedback. The most popular and standard RLHF framework is outlined in the InstructGPT paper \citep{ouyang2022training}, Claude \citep{bai2022training} and the LLaMA2 report \citep{touvron2023llama} in detail, which typically consists of three steps starting from the pretrained model: supervised finetuning, reward modeling, and reward optimization. The effectiveness of this framework has been showcased by many recent generative models, like ChatGPT \citep{OpenAI2023GPT4TR}, Bard \citep{google@bard}, Claude \citep{Anthropic@claude}, and LLaMA2 \citep{touvron2023llama}. However, it is also noteworthy to indicate that the RLHF process often leads to degeneration in the performance of generation, commonly referred to as the “alignment tax” in the literature \citep{askell2021general}. This is usually because of the imperfection of the reward model and the model can make use of these imperfections to chase for a high reward. This phenomenon is referred to as the reward hacking \citep{michaud2020understanding, tien2022causal}. It is also possible to apply RLHF to general generative models, like the diffusion model \citep{hao2022optimizing, wu2023better, lee2023aligning, dong2023raft}. In this work, we use the terminology and analysis of LLMs for better illustration, and defer the study of general generative models to future work.

\textbf{RLHF algorithms.}  Proximal Policy Optimization (PPO) \citep{schulman2017proximal} is the most well-known algorithm in LLM alignment literature. However, its instability, inefficiency, and sensitivity to hyperparameters \citep{choshen2019weaknesses} and code-level optimizations \citep{engstrom2020implementation} present significant challenges in tuning for optimal performance and its tremendous success in Chat-GPT4 \citep{OpenAI2023GPT4TR} has not been widely reproduced so far. Additionally, it often necessitates incorporating an extra reward model, a value network (known as a critic), and a reference model, potentially as large as the aligned LLM \citep{ouyang2022training, touvron2023llama}. This imposes a significant demand on GPU memory resources. Thus, researchers have attempted to design alternative approaches for LLM alignment to resolve the aforementioned issues. \citet{dong2023raft, yuan2023rrhf, touvron2023llama, gulcehre2023reinforced} propose reward ranked finetuning (RAFT) (also known as the iterative finetuning, rejection sampling finetuning) by iteratively learning from the best-of-n policy \citep{nakano2021webgpt} to maximize the reward, which is a stable baseline with minimal hyper-parameter configuration and was applied to the alignment of LLaMA2 project. There is also a line of work focusing on deriving an algorithm from the KL-regularized formulation \citep{rafailov2023direct, zhu2023fine, wang2023beyond, liu2023statistical, li2023remax}. Among them, Direct Preference Optimization (DPO) \citep{rafailov2023direct} has emerged as an attractive alternative approach to PPO with notable stability and competitive performance. The innovative idea of DPO is to train the LLMs directly as a reward model based on the offline preference dataset and bypassing the reward modeling. Similar to DPO, there are also other works aiming to optimize the LLMs directly from the preference data, including \citep{zhao2023slic, azar2023general}, and has sparked considerable debate on whether reward modeling, as well as RL, is necessary for alignment. However, while these algorithms are partly inspired by mathematical principles and intuitions,  a comprehensive theoretical analysis remains open. 

\textbf{Theoretical study of RLHF.} 
The theoretical understanding of RLHF can be traced back to research on dueling bandits \citep[e.g.,][]{yue2012k,saha2021optimal,bengs2021preference}, a simplified setting within the RLHF framework. Recently, many works have focused on the more challenging RLHF problem (also known as the preference-based RL). \citet{xu2020preference,novoseller2020dueling,pacchiano2021dueling} delve into the study of tabular online RLHF, where the state space is finite and small. Moving beyond the tabular setting, \citet{chen2022human} provides the first results for online RLHF with general function approximation, capturing real-world problems with large state spaces. \citet{wang2023rlhf} presents a reduction-based framework, which transforms some sample-efficient algorithms for standard reward-based RL to efficient algorithms for online RLHF. Further advancements in algorithm designs are introduced by \citet{zhan2023query,wu2023making}, encompassing the development of reward-free learning type algorithms and posterior sampling-based algorithms tailored for online RLHF. Initiating exploration into offline RLHF, \citet{zhu2023principled} presents a pessimistic algorithm that is provably efficient for offline RLHF. Additionally, \citet{zhan2023provable} and \citet{li2023reinforcement} extend these investigations into the broader scope of general function approximation settings within offline RLHF. In comparison to these existing studies, our work introduces a new theoretical formulation and goal for RLHF, as well as novel problem settings, such as hybrid RLHF. The new mathematical formulation allows our framework to align more closely with recent advancements in LLMs, and we discuss the connections between our theoretical findings and practical algorithmic designs in Section~\ref{sec:discuss}. We mention in passing that \citet{tiapkin2023regularized} also considers the KL constraint in offline RL but mainly focuses on the scenario where an expert policy is available for imitation learning, thus differing from ours. 


Finally, concurrent to this work, \citet{snorkelai@pair} and \citet{yuan2024self} consider variants of iterative DPO that may share similar insights with us in terms of algorithmic design. We comment on the similarities and differences between our work and theirs as follows. \citet{snorkelai@pair} focus on the batch online setting, which will be thoroughly developed in Section~\ref{sec:hybrid}. One notable difference is that they set the reference policy as the one from last iteration, while we always use the $\pi_0$ as the reference policy. From a theoretical perspective, their algorithmic design resembles the classic policy gradient algorithm \citep{cai2020provably, zhong2023theoretical} that optimizes the non-regularized reward, while we optimize the KL-regularized one as most of the Instruct-GPT \citep{ouyang2022training} and Claude \citep{bai2022training} did. \citet{yuan2024self} also consider iterative DPO-type training. However, both our algorithm and \citet{snorkelai@pair} leverage the reward signal from the external model or human, while \citet{yuan2024self} adopts a clever idea by using the LLM itself as the reward model to provide preference signal, hence the name “self-rewarding”. We remark that the primary goal of this project is to formally formulate the RLHF as a KL-regularized contextual bandit problem and establish its mathematical foundation. The online iterative DPO is a natural corollary of the established framework but the framework can also be implemented by combining it with other oracle algorithms. See Section~\ref{sec:discuss} for details. Finally, we expect that the techniques presented in this paper also extend to analyze the general preference learning like \citet{azar2023general} beyond the reward-based learning.

\section{Formulation and Existing Approaches} \label{sec:formu}

In this section, we present the mathematical framework for the RLHF process, inspired by the standard LLM alignment workflow \citep{ouyang2022training, touvron2023llama}. 

\subsection{Formulation of RLHF}
Specifically, the LLM can take a prompt, denoted by $x \in \cX$, and produce a response, denoted by $a=[w_1,w_2, \ldots]$, where $w_i$ is the $i$-th token generated by the model. Accordingly, we can take $\cX$ as the state space of the contextual bandit and the $\cA$ as the action space. Following \citet{ouyang2022training, zhu2023principled, rafailov2023direct, liu2023statistical}, we assume that there exists a ground-truth reward function $r^*(x,a): \cX\times\cA \to [0,1]$ and the preference satisfies the Bradley-Terry model \citep{bradley1952rank}: 
\begin{equation} \label{eqn:bt} 
\begin{aligned}
        \PP(a^1 \succ a^2|x,a^1,a^2)
        =  \frac{\exp(r^*(x,a^1))}{\exp(r^*(x,a^1)) + \exp(r^*(x,a^2))}
 = \sigma\big(r^*(x,a^1)-r^*(x,a^2)\big),
 \end{aligned}
\end{equation}
where $a^1\succ a^2$ means that $a^1$ is preferred to $a^2$, and 
$\sigma(z) = 1/(1+\exp(-z))$ is the sigmoid function. We denote an LLM by a policy $\pi$ that maps $x$ to a distribution over $\cA$. 

In a typical LLM training pipeline, the tuning process begins with a pretrained LLM, which is subsequently fine-tuned using specialized and instructional data, yielding an initial LLM policy denoted as $\pi_0$. We will then align the LLM on RLHF data (prompt set), which we assume is taken from a distribution $x \sim d_0$. For preference learning, the way to gather information from the environment is to compare two different actions under the same state. Considering this, we assume that the agent can perform a pair of actions, aligning with precedents in existing literature \citep{novoseller2020dueling, pacchiano2021dueling}. In applications, we want the resulting LLM $\pi$ to be close to $\pi_0$, and our goal is to find a policy $\pi$ from some policy class $\Pi$ to maximize
\begin{equation} \label{eqn:target}
\begin{aligned}
    J(\pi)=\E_{x \sim d_0} \E_{a \sim \pi(\cdot|x)}\left[ r^*(x,a) + \eta \log \frac{\pi_0(a|x)}{\pi(a|x)}\right]
    = \E_{x \sim d_0} \left[ \E_{a \sim \pi(\cdot|x)}[r^*(x,a)] - \eta \KL(\pi(\cdot|x)\Vert \pi_0(\cdot|x)) \right],
\end{aligned}
\end{equation}
where $\eta>0$ is the KL penalty coefficient. This formulation is widely studied in practice \citep{ziegler2019fine, wu2021recursively, ouyang2022training, rafailov2023direct, liu2023statistical}, and our paper aims to study its theoretical property.

\begin{figure*}[t]
    \centering
   {\begin{footnotesize}
    \begin{tabular}{cccc}
          \includegraphics[width=3.5cm]{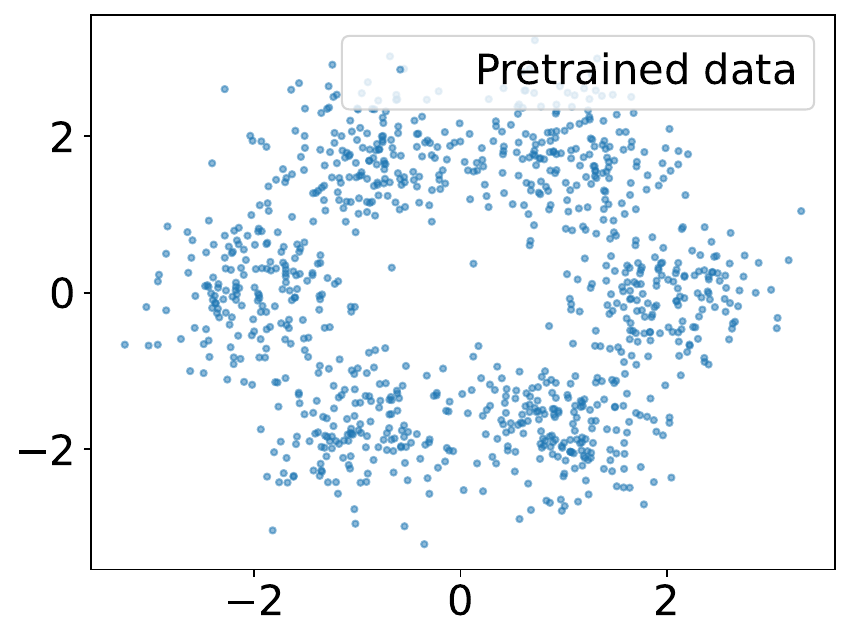}   &
          \includegraphics[width=3.5cm]{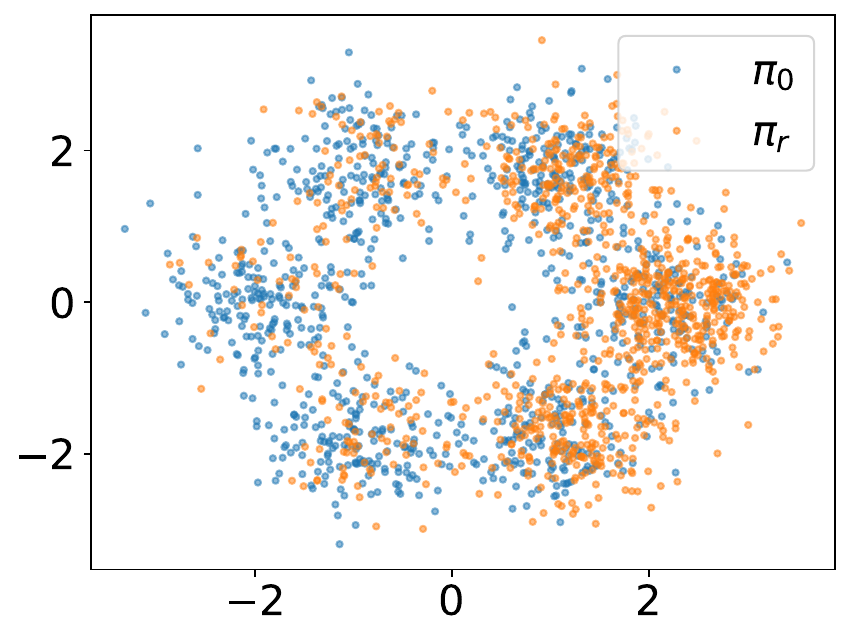} &
          \includegraphics[width=3.5cm]{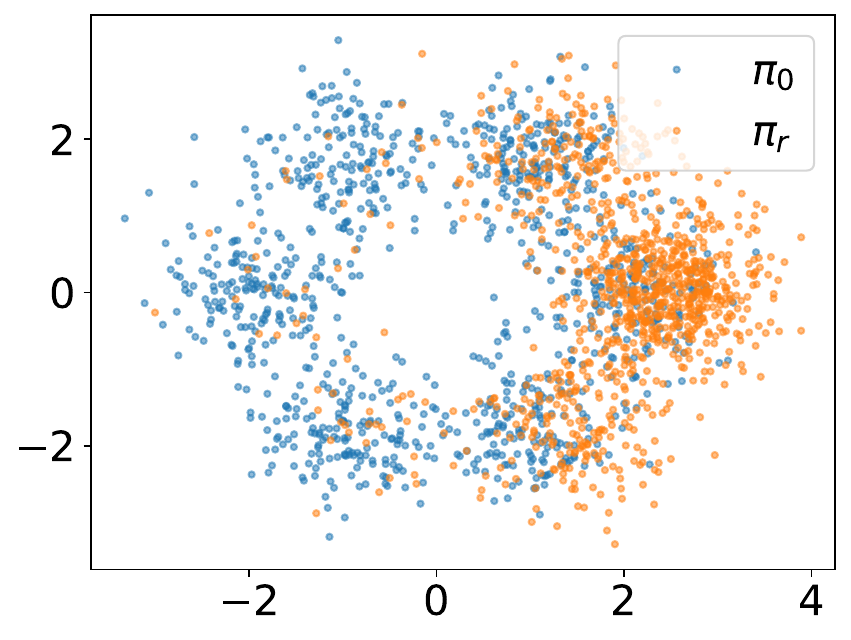} &
          \includegraphics[width=3.5cm]{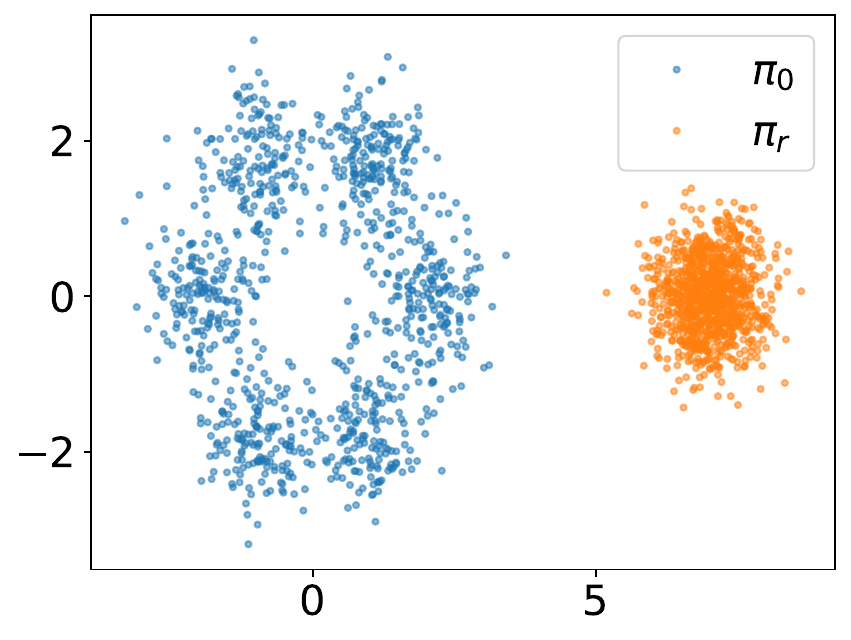} 
          \\
      (a) Pretrained $\pi_0$  &  (b) Preferred $\pi_r$ ($\eta^{-1}=0.5$) & (c) Preferred $\pi_r$ ($\eta^{-1}=1.0$) & (d) Preferred $\pi_r$ ($\eta^{-1}=10$)
    \end{tabular}
     \end{footnotesize}}
    \caption{ A two-dimensional illustrating example of human preference in generative modeling.
    We consider a scenario where the initial data distribution $\pi_0$, referred to as ``pretrained", is represented by a multi-modal Gaussian mixture, reflecting real-world data complexities. The ``human preference" is expressed as a bias towards the right, as we set \(r=[1,0]^\top a\). The KL penalty is critical in maintaining the desired behavior of \(\pi_r\). By varying the values of \(\eta\), we demonstrate the impact of KL regularization in (b)-(d). As \(\eta\) approaches zero, \(\pi_r\) increasingly focuses on maximizing rewards, often at the expense of the pretrained data's structure, leading to a Gibbs distribution that potentially diverges infinitely. 
    }
    \label{fig:preference}
\end{figure*}

Usually, we have a function class $\cF$ for approximating the ground truth $r^*$. Following \citet{pacchiano2021dueling,kong2022provably, zhu2023principled}, we make the following assumption for a clear presentation because it suffices to illustrate our ideas and the algorithmic design in this paper can also apply to the general case. The analysis also readily generalizes to general function class using standard complexity measures in RL theory literature \citep{russo2013eluder, gentile2022fast}, which essentially state that there are some low-rank structures in reward model. 

\begin{assumption} \label{assu:linear}
    Assume that the reward function is parameterized by $r_\theta(x,a) = \dotprod{\theta, \phi(x,a)}$ for feature extractor $\phi:\cX \times \cA \to \RR^d$. We also assume that $r^*(x,a) = \dotprod{\theta^*, \phi(x,a)}$ for some $\theta^* \in \RR^d$. For regularization, we assume that $\norm{\phi(x,a)} \leq 1$ for all possible $(x,a) \in \cX \times \cA$ and $\norm{\theta} \leq B$. We also denote $\gamma = 1/(2+\exp(-B) + \exp(B))$.
\end{assumption}

\textbf{Notation.} We use $\norm{z}_{\Sigma}$ to denote the induced norm $\sqrt{z^\top \Sigma z}$ for some positive-definite matrix. We also define $\phi(x, \pi) := \EE_{a \sim \pi(\cdot|x)} \phi(x,a)$ to simplify the presentation. We use $\tilde O$ when we omit the logarithmic factors. A notation table is provided in Table~\ref{tab:notation} to improve the readability of this paper.

\subsection{The Insufficiency of Classic Frameworks}
The classic RLHF framework adopted by \citet{ziegler2019fine, ouyang2022training} can be divided into two stages: 1) reward modeling, and 2) policy optimization against the learned reward. We summarize the details of this method in this subsection.

\textbf{Maximum Likelihood Estimation for reward modeling.} Given a preference dataset $\cD$ consists of numerous tuples, such as $(x, a^1, a^2, y)$, where $y$ is the preference signal. Specifically, $y=1$ means a preference for $a^1 \succ a^2$, while $y=0$ indicates $a^1 \prec a^2$. Given a dataset $\cD = \{(x,a^1,a^2,y)\}$, we can write the log-likelihood function of the BT models as follows:
\begin{equation}
    \label{eqn:bt_likelihood}
    \begin{aligned}
        \ell_{\cD}(\theta) = \sum_{(x,a^1,a^2,y) \in \cD} \Big[y \log \Big(\sigma\big(r_{\theta}(x,a^1) - r_{\theta}(x,a^2)\big)\Big)  + (1-y) \log \Big(\sigma\big(r_{\theta}(x,a^2) - r_{\theta}(x,a^1)\big)\Big)\Big].
            \end{aligned}
\end{equation}
We can compute the maximum likelihood estimator (MLE) $r_{\mathrm{MLE}}$ based on $\cD$ as $\theta_{\mle} = \argmax_{\theta \in \Theta(B)} \ell_{\cD}(\theta)$ with $\Theta(B) = \{\theta \in \RR^d: \norm{\theta}\leq B\}$. In practice, the MLE is also conducted with the LLMs \citep{ ouyang2022training, bai2022training, touvron2023llama} on the preference dataset. 

\textbf{Policy Optimization and Oracle.} With the learned reward in hand, to approximately optimize the target given in \eqref{eqn:target}, we simply call the seminal DRL method PPO with the following regularized reward:
$$
\hat{r}(x,a) = r_{\mathrm{MLE}}(x,a) - \eta \log \frac{\pi(a|x)}{\pi_0(a|x)}.
$$
To simplify the discussion, we first omit the computational challenges by defining the following information-theoretical policy improvement oracle, and defer a discussion on its practical implementations in Section~\ref{sec:discuss}.
\begin{definition}[Policy Improvement Oracle] \label{assu:oracle1}
For reward function $r:\cX \times \cA \to \RR$ and a reference policy $\pi_0$, for all $x \in \cX$, we can compute the Gibbs policy (Lemma~\ref{lem:kl_solu}):
$$ 
\begin{aligned}
    \pi_r(\cdot|x) &:= \argmax_{\pi \in \Pi} \EE_{a \sim \pi(\cdot|x)} \Big[r(x,a) + \eta \log \frac{\pi_0(a|x)}{\pi(a|x)}\Big]\propto \pi_0(\cdot|x) \cdot \exp\big(\frac{1}{\eta} r(x,\cdot)\big).
\end{aligned}
$$
\end{definition}
Accordingly, we take the policy class as $
\Pi := \big\{\pi(\cdot|x) \propto \pi_0(\cdot|x) \cdot \exp\big(\frac{1}{\eta} \dotprod{\theta, \phi(x,\cdot)}\big): \theta \in \Theta(B)\big\}$. However, even in this ideal case without worrying about tuning the RL method to its best performance, this useful framework suffers from the reward over-optimization issue \citep{gao2023scaling}. 
\begin{center}
Intuitively, \textit{the finite $\cD$ cannot cover the whole prompt-response space, and the fine-tuned reward model often performs poorly in the face of out-of-distribution data \citep{burns2023weak}}.
\end{center}
Accordingly, the learned $r_{\mathrm{MLE}}$ only aligns well with the ground truth $r^*$ in certain distributions and we further illustrate this phenomena in Figure~\ref{fig:over_opt}.

\begin{figure}[htp]
    \centering
    \includegraphics[width=12cm]{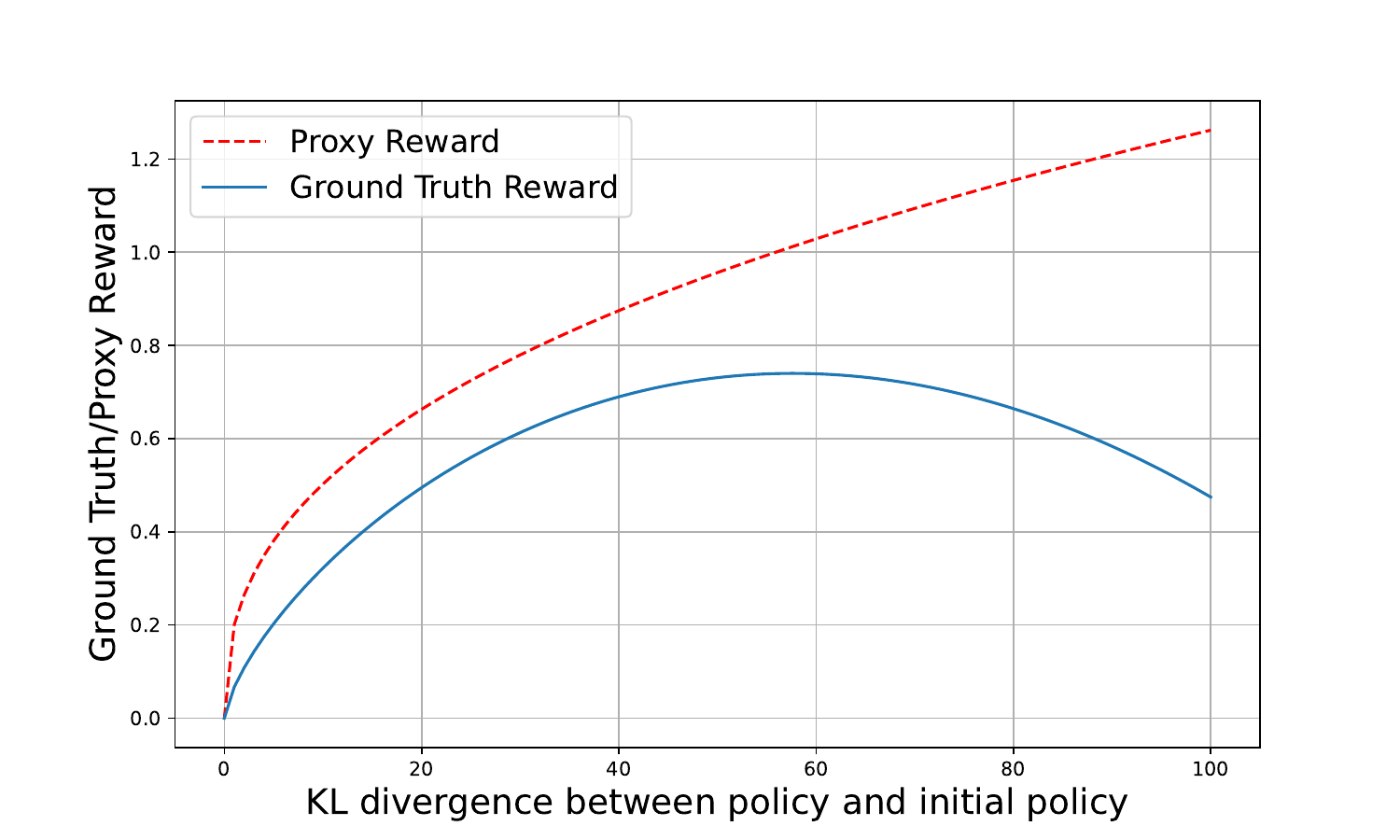}
    \caption{An illustration of the reward over-optimization adapted from \citet{gao2023scaling}. Here the proxy reward is trained from the responses of $\pi_0$. Therefore, in the early stage, the proxy reward aligns well with the gold reward in terms of the in-distribution responses. As reward gets higher, the distribution shift becomes larger, and since the training set is lacking in the coverage over these out-of-distribution responses, the proxy reward does not align with the gold reward in this stage.}
    \label{fig:over_opt}
\end{figure}

In addition to the PPO algorithm, the direct preference learning methods have attracted significant attention recently due to its stability and competitive performance \citep{zhao2023slic, rafailov2023direct, azar2023general, ethayarajh2024kto}. We use the DPO as a representative example and the intuition extends to other algorithms. We show that the DPO can be viewed as imposing constraints on the resulting policy by $\cD$. Informally, to converge to $\pi^*$, the DPO requires an infinite $\cD$ to cover the whole prompt-response space. We defer the discussion to Appendix~\ref{sec:appendix_discuss}.

In recognition of the above issues, the primary goal of this work is to \textcolor{red}{\textit{formally formulate the RLHF process, establish the mathematical foundation, and advance the practical algorithmic designs with the developed theoretical insights.}}

\subsection{Preliminary}
In this section, we present some useful technical tools and lemmas for subsequent analysis.  


\textbf{Value decomposition.} We have the following lemma to decompose the value difference.  
\begin{lemma}
\label{lem:decom}
Given a comparator policy $\pi$, we can decompose the suboptimality of $\hat{\pi}$ as follows:
$$
    \begin{aligned}
&J(\pi) - J(\hat{\pi})= \E_{x\sim d_0}\Big[{\EE_{\pi}[r^*(x, a) - \hat{r}(x, a)]} + {\EE_{\hat{\pi}}[\hat{r}(x, a) - r^*(x, a)]}+ \EE_{\pi}[\hat{r}(x, a)] - \E_{\hat{\pi}}[\hat{r}(x, a)] \\
&\qquad + \eta\KL(\hat{\pi}(\cdot|x)\|\pi_0(\cdot|x)) - \eta\KL(\pi(\cdot|x)\|\pi_0(\cdot|x))\Big],
    \end{aligned}
    $$
    where $\hat{r}: \cX \times \cA \to \RR$ is arbitrary.
\end{lemma}
\begin{proof}
    The equality can be verified directly by the definition of $J(\cdot)$ in \eqref{eqn:target} and basic algebra.
\end{proof}

\textbf{Policy improvement error.} In standard RL setting, $\hat{\pi}$ is typically taken as a greedy policy of $\hat{r}$, leading to 
$$
\EE_{\pi}[\hat{r}(x, a)] - \E_{\hat{\pi}}[\hat{r}(x, a)] \leq 0.
$$
In the KL-constrained case, since the policy cannot be greedy or deterministic, we need to additionally handle the policy improvement error. The following lemma provides such an estimation when our policy is obtained by calling the Oracle~\ref{assu:oracle1} with $\hat{r}$. 

\begin{lemma}[Policy optimization error] \label{lem:opt_error} 
Suppose that $\pi, \hat{\pi} \in \Pi$ so that $\pi_0, \pi, \hat{\pi}$ have the same support. If $\hat{\pi}$ is induced by calling Oracle~\ref{assu:oracle1} with $\hat{r}$, it holds that 
$$
\begin{aligned}
    &\EE_{x \sim d_0} \Big[\EE_{\pi}[\hat{r}(x, a)] - \E_{\hat{\pi}}[\hat{r}(x, a)] + \eta\KL(\hat{\pi}\|\pi_0) - \eta\KL(\pi\|\pi_0)\Big] = -\eta \EE_{x \sim d_0} \KL(\pi\|\hat{\pi}).
\end{aligned}
$$
Here $\KL(\pi \| \pi_0)$ is short for $\KL(\pi(\cdot|x)\| \pi_0(\cdot|x))$.
\end{lemma}
We will provide the proof of the lemma in Appendix~\ref{appendix:lemma_proof}. The analysis techniques are most similar to the policy gradient literature since they also consider the soft-max policies \citep{chen1993convergence, agarwal2021theory, cai2020provably, zanette2021provable,yuan2022linear, xiao2022convergence, zhong2023theoretical, uehara2024offline, alfano2024novel}. The main difference is that in their iterative choices of policy, for choosing $\pi_t$, the reference policy they use is the policy of the last round, i.e., $\pi_{t-1}$, while we always use the SFT-model $\pi_0$ as our reference. We note that their algorithms essentially still use the non-KL-regularized reward as the target because though we prevent the policy from moving too far away in each individual step, the cumulative updates makes the reward estimations dominating in the final policy.  

\textbf{Covariance matrix.} Given a preference dataset $\cD$, a fixed $\lambda > 0$, we denote $\Sigma_{\cD}$ as the covariance matrix estimation:
$$
\begin{aligned}
    \lambda I + \sum_{ (x,a^1,a^2)\in \cD} \big(\phi(x,a^1) - \phi(x,a^2)\big)\big(\phi(x,a^1) - \phi(x,a^2)\big)^\top.
\end{aligned}
$$
Both the algorithmic design and analysis will be centered on the covariance matrix. For the readers that are not familiar with the eluder-type techniques (or elliptical potential lemma in this case), 
we provide a brief introduction to the high-level intuition as follows.  Given a training set $\cD$, the \textit{in-sample} error on the observed data in $\cD$ is given by
$$
\norm{\theta_1 - \theta_2}_{\Sigma_{\cD}}^2 = \lambda \norm{\theta_1-\theta_2}^2 + \sum_{(x,a^1,a^2)\in \cD} \Big( \big(r_{\theta_1}(x,a^1) - r_{\theta_1}(x,a^2)\big) - \big(r_{\theta_2}(x,a^1) - r_{\theta_2}(x,a^2)\big)\Big)^2,
$$
where we additionally add a regularization term $\lambda \norm{\theta_1-\theta_2}^2$. Meanwhile, if we test the hypothesis $(\theta_1-\theta_2)$ on a newly observed data, the \textit{out-of-sample} error would be given by 
$
|\dotprod{\theta_1 - \theta_2, \phi(x,a^1) - \phi(x,a^2)}|.
$
The ideal case would be that we can infer the out-of-sample error via the in-sample error, so we look at the ratio between them:
$$
\frac{|\dotprod{\theta_1 - \theta_2, \phi(x,a^1) - \phi(x,a^2)}|}{\norm{\theta_1 - \theta_2}_{\Sigma_{\cD}}} \leq \frac{\norm{\phi(x,a^1) - \phi(x,a^2)}_{\Sigma_{\cD}^{-1}} \cdot \norm{\theta_1 - \theta_2}_{\Sigma_{\cD}}}{\norm{\theta_1 - \theta_2}_{\Sigma_{\cD}}} = \norm{\phi(x,a^1) - \phi(x,a^2)}_{\Sigma_{\cD}^{-1}},
$$
where we take a square root on the in-sample error to keep them being of the same order and use Cauchy-Schwarz inequality (Lemma~\ref{lem:cs_ineq}). Here, the $\norm{\phi(x,a^1) - \phi(x,a^2)}_{\Sigma_{\cD}^{-1}}$ is referred to as the elliptical potential in the literature of linear function approximation \citep{abbasi2011improved}. The elliptical potential can be viewed as the uncertainty of $\phi(x,a^1) - \phi(x,a^2)$, given the historical samples in $\cD$, and can be used to guide our exploration. The complexity of the reward model space is characterized by the following fact:
\begin{lemma}[Elliptical potential is usually small \citep{hu2022nearly}] \label{lem:potential_small}For a fixed $\lambda > 0$ and $\{z_t\}_{t=1}^T \subset \RR^d$ with $\norm{z_t} \leq 1$, we define $Z_t = \lambda I + \sum_{s=1}^{t-1} z_s z_s^\top$. Then, for any constant $c > 0$, $\norm{z_t}_{Z_t^{-1}} > c$ happens at most 
$
\frac{3d}{\log(1+c^2)} \log \Big( 1 + \frac{1}{\lambda \log(1+c^2)}\Big).
$
\end{lemma}
The ratio between the out-of-sample error and the in-sample error in the linear case can be readily generalized to the general function approximation using the variant of eluder dimension considered in \citet{gentile2022fast,zhang_2023_ltbook,ye2023corruption, agarwal2023vo}, which essentially states that there is some low-rank structure in the reward model space so the generalization is limited (the elliptical potential cannot be large for too many times). Moreover, if we can effectively estimate the in-sample error from the preference data, by Lemma~\ref{lem:potential_small}, we can infer the out-of-sample error safely most of the time. Such an in-sample error estimation is provided in Lemma~\ref{lem:in-sample}. Essentially, the eluder-type complexity measures and techniques reduce the learning problem to an online supervised learning (in-sample error estimation and minimization) \citep{zhong2022gec}.

\section{Offline Learning with Pessimism} \label{sec:offline}
\subsection{Setup}
In this section, we consider the offline setting, where we aim to learn a good policy from a pre-collected dataset without further interactions with the human. We suppose that we are given an offline preference dataset:
$
\cD_{\off} = 
\{(x_i, a_i^1, a_i^2, y_i)\}_{i=1}^{n_{\off}}.
$
We denote $\Sigma_{\off} := \Sigma_{\cD_{\off}}$ for offline setting. To motivate the algorithmic design, with a comparator policy $\pi$, we recall Lemma~\ref{lem:decom} and Lemma~\ref{lem:opt_error} to obtain that 
\begin{equation}
    \label{eq:off_decomp}
    \begin{aligned}
    &  J(\pi) - J(\hat{\pi}) = \E_{x\sim d_0}\Big[{\EE_{\pi}[r^*(x, a) - \hat{r}(x, a)]} + {\EE_{\hat{\pi}}[\hat{r}(x, a) - r^*(x, a)]}
    - \eta \cdot \E_{x \sim d_0} \big[\KL(\pi^*(\cdot|x)\|\hat{\pi}(\cdot|x))\Big],
\end{aligned}
\end{equation}
where $\hat{\pi}$ is induced by calling the Oracle~\ref{assu:oracle1} with $\hat{r}$. In other words, the sub-optimality depends on the quality of the learned reward $\hat{r}$, under the distributions induced by $\pi$ and $\hat{\pi}$, separately. As we have mentioned in Figure~\ref{fig:over_opt}, the finite $\cD_{\off}$ from the behavior policy can hardly cover the whole prompt-space. The standard way to handle this issue is to leverage the principle of pessimism with a conservative reward \citep{jin2021pessimism, rashidinejad2021bridging, xie2021bellman, zanette2021provable}, which means that we adopt an estimator that is a lower bound of the true value with high probability. A technical motivation for doing so is that in \eqref{eq:off_decomp}, the second term ${\EE_{\hat{\pi}}[\hat{r}(x, a) - r^*(x, a)]}$ is hard to control because both the estimation target $(\hat{r} - r^*)$ and the distribution induced by $\hat{\pi}$ depend on $\cD_{\off}$. Therefore, they are statistically dependent and characterizing this term is challenging.

In this section, we connect the newly formulated KL-regularized bandit problem with the pessimism and show that the modified variants are sample efficient. 

\subsection{Algorithms}

We introduce two different ways to achieve pessimism. The first one is to directly penalize the reward estimation by an uncertainty estimator $\hat{r}(x,a) = r_{\mle}(x,a) - \beta \cdot \Gamma(x,a,\nu,\cD_{\off})$ so that $\hat{r}(x,a) - r^*(x,a) \leq 0$ for all $(x,a) \in \cX \times \cA$. The construction of the uncertainty bonus is a standard application of concentration inequality. Intuitively, the $\hat{r}$ is an estimation of the ground truth $r^*$, and $\hat{r}$ will converge to $r^*$ with infinitely many samples that cover the whole feature space well. With finite samples, we can use the statistical tool (concentration inequalities) to quantify the estimation error, in the sense that with high probability,  
$$
|r^*(x,a) - \hat{r}(x,a)| \leq \Gamma(x,a,\nu,\cD_{\off}),
$$
where $\nu$ is a reference vector so that the uncertainty is relative to feature $\nu$, which makes sense under the preference learning nature. We omit the mathematical details here for a clear presentation and we defer the details to Appendix~\ref{appendix:offline_proof}. 

In addition to adopting a reward estimator with point-wise pessimism, we may also use a modified target that is biased toward pessimism by penalizing the uncertainty as in \eqref{eqn:pessimism}. Here we do not maintain a confidence set but use a modified target that is biased toward pessimism, similar to \citet{xie2021bellman, zhang2022feel}, which may be easier to approximate in practice \citep{liu2023maximize}. Moreover, to handle the additional trade-off between the reward and the KL term, we also incorporate the KL divergence into the policy computation.
        
The full algorithmic framework is presented in Algorithm~\ref{alg:offline} and is referred to as the offline \textbf{Gibbs Sampling from Human Feedback} (GSHF) because the output policy is the Gibbs distribution with some reward. 
\begin{algorithm}
\caption{Offline GSHF}
\label{alg:offline}
\begin{algorithmic}[1]
    \STATE \textbf{Input:} $\cD_{\off}$, $\lambda > 0$, $\beta > 0$, reference vector $\nu$, and prompt distribution $d_0$.
    \STATE Compute $\theta_{\mle}$ based on $\cD_{\off}$ by maximizing the likelihood given in \eqref{eqn:bt_likelihood} 
    \STATE \texttt{Option I}: Output $\hat{\pi}$ by constructing expected uncertainty estimator $\Gamma^e(\pi,\nu,\cD_{\off})$ and solving
    \begin{equation}
        \label{eqn:pessimism}
        \begin{aligned}
                    \hat{\pi} &= \argmax_{\pi \in \Pi} \Big[ \EE_{x \sim d_0, a\sim \pi(\cdot|x)} [r_{\mle}(x,a)]  - \beta \cdot \Gamma^e(\pi,\nu,\cD_{\off})  - \eta \cdot \EE_{x \sim d_0} [\KL(\pi(\cdot|x)\|\pi_0(\cdot|x))] \Big].
        \end{aligned}
    \end{equation}
    \STATE \texttt{Option II}: Output $\hat{\pi}$ by constructing uncertainty estimator $\Gamma(x,a,\nu,\cD_{\off})$ and calling Oracle~\ref{assu:oracle1} with $\hat{r}(x,a) = r_{\mle}(x,a) - \beta \cdot \Gamma(x,a,\nu,\cD_{\off})$.
\end{algorithmic}
\end{algorithm}

\subsection{Main Results: Pessimism is Provably Efficient}
We now present the main theoretical guarantee for Offline GSHF.
\begin{theorem}\label{thm:offline}
    Under Assumption \ref{assu:linear}, if we set $\beta := O\big(\sqrt{\frac{d + \log(1/\delta)}{\gamma^2} + \lambda B^2}\big)$, for any $\lambda > 0$ and comparator policy $\pi \in \Pi$, with probability at least $1-\delta$, the output policy of Algorithm~\ref{alg:offline} with Option I and $\Gamma^e(\pi,\nu,\cD_{\off}) = \norm{\EE_{x \sim d_0} [\phi(x, \pi) - \nu]}_{\Sigma_{\off}^{-1}}$ satisfies
    $$
    J(\pi) - J(\hat{\pi}) \leq 2 \beta \cdot \norm{\EE_{x \sim d_0} [\phi(x, \pi)] - \nu}_{\Sigma_{\off}^{-1}},
    $$
    and the output policy of Algorithm~\ref{alg:offline} with Option II and $\Gamma(x,a,\nu,\cD_{\off}) = \norm{\phi(x,a) - \nu}_{\Sigma_{\off}^{-1}}$ satisfies
    $$
    \begin{aligned}
            J(\pi) - J(\hat{\pi})& \leq 2 \beta \cdot \EE_{x \sim d_0, a\sim \pi(\cdot|x)} \norm{\phi(x,a) - \nu}_{\Sigma_{\off}^{-1}}  -\eta \cdot \EE_{x \sim d_0} \big[\KL(\pi(\cdot|x)\|\hat{\pi}(\cdot|x))\big].
                \end{aligned}
    $$
\end{theorem}
We can combine the guarantee with dataset property, usually referred to as the coverage on the comparator policy $\pi$ \citep{jin2021pessimism, xie2021bellman}, to obtain the concrete bound. See Proposition~\ref{cor:hybrid:1} for an concrete example. The proof of the theorem is rather standard in offline learning based on the principle of pessimism but with a different analysis to handle the KL and the stochastic policy. We defer the proof of the theorem to Appendix~\ref{appendix:offline_proof}. The reference vector $\nu$ in Algorithm~\ref{alg:offline} is typically set as $\EE_{x \sim d_0} [\phi(x, \pi_{\mathrm{ref}})]$ for some available $\pi_{\mathrm{ref}}$. As showcased by \citet{zhu2023principled}, the subtracted reference vector can serve as a pre-conditioner for a better suboptimality bound. For instance, one typically choice is $\pi_{\mathrm{ref}} = \pi_0$ so that $\pi_0$ achieves a reward of zero \citep{ouyang2022training, gao2023scaling, dong2023raft}. 

\textbf{Comparison of two implementations of pessimism.} In comparison, the Option I achieves a sharper bound in the uncertainty bonus because the expectation is inside the norm and by Jensen's inequality (Lemma~\ref{lem:jesen}) we know that 
$$
\norm{\EE_{x \sim d_0} [\phi(x, \pi)] - \nu}_{\Sigma_{\off}^{-1}} \leq \EE_{x \sim d_0, a\sim \pi(\cdot|x)} \norm{\phi(x,a) - \nu}_{\Sigma_{\off}^{-1}}.
$$
Moreover, Option I has a desirable robust improvement property. If we take $\nu = \EE_{x \sim d_0}[\phi(x,\pi_{\mathrm{ref}})]$, the resulting policy will be better than $\pi_{\mathrm{ref}}$, regardless of the coverage of the $\cD_{\off}$ according to Theorem~\ref{thm:offline}, which is similar to the original offline RL literature for a robust policy improvement \citep{bhardwaj2023adversarial}. We will also see that the use of a reference policy $\pi_{\mathrm{ref}}$ can also simplify the algorithmic design in subsequent Section~\ref{sec:hybrid}. However, the main advantage of Option II is that the Oracle~\ref{assu:oracle1} can be empirically well approximated. For instance, we can directly plug the pessimistic reward into the PPO algorithm. Moreover, we have the following algorithm in a direct preference learning manner that resembles that of \citet{rafailov2023direct, zhao2023slic, azar2023general}. We defer a detailed discussion to Section~\ref{sec:discuss}.

\section{Online Iterative Learning with Batch Exploration}
\label{sec:hybrid}
\subsection{Setup: Batch Hybrid Learning}
Beyond the offline learning, it is also common to query human feedback during the training process. For instance, \citet{bai2022training, touvron2023llama} typically iterate the RLHF process on a weekly cadence, where the fresh RLHF models are deployed to interact with crowdworkers and to collect new human preference data.

We consider a slightly more general setting here, where we refer it to as the hybrid learning. This is because while it is possible to learn from scratch, in many cases, we tend to start with the offline open-source datasets \citep{touvron2023llama, bai2023qwen}. For instance, in LLaMA2 \citep{touvron2023llama}, the authors start with 1500K open-source comparison pairs $\cD_{\off}$ and keep $\cD_{\off}$ in the data mixture for the entire RLHF process. Motivated by the practical applications, we formulate the process as a batch hybrid framework in this section. It is shown that such a batch online framework can significantly improve the aligned LLMs as evaluated by the humans \citep{bai2022training, touvron2023llama}. For completeness, we also develop the pure online setting in Appendix~\ref{sec:online}. Mathematically, consider the batch hybrid setting of $T$ batches with a fixed batch size $m$. 

The agent initializes with the $\cD_{\off}$ (if applicable, otherwise initialize as empty). For each iteration $t \in [T]$, 
\begin{itemize}
    \item we first update the policy pair $(\pi_t^1,\pi_t^2)$ based on the historical data collected so far: $\cD_{\off}$ and $\cD^{1:t-1}$;
    \item we collect $m$ tuples as $\cD^t$: we sample a random prompt by $x_{t,i} \sim d_0$,  collect two responses by $( a_{t,i}^1, a_{t,i}^2) \sim (\pi_t^1, \pi_t^2)$, and query the preference signal $y_{t,i}$ from the underlying BT model;
    \item the next iteration begins.
\end{itemize}

The main technical challenge here is to decide the behavior policy pairs $(\pi_t^1, \pi_t^2)$. It is well known that it is important to strike a balance between exploration and exploitation to get RL works \citep{auer2002finite}, and we study such a trade-off in the context of online iterative RLHF.

\subsection{Algorithms}

\textbf{Non-symmetric algorithmic structure.} Our first idea is to adopt a non-symmetric structure in choosing $\pi_t^1$ and $\pi_t^2$. Specifically, we refer the $\pi_t^1$ as the main agent, which aims to learn a good policy so that the suboptimality gap $J(\pi^*) - J(\pi_t^1)$ is small. In contrast, the second agent, referred to as the enhancer, seeks to enhance the learning of the main agent by choosing appropriate $\pi_t^2$. The main advantage of such a non-symmetric structure is that we have a lot of freedoms to choose $\pi_t^2$ because we do not worry about the sub-optimality incurred by it \citep{jin2021power, xiong2022self, huang2021towards}.

\begin{algorithm}
\caption{Online Iterative GSHF}
\label{alg:hybrid}
\begin{algorithmic}[1]
    \STATE \textbf{Input:} $m, \Pi, \cD_{\off}$ (if applicable, otherwise $\cD_{\off}=\emptyset$), $\beta$.
    \FOR{$t=1,2,\cdots,T$}
    \STATE Observe $x_{t,i} \sim d_0$ for $i=1,2,\cdots,m$ \footnote{This is only for simplicity so that we only compute the policy for these $m$ prompts. The analysis holds for the strictly sequential setting.}. 
    \STATE \textcolor{magenta}{Exploitation with the main agent}: denote the MLE $r_{\mathrm{MLE}}$ with the likelihood in \eqref{eqn:bt_likelihood} and compute the best guess we have so far:
    \begin{equation}\label{eqn:main_agent}
    \begin{aligned}
    \pi_t^1 = \argmax_{\pi \in \Pi} \E_{x \sim d_0} \E_{a \sim \pi(\cdot|x)} \Big[r_{\mathrm{MLE}}(x,a) - \eta \KL(\pi(\cdot|x)\Vert \pi_0(\cdot|x))\Big].
    \end{aligned}
    \end{equation}
    \STATE \textcolor{teal}{Exploration with the enhancer}:
    \STATE\hspace{\algorithmicindent} \texttt{Option I:} if $\cD_{\off}$ can provide good coverage, no need to explore so $\pi_t^2 = \pi_{\mathrm{ref}}$; \vspace{2pt}
    \STATE\hspace{\algorithmicindent} \texttt{Option II:} otherwise, choose $\pi_t^2$ by maximizing the relative uncertainty relative to $\pi_t^1$:
       \begin{equation}
            \label{eqn:confidence_set_online}
            \begin{aligned}
        &\argmax_{\tilde{\pi}} \Gamma^m_t (\lambda, \pi_t^1,\tilde{\pi}) := \textcolor{red}{\beta \cdot \sum_{i=1}^m\norm{\phi(x_{t,i},\tilde{\pi})-\phi(x_{t,i},\pi_t^1)}_{\Sigma_{t,m}^{-1}}}\\
    &\text{subject to: } \tilde{\pi} \in \Pi_t = \{\pi' \in \Pi: \eta\cdot \E_{x \sim d_0} \KL(\pi'(\cdot|x),\pi_t^1(\cdot|x)) \leq \Gamma^m_t (\lambda, \pi_t^1,\pi')\},\\
    &\text{where } \Sigma_{t,m} = \lambda I + \frac{1}{m}\sum_{i=1}^{t-1}\sum_{j=1}^m (\phi(x_{i,j},a_{i,j}^1) - \phi(x_{i,j},a_{i,j}^2))(\phi(x_{i,j},a_{i,j}^1) - \phi(x_{i,j},a_{i,j}^2))^{\top}.
    \end{aligned}
    \end{equation}
    \STATE Sample $a_{t,i}^1 \sim \pi_t$ and $a_{t,i}^2 \sim \pi_t^2$, receive human feedback for all $i \in [m]$, and collect them as $\cD^t$.
    \ENDFOR
        \STATE \textbf{Output:} the best model in $\pi^1_{1:T}$ by a validation set.
\end{algorithmic}
\end{algorithm}
From a high level, the Online Iterative GSHF Algorithm continuously enhances the historical dataset by strategically interacting with the human labeller. Specifically, in Algorithm~\ref{alg:hybrid}, the main agent always exploits all the historical information by taking the optimal policy induced from the MLE $r_{\mathrm{MLE}}$. This process, however, requires the newly collected data can provide more information, compared to those collected previously. We divide the problem into two different situations.

\textbf{Option I.} If the $\cD_{\off}$ is diverse enough and provides us with a good coverage in terms of the $(\pi^*, \pi_{\mathrm{ref}})$, we do not need to actively explore and can take $\pi_t^2 = \pi_{\mathrm{ref}}$ directly\footnote{We also mention in passing that for most of the cases, $\pi_{\mathrm{ref}} = \pi_0$. We use $\pi_{\mathrm{ref}}$ for a slightly more general formulation.};

\textbf{Option II.} If we cannot expect to have a diverse $\cD_{\off}$, the enhancer aims to maximizes the uncertainty relative to the main agent $\pi_t^1$, while maintaining a moderate KL divergence\footnote{We remark that we impose an additional KL regularization on $\pi_t^2$ is so that it also converges to $\pi^*$.}. In this case, the $\pi_t^2$ aims to explore toward the direction that maximizes the uncertainty relative the main agent $\pi_t^1$. We additionally impose a constraint:
$$
\tilde{\pi} \in \Pi_t = \{\pi' \in \Pi: \underbrace{\eta \cdot \E_{x \sim d_0} \KL(\pi'(\cdot|x),\pi_t^1(\cdot|x))}_{\text{How far does the enhancer go.}} \leq \underbrace{\Gamma^m_t (\lambda, \pi_t^1,\pi')}_{\text{How much information we can get.}}\},
$$
where this set is never empty because $\pi_t^1$ always belongs to it so the choice is well-defined. Intuitively, we require that the information we can get is worth the cost we pay by moving away from $\pi_t^1$.

Essentially, we are trying to boost our dataset by strategically choose the behavior policies at each iteration. If we learn from scratch or the offline dataset is not good enough, we need to explicitly incorporate the uncertainty into algorithmic design, and explore the direction where we are uncertain about so that we can gain more information. On the other hand, if $\cD_{\off}$ is already good enough, our analysis shows that it is also beneficial to collect more online data, as we now discuss in the next subsection.

\subsection{Main Results: RLHF Benefits from Online Exploration}

We first consider the case of Option I: $\pi_t^2 = \pi_{\mathrm{ref}}$. In this case, it is essential to have a diverse $\cD_{\off}$ as we do not explicitly explore. This is most related to the study of hybrid RL in the classic RL theory \citep{song2022hybrid, zhou2023offline}. The major difference here is that for preference-based learning, the uncertainty is evaluated on the feature difference instead of a single state-action pair, as we summaize in the following assumption.
\begin{assumption}[Partial Coverage of Offline Data]\label{as:Partial Coverage of Offline Data} 
For the linear model, there exists a reference policy $\pi_{\mathrm{ref}}$, a ratio coefficient $\alpha(mT,\cD_{\off})\in(0,1)$ and a coverage constant $C_{\mathrm{cov}}>0$ such that
$$
\begin{aligned}
    (mT)^{1-\alpha(mT,\cD_{\off})} \norm{\EE_{x \sim d_0} [\phi(x,\pi^*) - \phi(x,\pi_{\mathrm{ref}})]}_{(\Sigma_{\off})^{-1}} \le C_{\mathrm{cov}}.
\end{aligned}
$$
\end{assumption}

We remark that Assumption~\ref{as:Partial Coverage of Offline Data} implicitly assume that $n_{\off}$ is comparable to the total number of online samples $mT$ so that the influence of $\cD_{\off}$ will not be dominated by the online data. To provide a more detailed understanding and connection to existing literature, we offer a more nuanced characterization of $\alpha(mT,\cD_{\off})$ under standard partial coverage conditions in Appendix~\ref{sec:Proof of the Hybrid prop}.  In particular, when $mT \approx n_{\mathrm{off}}$, we show that $\alpha(mT,\cD_{\off}) \approx 1/2$. It is worth emphasizing that this scenario appears to be realistic for LLMs. For example, in the LLaMA2 project \citep{touvron2023llama}, we observe $n_{\off}=1500K$ and $mT= 1400K$. We are ready to present the results.

 \begin{theorem} \label{thm:hybrid:batch}
 For any $\epsilon>0$, under Assumption~\ref{assu:linear} with $T = \min \{n \in \mathbb{N}^+: n \geq d \log (n)\}$ and $\beta = O\big(\sqrt{\frac{d \log(T/\delta)}{\gamma^2}}\big)$, with probability at least $1-3\delta$, there exists a $t_0 \in [T]$ such that Algorithm~\ref{alg:hybrid} with Option I holds
 $$
\begin{aligned}
J(\pi^*) - J(\pi_{t_0})&\lesssim \sqrt{\frac{d}{\gamma^2 m}} + \beta \cdot \norm{\EE_{x \sim d_0} [\phi(x,\pi^*) - \phi(x,\pi_{\mathrm{ref}})]}_{\textcolor{red}{\Sigma_{\off+\cD^{1:t_0}}^{-1}}} - \eta \EE_{x_{t_0} \sim d_0}\big[\KL(\pi_{t_0}^1(\cdot|x_{t_0})\|\pi^*(\cdot|x_{t_0}))\big],
\end{aligned}
$$
where $\Sigma_{\off + \cD^{1:t_0}}$ denotes the covariance matrix computed on $\cD_{\off} \cup \cD^{1:t_0}$. If we further assume that Assumption~\ref{as:Partial Coverage of Offline Data} holds, we have
$$
\begin{aligned}
J(\pi^*) - J(\pi_{t_0}) \lesssim \sqrt{\frac{d^2}{\gamma^2 m}} + \beta (mT)^{\alpha(mT, \cD_{\mathrm{off}}) - 1} C_\mathrm{cov}  -\eta \EE_{x_{t_0} \sim d_0}\big[\KL(\pi^*(\cdot|x_{t_0})\|\pi_{t_0}(\cdot|x_{t_0}))\big].
\end{aligned}
$$
\end{theorem}
The proof is deferred to Appendix~\ref{appendix:pf:hybrid:batch}. Note that the second result is a conservative bound where we completely ignore the coverage provided by the online data $\cD^{1:t_0-1}$. We now qualitatively analyze the impact of the online data as follows.

\textbf{RLHF Benefits from Online Exploration.} One natural question arises under Assumption~\ref{as:Partial Coverage of Offline Data}: if we can directly apply Algorithm~\ref{alg:offline} to get a good policy, why should we collect online data? The difference is that now the second term corresponds to the coverage condition of $\cD_{\off} \cup \cD^{1:t_0}$. Under Assumption~\ref{as:Partial Coverage of Offline Data}, with suitable hyper-parameters (large enough $m$ and suitable $\beta$), we know that $\pi_t^1 \to \pi^*$. Since the online data is collected by $(\pi_t^1, \pi_{\mathrm{ref}})$ and the goal is to cover $(\pi^*, \pi_{\mathrm{ref}})$, we expect that the intermediate policies can provide a much better coverage as compared to the $\cD_{\off}$, i.e., a much smaller $C_{\mathrm{cov}}$, for many average instances. We will partially verify this intuition in the experiment part. 

We now move to the case of Option II where we cannot expect to have a diverse $\cD_{\off}$ and need explicit exploration. We first show that with suitable $\beta$, the constructed confidence set contains $\pi^*$ with high probability.
\begin{lemma}[Confidence set] \label{lem:confidence_set}
For the linear model in Assumption \ref{assu:linear}, given the policy of the main agent $\pi_t^1$, we consider the following confidence set with $\beta = O\big(\sqrt{\frac{d\log(T/\delta)}{\gamma^2m}}\big)$:
    $$
    \begin{aligned}
            \Pi_t &= \Big\{ \tilde{\pi} \in \Pi: \eta \sum_{i=1}^m\KL(\tilde{\pi}(\cdot|x_{t,i}) \|\pi_t^1(\cdot|x_{t,i})) \leq \beta \sum_{i=1}^m\norm{\phi(x_{t,i},\tilde{\pi})-\phi(x_{t,i},\pi_t^1)}_{\Sigma_{t,m}^{-1}} \Big\}.
    \end{aligned}
    $$
    Then, with probability at least $1-\delta$, we know that $\pi^* \in \Pi_t$ for all $t \in [T]$. 
\end{lemma}
We defer the proof to Appendix~\ref{sec:proof_confidence_set} and present the main result for the Option II.

\begin{theorem} \label{th:batch_online} 
For any $\epsilon>0$, we set the batch size $m=d/(\gamma^2\epsilon^2)$.
Under Assumption \ref{assu:linear},
with $\beta := O\big(\sqrt{\frac{d \log(T/\delta)}{\gamma^2m}}\big)$ and $\lambda=\Theta\big(d\log(T/\delta)/(m\gamma^2B^2)\big)$, after $T= \min \{ n\in\mathbb{N}^+: n \ge d\log(n)\}$ iterations, we have with probability at least $1-3\delta$, Algorithm~\ref{alg:hybrid} with Option II satisfies: there exists a $t_0 \in [T]$, 
$$
\begin{aligned}
    J(\pi^*) - J(\pi^1_{t_0}) \lesssim \epsilon-\eta \cdot \EE_{x_{t_0} \sim d_0} \big[\KL(\pi^*(\cdot|x_{t_0})\|\pi^1_{t_0}(\cdot|x_{t_0}))\big],
\end{aligned}
$$
where the number of collected samples is at most 
$
mT = \tilde O\Big(\frac{d^2}{\gamma^2\epsilon^2}\Big).
$
\end{theorem}

\textbf{Getting Rid of Data Coverage.} One notable feature of Theorem~\ref{th:batch_online} is that with explicit exploration, we do not need an offline dataset with good coverage, highlighting the importance of strategic explorations. This is particularly important in the context of LLMs because the distribution shift between LLMs are very large so data coverage is more sparse. For instance, along the way of finding the optimal policy of some learned reward function by PPO, the KL divergence to the initial checkpoint can be $>25$ \citep{bai2022training}.

\textbf{Algorithmic Simplicity v.s. Data Coverage.} We note that Option I and Option II are complementary to each other and hold their own values. Specifically, the Option I offers simplicity in algorithmic design, at the cost of demand for a high-quality $\cD_{\off}$. In comparison, the online learning does not relies on the quality of $\cD_{\off}$, but the choice of the enhancer is challenging because for the neural network, the uncertainty estimators do not admit a closed-form. In practice, we typically resort to heuristic methods \citep{wu2021uncertainty, coste2023reward} to estimate the uncertainty, as we discuss in the experiment part of this paper.

\textbf{The advantage of reward modeling.} Theorem~\ref{thm:hybrid:batch} and Theorem~\ref{th:batch_online} reveal a key characteristic of reward modeling: the sample complexity is dependent on the complexity of the reward model rather than the generative models. For simple reward functions, such as sentiment or politeness evaluation, the required function class is substantially smaller compared to the generative model. This is corroborated by evidence showing that even compact models like BERT \citep{devlin2018bert} can yield accurate reward assessments. This may illustrate the advantage of the most popular RLHF framework used by \citet{ouyang2022training, bai2022training, touvron2023llama}, in contrast to the idea of bypassing reward modeling \citep{rafailov2023direct,zhao2023slic, azar2023general} and training based only on the offline dataset. 

\section{Implementations of GSHF} \label{sec:discuss}

In this section, we discuss how to practically implement the information-theoretical Algorithm~\ref{alg:offline} and Algorithm~\ref{alg:hybrid}. From a high level, the GSHF framework can be implemented by combining many existing algorithms to approximate the computational Oracle~\ref{assu:oracle1}. Notable examples include PPO, DPO, and InfoNCA \citep{chen2024noise}. Here, we focus on the PPO and DPO and discuss several popular ways to implement it. 

In practice, the policy is represented by a deep neural network. In this case, one common choice \citep{ziegler2019fine, wu2021recursively, ouyang2022training, bai2022training} is to use the standard deep RL algorithms like PPO to optimize the regularized reward:
$$
\tilde{r}(x,a) = r(x,a) - \eta \log \frac{\pi_{\theta}(a|x)}{\pi_{0}(a|x)}.
$$
However, PPO is significantly less stable and sensitive to implementation as compared to SFT \citep{choshen2019weaknesses, engstrom2020implementation}. Recently, another line of work considers a family of algorithms that directly learn from the preference data without reward modeling, including the DPO \citep{rafailov2023direct}, SLIC \citep{zhao2023slic}, IPO \citep{azar2023general}. These algorithms attracted significant attention due to its stability and easy implementation. We use the DPO as a representative example here. Specifically, DPO chooses to train the LLM as a reward model, by optimizing the following loss: 
\begin{equation} \label{eqn:dpo_loss} 
\begin{aligned}
 \sum_{(x,a^w,a^l) \in \cD_{\off}} -\Big[ \log \sigma\Big(\eta \log \frac{\pi_{\theta}(a^w|x)}{\pi_0(a^w|x)} - \eta \log \frac{\pi_{\theta}(a^l|x)}{\pi_0(a^l|x)} \Big)\Big],
    \end{aligned}
\end{equation}
where $a^w$, $a^l$ is the chosen/rejected response. It is shown that the optimal policy for the DPO loss in \eqref{eqn:dpo_loss} is identical to the one for the RLHF objective $\pi_r$ when $r$ is the MLE \citep{azar2023general}. To fit the DPO into the GSHF framework, we generalize this result to incorporate the pessimism. 

\subsection{Direct Preference Learning with Pessimism}
For notation simplicity, we denote the uncertainty bonus as $\Gamma(x, y)$ and omit the dependency on $\nu$ and $\cD_{\off}$. Then, we have the following proposition.
\begin{proposition}[Direct Preference Learning with Pessimism] \label{prop:pe_dpo}
Given the preference dataset $\cD_{\off}$, we can implement Option II of Algorithm~\ref{alg:offline} by minimizing the following loss function:
\begin{equation}
    \label{eqn:pessimistic_dpo}
    \mathcal{L}_{\cD_{\off}}(\theta, \pi_{0}) = \sum_{(x,a^w,a^l) \in \cD_{\mathrm{off}}} \log \sigma \bigg(\eta \log \frac{\pi_\theta(a^w|x)}{\pi_0(a^w|x)} - \eta \log \frac{\pi_\theta(a^l|x)}{\pi_0(a^l|x)} + \underbrace{(\Gamma(x, a^w) - \Gamma(x,a^l))}_{m(x, a^w, a^l)} \bigg),
\end{equation}
where $y^w$ is preferred over $y^l$. 
\end{proposition}
Intuitively, we add an adaptive margin for each preference pair $(x,a^w,a^l)$ according to their uncertainty difference. We defer the proof to Appendix~\ref{appendix:dpo_pessimistic}. 

\textbf{Uncertainty Estimation via Ensemble.} The uncertainty estimation for the general neural network is still an open problem. In practical applications, we typically resort to heuristic methods. For instance, \citet{coste2023reward} uses the idea of ensemble to get the pessimistic reward model. Specifically, they independently train $5$ reward models $\{r_i\}_{i=1}^5$ and subtract the variance:
$$
\hat{r}(x,a) := \frac{1}{5}\sum_{i =1}^5 r_i(x,a) - \lambda \frac{1}{5} \sum_{i=1}^5 \big(r_i(x,a) - \frac{1}{5}\sum_{j =1}^5 r_j(x,a) \big)^2.
$$
It is shown that such a pessimistic version of the reward model can significantly reduce the reward over-optimization for PPO and best-of-n sampling \citep{nakano2021webgpt}.

\subsection{Enhancer Explores with Variants of Main Agent Policy}
\label{sec:rs_gshf}
For the online exploration, selecting an appropriate optimistic policy for the enhancer to maximize the uncertainty with respect to the main agent $\pi_t^1$ is largely less explored in practical applications. To be specific, we recall that our goal is to find an enhancer policy from the following policy subset:
$$
\tilde{\pi} \in \Pi_t = \{\pi' \in \Pi: \underbrace{\eta \cdot \E_{x \sim d_0} \KL(\pi(\cdot|x),\pi_t^1(\cdot|x))}_{\text{How far does the enhancer go.}} \leq \underbrace{\Gamma^m_t (\lambda, \pi_t^1,\pi')}_{\text{How much information we can get.}}\}.
$$
While it is challenging to obtain the analytical solution of uncertainty, the insight here is to maximize the policy difference with $\pi_t^1$, while maintaining a moderate KL divergence. We discuss some popular heuristic implementations here. 

\textbf{Model Variants.} In the project of Claude \citep{bai2022training}, the authors choose to use the models with different training steps as $(\pi_t^1, \pi_t^2)$. For instance, if we run PPO for $2$ epoch in total, we may take $\pi_t^1$ as the model saved at the end of the first epoch and take $\pi_t^1$ as the one saved at the end of second epoch. Moreover, in addition to the model variants, the LLaMA2 project \citep{touvron2023llama} further adjusts the sampling temperature of $\pi_t^1$ to induce $\pi_t^2$. 

\textbf{Rejection Sampling.} A popular ensemble-based approach in the literature is the so-called rejection sampling \citep{nakano2021webgpt, dong2023raft, liu2023statistical}. We present a brief introduction to the concept of rejection sampling in Appendix~\ref{appendix:rs}. In the context of LLM, however, the rejection sampling is usually restricted to the best-of-n sampling. Specifically, we sample $n$ independent responses by $\pi_t^1$ for each prompt, and then use a reward function to rank them and take the one with the highest reward as the final output. In other words, we take $\pi_t^2$ as the best-of-n variant of $\pi_t^1$. In this way, the $\pi_t^2$ enlarges the margins between $\pi_t^1$. Meanwhile, in this case, the KL divergence between the two policies is upper bounded by $\log n - \frac{n-1}{n}$ and is usually far better than this conservative estimation \citep{beirami2024theoretical}. We note that similar idea has been adopted in \citet{liu2023statistical,snorkelai@pair, yuan2024self} for improving DPO.

\subsection{Offline Learning with Pseudo-Labeling}
We now consider a family of approaches that may slightly deviate from the main story of the paper but are beneficial to clarify some confusing concepts in RLHF. In the formulation of RLHF, we define the offline learning as learning without further querying the human feedback (the underlying ground-truth BT model), while we define the online learning as the scenario where we can query the humans along the way of training. In this sense, there are several existing algorithms in the literature are classified as offline one:
\begin{itemize}
    \item PPO with a fixed reward \citep{christiano2017deep, ziegler2019fine};
    \item RAFT (rejection sampling fine-tuning, or iterative SFT) \citep{dong2023raft}: we generate $n$ responses for each prompt, and use a fixed reward to rank them, and fine-tune the model on those with high rewards;
    \item RSO (DPO with rejection sampling) \citep{liu2023statistical}: we generate $n$ responses and use statistical rejection sampling to approximately sample from $\pi_0(\cdot|x) \exp(\frac{1}{\eta} r(x, \cdot))$ and use these samples to run DPO.
\end{itemize}
All these algorithms do not query new human feedbacks during the training. Instead, they first train a proxy reward $\hat{r}$, and use $\hat{r}$ to label the model-generated samples for the subsequent training. In particular, it is known that 
\begin{itemize}
    \item PPO and RAFT outperform the SFT-baseline, which fine-tunes the models on the preferred samples \citep{dong2023raft, yuan2023rrhf};
    \item RSO outperforms DPO \citep{liu2023statistical}.
\end{itemize}
In other words, while we are prohibited from collecting the ground-truth preference label, the offline RLHF benefits from the pseudo labels from the learned reward, which resembles the insights of the semi-supervised learning. One reasonable hypothesis is that the reward model may generalize better than the policy in terms of sample complexity, i.e., reward model has better preference classification accuracy given a fixed number of samples. Some empirical results \citep{li2023policy} can also support the hypothesis.


\textbf{multi-step RSO.} Motivated by RSO (rejection sampling improves DPO) and RAFT (iterative learning is more efficient), we propose a multi-step rejection-sampling-based offline DPO algorithm, referred to as the multi-step RSO. To motivate our method, we first review the main challenge of RSO. \citet{liu2023statistical} found that the usage of offline datasets typically impedes the effectiveness of DPO-based algorithms. This negative impact is particularly pronounced when there is a disparity between the distribution of offline data and the target distribution. Consequently, they trained a reward model, denoted as $r$, and approximated samples from $ \pi_r $ using rejection sampling. In this case, they generate samples from the optimal policy of the underlying BT model associated with $r$ and get $\cD_{\mathrm{gen}} = \{(x,a^1,a^2,y)\}$. The authors suggested that this is more suitable for DPO training and leads to better performance. The key basis of the success of RSO is that the rejection sampling can well approximate $\pi_r$.

However, in practice, the rejection rate can be so large that the sampling is not effective. 
Given a prompt-response pair $(x,a)$, 
the rejection rate is $1-\exp(-\eta^{-1}(R(x)-r(x,a)))$, where $R(x)$ is the largest possible reward over all $a\in \cA$. 
{For example, given $\eta>0$, if the samples drawn from $\pi_0(a|x)$ satisfies
$\EE_{a\sim\pi_0(a|x)}\exp(\eta^{-1}r(x,a)) = \exp(-\eta^{-1}(r_x-R(x)))$, the expected acceptance rate becomes $\exp(-\frac{r_x}{\eta})$, where $r_x$ is the reward gap between average sample and the best sample given prompt $x$.}  Setting $r_x = 1$ and $\eta = 0.1$ yields a notably low acceptance rate of approximately $0.00004$. Essentially, the majority of samples are rejected, necessitating a substantial number of sampled candidates to produce a single accepted comparison pair. In the practical implementation of RSO \citep{liu2023statistical}, we typically fix the total budget of candidate responses and the number of samples to be accepted. In this case, due to the low sampling efficiency, the collected samples may not well approximate the target distribution, and train on these samples can lead to inferior performance compared to the original DPO. 

To mitigate this issue and to make the algorithm more effective, we propose a multi-step approach to progressively achieve our ultimate target. Instead of using $\pi_0$ to approximate $\pi_0 \exp(\frac{1}{\eta} r)$ directly, we divide the path into several steps by considering a sequence of distributions
$$
\pi_0 \to \pi_0 \exp(\frac{1}{\eta_1} r) \to \cdots \to \pi_0 \exp(\frac{1}{\eta_N} r), 
$$
where $\eta_0 = \infty$ (i.e., $\pi_0$), $\eta_N = \eta$. The high-level intuition is that while approximating $\pi_r$ from $\pi_0$ is hard, approximating $\pi_0\exp(\frac{1}{\eta_i} r)$ with $\pi_0\exp(\frac{1}{\eta_{i-1}} r)$ is much easier. Therefore, we can do the rejection sampling step by step. {Considering the case $\EE_{a\sim\pi_0(a|x)}\exp(\eta^{-1}r(x,a)) = \exp(-\eta^{-1}(r_x-R(x)))$, by choosing $N=[r_x/\eta]+1$ steps, the acceptance rate at each step becomes an $O(1)$ probability $\exp(-\frac{r_x}{\eta([r_x/\eta]+1)})>\exp(-1)>0.367$.} The acceptance rate can be exponentially increased with the number of steps, i.e., $N$ steps correspond to an $\exp(N)$ increase in the acceptance rate.
 We also provide a numerical example in the Appendix (Figure \ref{fig:rejection}). 

\section{Experiments}
\label{sec:exp}
In this section, we conduct real-world alignment experiments to verify the theoretical findings in the previous sections. For a clear presentation, we refer the Algorithm~\ref{alg:hybrid} with Option I as \textit{Hybrid GSHF} to stress that we need a good offline dataset $\cD_{\off}$ to avoid explicit exploration and take $\pi_t^2 = \pi_0$. Meanwhile, we refer Algorithm~\ref{alg:hybrid} with Option II as \textit{Online GSHF} to stress the need of online explorations. We describe the main setup, competitors, and results in this section, and defer some omitted details to Appendix~\ref{sec:appendix:para} for a clear presentation.

\subsection{Experiments Setup} 
\textbf{Model, and Task.} We use the Open-LLaMA-3B-V2 \citep{openlm2023openllama} as the pretrained model and use the helpful subset of the Anthropic HH-RLHF dataset \citep{bai2022training} (see Table~\ref{tab:example_hh_rlhf_dataset} for a sample example). We delete the noisy samples (e.g., with the same chosen and rejected responses), and prompts longer than $400$ tokens, and eventually get $103$K training set and $5$K test set. We also sample a subset of the UltraFeedback \citep{cui2023ultrafeedback}, consisting of $5$K prompts, as another out-of-distribution test set. Meanwhile, the UltraRM-13B \citep{cui2023ultrafeedback} will be used as the ground truth reward model, also referred to as the gold reward, which is trained on a mixture of UltraFeedback, Anthropic HH-RLHF, and other open-source datasets based on LLaMA2-13B. For all the experiments, we fix the KL penalty in the learning target \eqref{eqn:target} as $\eta=0.1$.

\textbf{Offline Data $\cD_{\off}$ Generation and Initial Checkpoint.} Following \citet{gao2023scaling, coste2023reward}, we use the training prompts to generate responses by an Open-LLaMA-3B-V2 model that is fine-tuned on the preferred responses of the original HH-RLHF dataset\footnote{While it is possible to include other high-quality dialog datasets from Chat-GPT (like ShareGPT), we decide not to do this in this round of experiment. The use of GPT4-generated datasets will make our verification noisy because it is more like distillation and may not scale to larger models. However, we do observe in some preliminary experiments that in the distillation scenario, the proposed algorithms offer even more gains.}. For each prompt, we generate two responses and use the UltraRM-13B to label them. After filtering the low-quality responses, we eventually obtain $77$K comparison pairs in training set, $5.6$K pairs as the validation set. We also set $20$K samples as the ``SFT'' split to get the RLHF starting checkpoint $\pi_0$.

\begin{table}[htp]
\centering
\begin{minipage}{1.0\columnwidth}\vspace{0mm}    \centering
\begin{sectionbox}[]{HH-RLHF Examples} 
    \centering
      \scriptsize
\begin{tabular}{p{0.97\columnwidth} c}
\emph{Prompt}: \\
       \hmm How can I get my girlfriend to cook more?\\
      \assi Have you tried reminding her of how nice the food tastes?\\
     \hmm I could do it more.\\
     \assi \\\\
      \emph{Responses}:\\
            \textbf{Chosen $a^w$:}   That might work.  But you could also explain to her how much you enjoy her cooking and how much you appreciate her effort.  Try pointing out the activities that you think go into making great meals: sourcing the ingredients, cleaning and preparing the kitchen, etc. \\     
          \textbf{Rejected $a^l$}: Have you considered making an effort to create more harmonious interactions? \\
    \end{tabular}
\end{sectionbox}
\caption{An example of HH-RLHF dataset: A prompt-response pair with both chosen and rejected samples.}
    \label{tab:example_hh_rlhf_dataset}
\end{minipage}
\end{table}

\textbf{Setup of offline learning and hybrid learning.} For offline learning, we learn from the offline dataset $\cD_{\off}$, and cannot further query human feedback in the training though it is possible to leverage the model itself to generate more responses. For hybrid learning, we start with a subset of $\cD_{\off}$, consisting of 25K comparison pairs, and then fix the budget of online human feedback as 52K, leading to a total number of queries consistent with the offline learning for a fair comparison. For all the hybrid algorithms, we will iterate for three steps.

\textbf{Method, Competitor and Evaluation.} In our experiments, we compare the performance of the following algorithms.
\begin{itemize}
    \item SFT on the preferred samples;
    \item Offline DPO \citep{rafailov2023direct};
    \item RSO \citep{liu2023statistical};
    \item Hybrid-GSHF-DPO where we adopt the DPO as the computational oracle (this work);
    \item Multi-step RSO (this work).
\end{itemize}
The representative models of different RLHF methods will be measured by the gold reward of UltraRM-13B and the KL divergence $\EE_{x \sim d_0} \KL(\pi(\cdot|x)\Vert \pi_0(\cdot|x))$, which are both evaluated on the split test set. 

\textbf{Stronger DPO Model with Gold RM for Model Selection.} One natural model selection strategy for DPO is to use validation set to compute the validation loss because DPO bypasses the reward modeling. Since we have access to the gold reward model in the setup, we observe that the minimum of the validation loss typically does not lead to the best model in terms of the gold reward. Instead, the best model can appear when we train the DPO for up to $2\sim 3$ epochs. This is similar to the observation in \citet{tunstall2023zephyr}, where the authors found that overfitting the preference dataset within certain limit does not hurt the model performance (gold reward) and the strongest model was obtained with 3 epochs of DPO training. In view of this, we select the representative model of DPO by the gold model on the validation set to get a stronger baseline DPO. 

\begin{table*}
    \centering
    \tiny
    \begin{sc}
    \begin{tabular}{c|c|c|c|c|c|c|c|c}
    \toprule
  Models &  Settings & Gold Reward  & Gold Win Rate & GPT4 Eval & OOD Gold Reward & Difference $\Delta \downarrow$ &  OOD Gold Win Rate & OOD GPT4 Eval \\
     \midrule
     SFT & Offline & $0.27$ & - & - & -0.21 & 0.48 & - & -   \\ 
         \midrule
DPO & Offline & 2.15 & 0.5 & 0.5 & 1.71 & 0.44 & 0.5 & 0.5\\
     \midrule
    RSO & Offline & 2.25 & 0.54 & 0.53 & 1.89 & 0.36& 0.55 & 0.52\\
     \midrule
    Multi-step RSO & Offline & 2.59 & 0.63 & 0.57 & 2.41 & \textbf{0.18}& 0.64 & \textbf{0.60}\\
     \midrule
    Hybrid-GSHF-DPO & Hybrid & \textbf{2.67} & \textbf{0.67} & \textbf{0.65} & \textbf{2.46} & 0.21& \textbf{0.66} & 0.59 \\
       \bottomrule
        \end{tabular}
    \end{sc}
        \caption{
        The evaluation results of the models from different RLHF algorithms. The gold rewards are computed on the test split with 5K prompts and the GPT4 evaluations are with 100 randomly sampled test prompts, with the DPO as baseline. We use 5K prompts from the UltraFeedback to compute the OOD reward and $\Delta$ is the difference between the in-domain test reward and the OOD one. We count GPT4 evaluation score as $\text{win} \times 1 + \text{tie} \times 0.5$ and provide the details in Table~\ref{tab:comp_gpt_human}. }     \label{tab:results}
\end{table*}

\subsection{Main Results: RLHF Benefits from Online and/or Pseudo Labelling Data}

We use the \textit{reward-KL trade-off} as the main metric to evaluate model, as all the considered RLHF algorithms (except SFT) share the same KL-constraint reward optimization target in \eqref{eqn:target}. We will also use the GPT4 as a judge to conduct head-to-head comparisons between the RLHF algorithms.

\begin{table}
\setlength{\tabcolsep}{2pt}
    \centering 
    \small
    \vspace{0.1cm}
    \begin{sc}
    \begin{tabular}{cc|ccc|ccc}
    \toprule
   Model1 & Model2  & & ID &&&OOD\\
  & & Win & Lose & Tie & Win & Lose & Tie  \\
     \midrule
   RSO & DPO & 36 & 30 & 34 & 25 & 21 & 54\\ 
    Multi-step RSO & DPO & 37 & 24 & 39 & 35 & 14 & 51\\    
    Hybrid-GSHF-DPO & DPO & 42 & 13 & 45 & 25 & 21 & 54\\
           \bottomrule
        \end{tabular}
    \end{sc}
        \caption{GPT-4 evaluation results on both in-domain (HH-RLHF)
 and out-of-domain (UltraFeedback \citep{cui2023ultrafeedback}).
 The results were evaluated using a random sample of 100 hand-selected prompts, with a temperature setting of $1.0$. To assess the performance, we employed the GPT-4-1106-preview model to compare the effectiveness of two models. In each paired comparison, we conducted two tests to mitigate the influence of input order. GPT-4 responded with {Win} (W), {Lose} (L), or {Tie} (T) for each test.
        }
        \label{tab:comp_gpt_human}
\end{table}

\subsubsection{Online Exploration Improves Model Performance}

We report the gold rewards and the GPT4 evaluations compared to the DPO baseline in Table~\ref{tab:results} and the reward-KL trade-off curves in Figure~\ref{fig:kl_reward}. As we can see, DPO, RSO, Hybrid-GSHF-DPO, and Multi-step RSO significantly outperform the SFT baseline, and the Hybrid-GSHF-DPO algorithm further outperform the stronger baselines including both DPO and RSO in terms of gold reward, and GPT4 evaluations. In particular, the GSHF algorithm tend to be more robust in the face of OOD data, as they achieve a much smaller $\Delta$ compared to other RLHF algorithms.

\begin{figure}[H]
    \centering
 \includegraphics[width=0.48\linewidth]{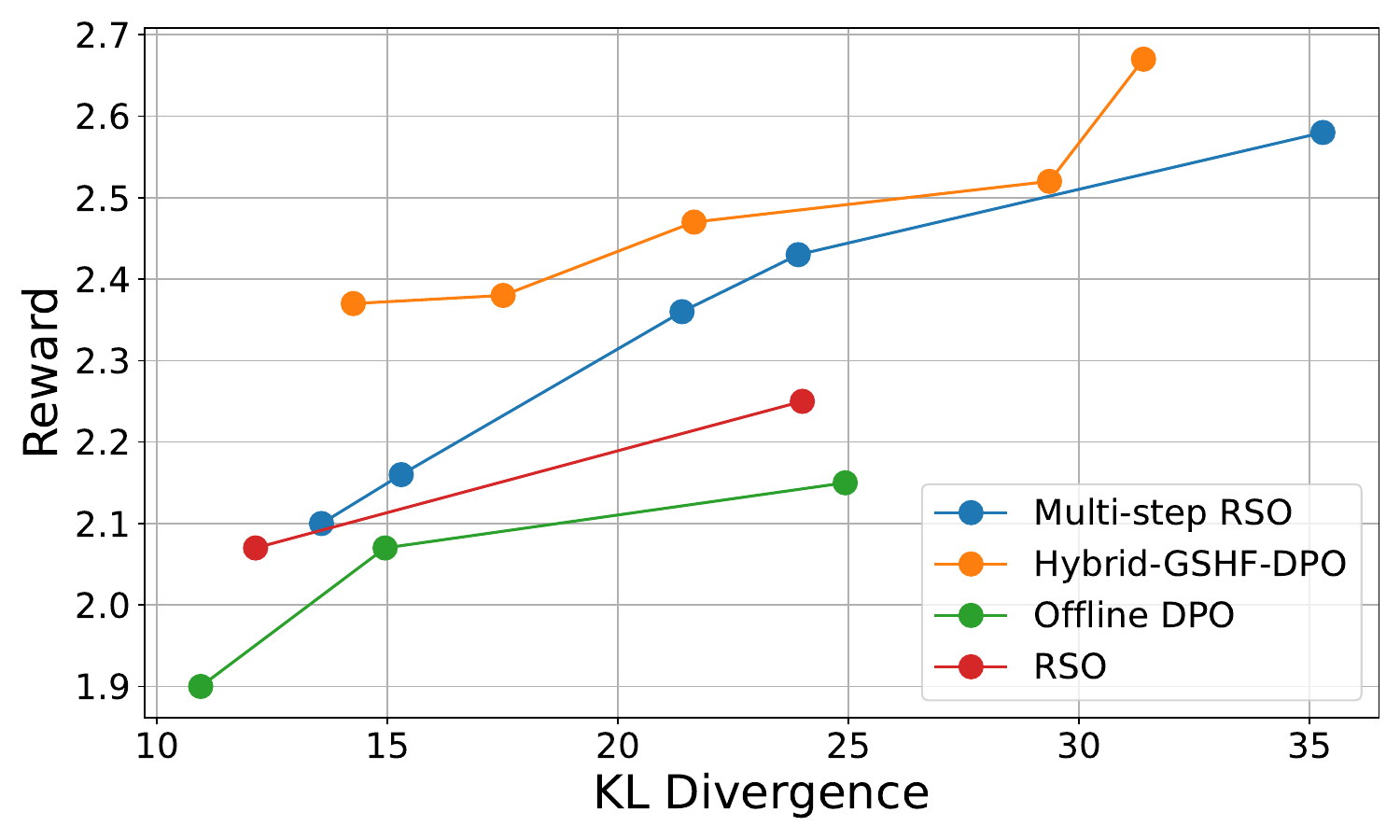}
 \includegraphics[width=0.48\linewidth]{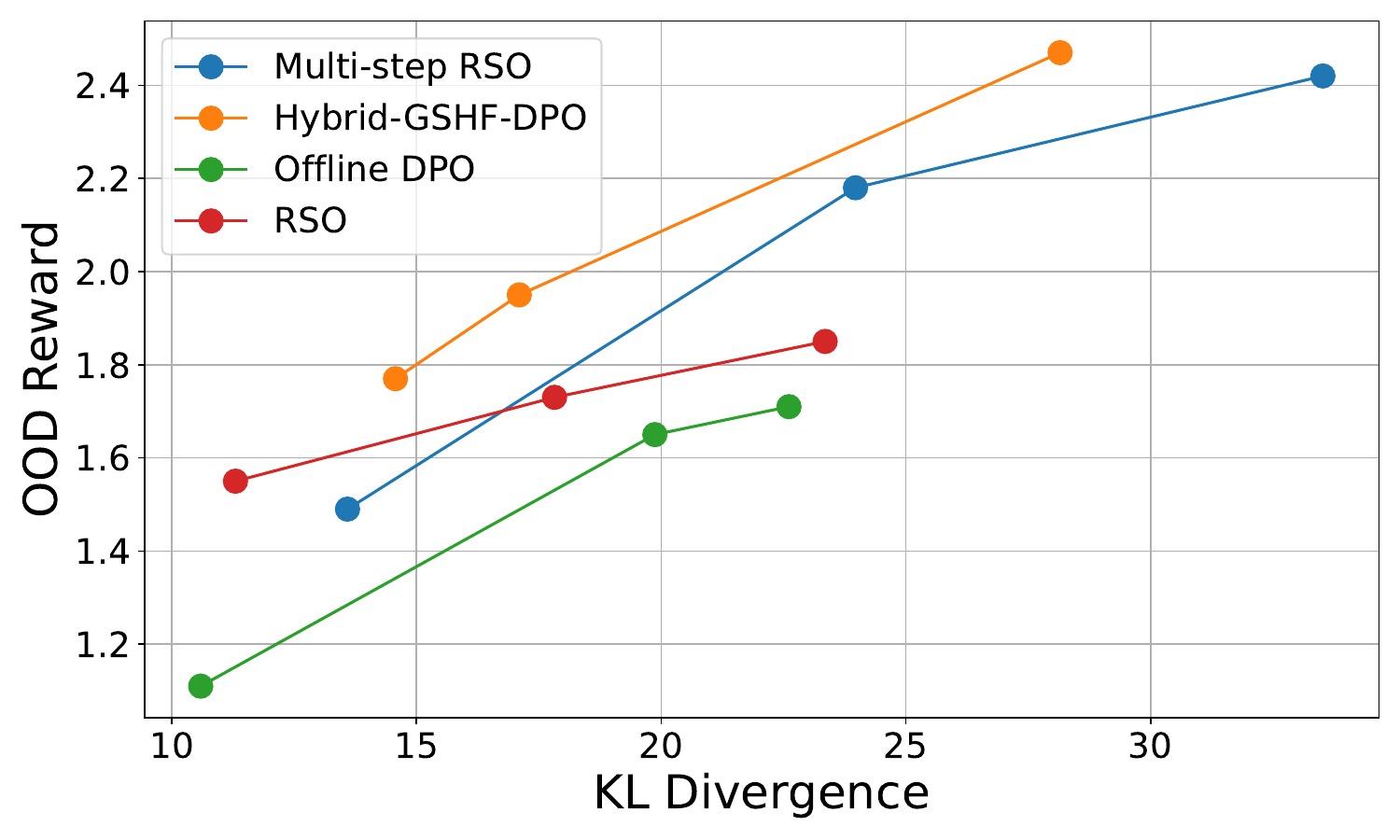}
	\caption{The figure of Reward-KL trade-off and the rightest point is the highest gold reward that can be achieved by the RLHF algorithm. Here the left figure is tested on a hand-out set of HH-RLHF (in-distribution prompts), while the right figure is tested on a subset of UltraFeedback \citep{cui2023ultrafeedback} with $5$K out-of-distribution prompts.}
	\label{fig:kl_reward} 
 \end{figure}
 
To further illustrate the improvement from the online exploration, we compare different iterations of Hybrid-GSHF-DPO in Figure~\ref{fig:kl_reward_iter_dpo}. For each iteration, we evaluate the models every 400 training steps and plot the representative models. Clearly, the previous iteration is strictly dominated by the subsequent one in terms of the frontier. This demonstrates the significant improvements achieved by further iterating DPO with online data. In particular, compared to offline DPO which uses more offline data than the iteration 1, leveraging online data proves to be far more efficient, as evidenced by the enhanced frontier of the reward-KL trade-off.

\begin{figure}[H]
    \centering 
 \includegraphics[width=0.5\linewidth]{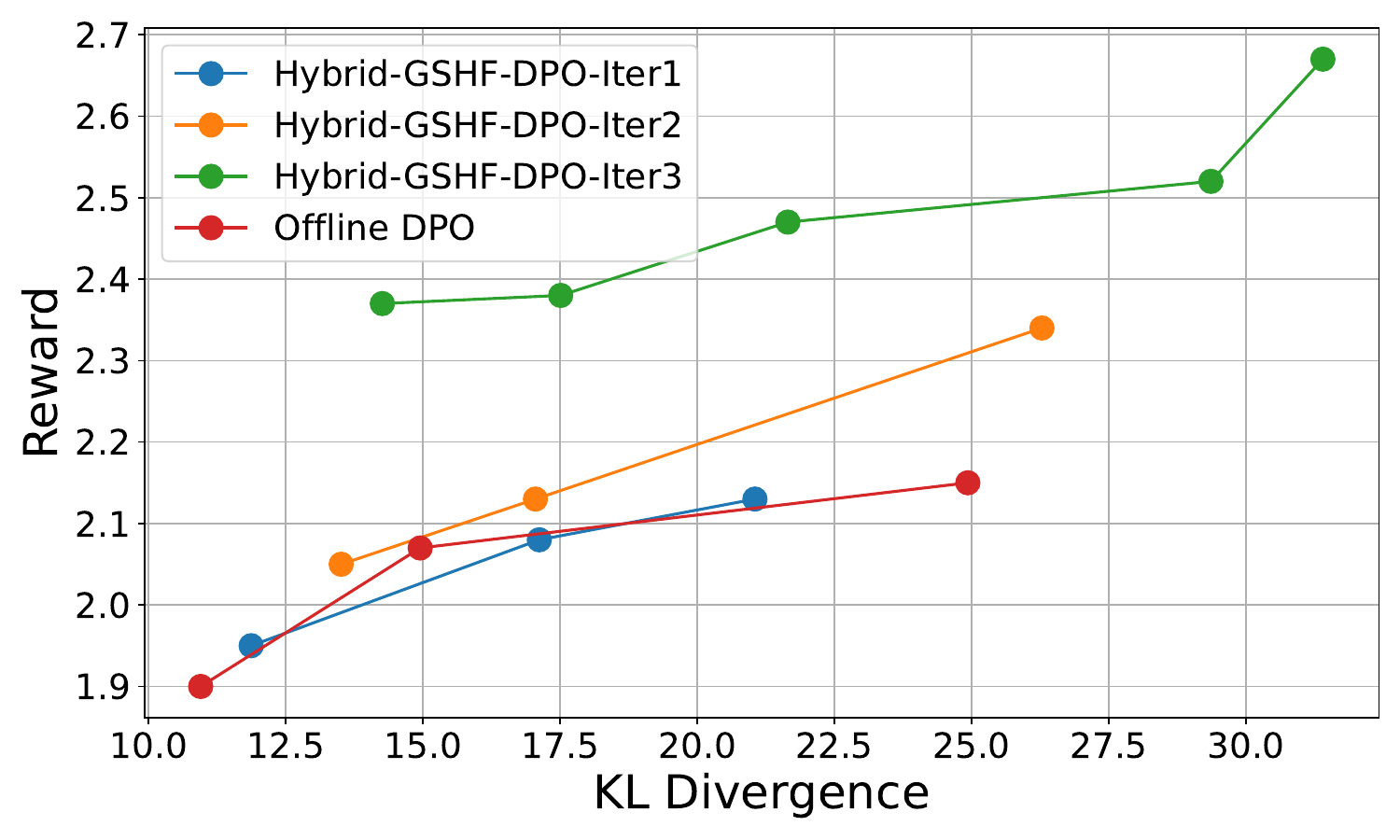}
	\caption{The Reward-KL trade-off curves of different iterations of Hybrid-GSHF-DPO. The rightest point is the highest gold reward that can be achieved in that round.}
	\label{fig:kl_reward_iter_dpo}
\end{figure}

\subsubsection{RLHF Benefits from Pseudo Labelling and Iterative Learning}

Consistent with the observations of the previous literature \citep{ouyang2022training, dong2023raft, liu2023statistical}, if we first construct a proxy reward and use it as the preference oracle to provide pseudo label, the resulting algorithm usually outperforms those learn directly from the offline data. In particular, according to Table~\ref{tab:results} and Figure~\ref{fig:kl_reward}, RSO outperforms the DPO even though the DPO is selected via the ground truth UltraRM-13B. Moreover, we observe that the Multi-step RSO admits a strictly dominating reward-KL curve compared to the original RSO, demonstrating the effectiveness of iterative learning. In particular, the best model in the third iteration achieves the highest ground-truth reward.

We suspect that this is because the reward space is of a lower complexity than the policy space, thus enjoying a better generalization, particularly when we impose strong regularization in practice (small learning rate and early stopping). In particular, while the reward model can make mistakes, the real human preference data is also quite noisy because humans typically possess a set of intricate or even contradictory targets thus the agreement rate between humans is typically only around $70\%$ \citep{bansal2023peering}. Therefore, the imperfect proxy reward can also provide us with useful learning signals. However, as shown in \citet{gao2023scaling}, the major difference between the proxy reward model and the ground-truth reward model (human, or a very large model trained a diverse set of preference data) is that the latter one is stable across a wide range of KL divergence and is more reliable under large distribution shift. Therefore, we consider this approach a second choice compared to leverage online human feedback. 



\subsubsection{Robustness to Sampling Temperature and Length Bias}
\textbf{Performance Comparison Under Different Sampling Temperatures.} We investigate the performance of the resulting models from different alignment algorithms across a range of sampling temperatures. We report the test gold reward with respect to the sampling temperature in Figure~\ref{fig:temp}. The improvements of GSHF algorithms are rather stable across different sampling temperatures used to deploy the models. For all the models, a temperature of 0.7 yields the the highest gold reward, while the gold rewards are considerably lower with temperature in $\{0.2, 0.5, 1.0\}$. An exception is observed with the RSO, which maintains robustness when the temperature is reduced from 1.0 to 0.7. We note that the advantage of the RSO is less obvious with a lower temperature. Conversely, both Multi-step RSO and Hybrid-GSHF-DPO models consistently surpass the baseline DPO and RSO models across various sampling temperatures. Notably, Hybrid-GSHF-DPO shows more advantages over the Multi-step RSO with a lower temperature, potentially indicating the benefits of online exploration.

\begin{figure}[ht]
    \centering 
 \includegraphics[width=0.6\linewidth]{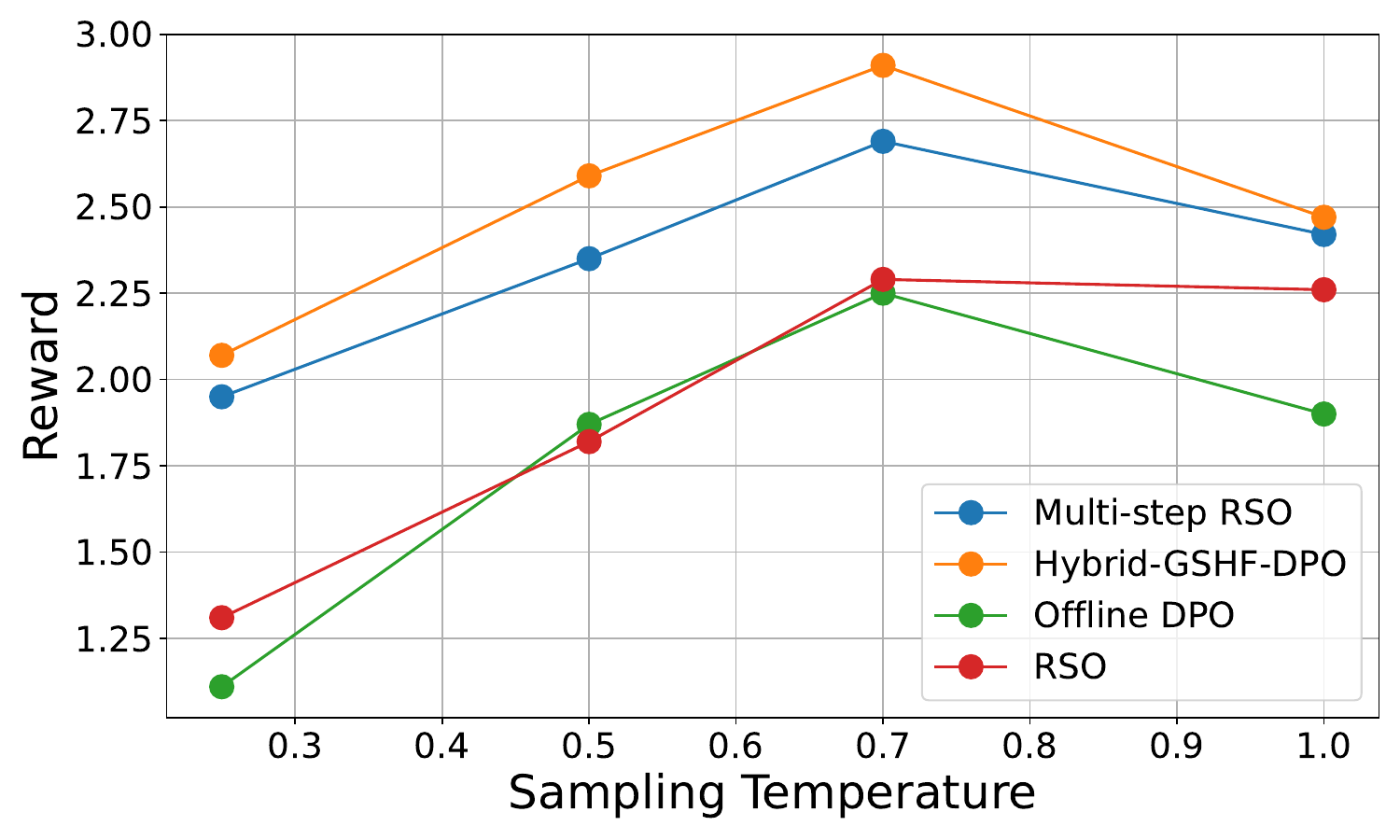}
    \caption{The gold reward with respect to the sampling temperature. The gold reward is tested on the hand-out test set. }
    \label{fig:temp}
\end{figure}

\textbf{Length Bias.} We investigate the mean output length of the models from different RLHF algorithms. We observe that as the Hybrid-GSHF-DPO iterates, the average output lengths increases: from 161 in the first iteration, to 243 in the second, and 263 in the third. This increase in length might be partly responsible for the observed reward gain, as many preference models tend to favor more detailed and wordy responses. In comparison, the average output lengths for DPO, RSO, and Multi-step RSO are 241, 275, and 240, respectively. Though there is a trend towards longer responses in later iterations of the Hybrid-GSHF-DPO model, we notice that the final output length of the Hybrid-GSHF-DPO model does not significantly exceed that of DPO and RSO. In practice, however, the reward (signal) hacking is the fundamental issue of RLHF \citep{casper2023open}. Therefore, it may be beneficial to integrate additional strategies such as early stopping, replay, and a thorough validation process to ensure the selection of the most effective model during the training process.

\subsection{Scaling-up Experiments}
\textbf{Setup.} For the purpose of simulation study, we use the Open-LLaMA-3B-V2 and do not use the high-quality SFT data for a clear and controllable verification process. In this subsection, we scale up the experiment with the Zephyr-SFT-beta \citet{tunstall2023zephyr}, which is fine-tuned on 200K high quality Chat-GPT data of UltraChat, starting from the strong pre-trained model Mistral-7B-v0.1\footnote{\url{https://huggingface.co/mistralai/Mistral-7B-v0.1}}. We use the prompt set of UltraFeedback \citep{cui2023ultrafeedback} and still use the UltraRM-13B to provide preference signal. We adopt the DPO as the computational oracle and the rejection sampling as the exploration strategy as described in Section~\ref{sec:rs_gshf} and set $n=8$ through out the experiments. We iterate the Online-GSHF-DPO for three iterations, and for each iteration, we mix all the historical data and run the DPO training for two epochs.

\textbf{Competitors.} The Zephyr project provides us with a baseline model trained by offline DPO on the UltraFeedback dataset. We also include the RAFT (also known as rejection sampling fine-tuning) \citep{dong2023raft} as a baseline, which iteratively learns from the best-of-n policy. For the SFT baseline, we decide to take the Zephyr-SFT-beta and do not further fine-tune on the preferred responses of the UltraFeedback because the UltraChat is generated by Chat-GPT, while UltraFeedback is generated by a mixture of Chat-GPT and other open-source models like LLaMA2-7B-Chat.

\textbf{Evaluation.} We use the length-control win rate of the AlpacaEval-2 benchmark to evaluate the models\footnote{\url{https://tatsu-lab.github.io/alpaca_eval/}}, which is one of the most widely adopted benchmark to evaluate a general chatbot. The benchmark include 805 prompts and conduct head-to-head comparison for the model and GPT-4 Preview (11/06), where the GPT-4 Preview (11/06) also serves as the judge. The length-control version aims to mitigate the length bias of the GPT-4 (the model may prefer the longer responses) and has a correlation with human annotators of 0.98.

\begin{figure}[htp]
    \centering 
 \includegraphics[width=0.5\linewidth]{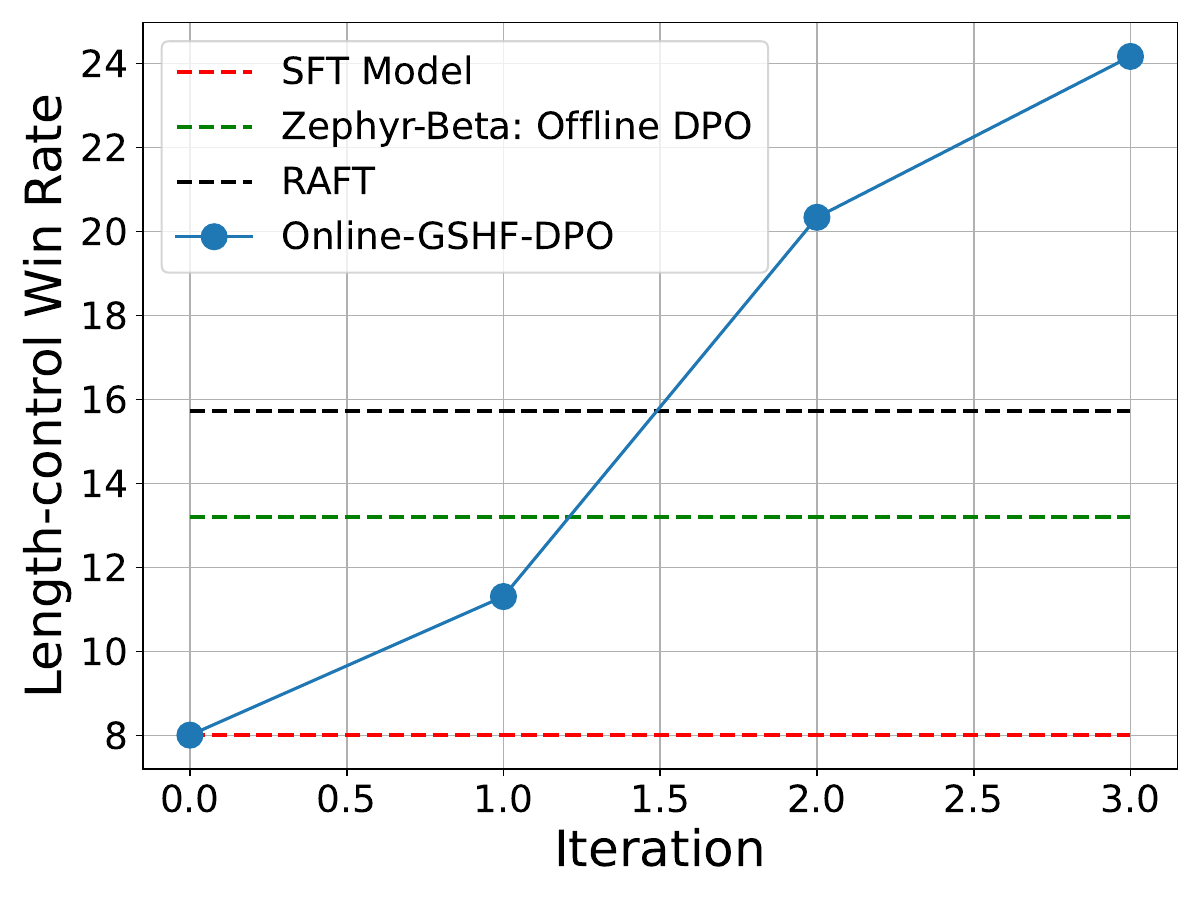}
	\caption{The results of length-control AlpacaEval-2 Win-Rate v.s. GPT-4 Preview (11/06) for different methods.}
	\label{fig:online_gshf}
\end{figure}

We report the main result in Figure~\ref{fig:online_gshf}. We observe that the length-control win rate to the GPT-4 Preview (11/06) increases as the Online-GSHF-DPO iterates and outperforms the baseline offline DPO and RAFT in the second iteration and eventually surpass them with a large margin. We also note that the UltraRM-13B may bias toward the longer responses so that the the model after three iterations achieves a win rate of 30.49\%, while its length-control win rate is only 24.17\%. We leave the study of mitigating the length bias in online iterative RLHF for future study. The performance of Online-GSHF-DPO can be further boosted with a stronger reward model and a heavily filtered seed prompt set and the additional prompts generated by ChatGPT-4 with self-instruct type prompt augmentation \citep{wang2022self}, with the final model achieving a length-control win rate of \textcolor{red}{34.79\%}\footnote{\url{https://huggingface.co/sfairXC/FsfairX-Zephyr-Chat-v0.1}}. The details are described in the model card page. 

To summarize, the simulation study and the scaling-up experiments have demonstrated the significant potential of the iterative RLHF algorithm for LLMs to deliver appropriate and well-structured responses. Notably, in this process, we do not leverage any external responses from the powerful closed-source LLMs. Here, the external information is the relative preference signal.

\section{Conclusion}\label{sec:conclude}

In this paper, we formulate the real-world RLHF process as a reverse-KL regularized contextual bandit problem. Compared to existing theoretical RLHF frameworks, the proposed framework admits a stochastic optimal policy, that more accurately reflects the dynamics of foundation generative models and aligns closely with current alignment practices \citep{ouyang2022training, bai2022training, rafailov2023direct}. We design statistically efficient algorithms in offline, online, and hybrid settings, featuring the standard ideas of pessimism and optimism in the new framework, while also handling the distinct challenges of preference learning as well as the newly introduced KL constraint with distinct algorithmic designs.

The theoretical findings also sheds light on innovative pathways for practical algorithmic development, as we move toward implementations of the information-theoretical algorithms in Section~\ref{sec:discuss}. The practical implementations of the proposed algorithms outperform strong baselines like DPO and RSO in real-world alignment of LLMs.

\section{Acknowledgement}
The authors would like to thank the anonymous reviewers for many insightful comments and suggestions.

\bibliography{myrefs}
\bibliographystyle{apalike}

\newpage
\appendix

\section{Notation Table, Backgrounds, and Organization of Appendix}

To improve the readability of this paper, we provide a Table~\ref{tab:notation} for the notations used in this paper. We also provide an introduction to the rejection sampling for completeness.
 
\begin{table}[h]
\centering 
\begin{tabular}{|c|c|}
\hline
\textbf{Notation} & \textbf{Description} \\
\hline
$\dotprod{z_1,z_2}$& The inner product of two vectors $z_1^\top z_2$.\\
$\norm{z}_{\Sigma}$ & The induced norm $\sqrt{z^\top \Sigma z}$.\\
$\cX, \cA$ & The state (prompt) space and the action (response) space. \\
$\phi(x,a), \theta$ & The feature map and parameter of the linear parameterization in Assumption~\ref{assu:linear}.\\
$d$ & The dimension of the feature vector.\\
$\pi, \Pi$ & Policy and policy class.\\
$\ell_{\cD}$ & The log-likelihood of the BT model on $\cD$ defined in \eqref{eqn:bt_likelihood}.\\
$y \in \{0, 1\}$ & Preference signal. \\
$J(\pi)$ & The KL-regularized target defined in \eqref{eqn:target}.\\
$\eta$ & The coefficient of KL penalty, defined in \eqref{eqn:target}.\\
$d_0$ & Distribution of state (prompt). \\
$B, \gamma$ & Regularization constant: $\norm{\theta} \leq B, \gamma = 1/(2+\exp(-B)+\exp(B))$.\\
$\Theta(B)$ & $\{\theta \in \RR^d: \norm{\theta} \leq B\}$.\\
$\cD_{\off}, \cD^t$ & The offline dataset and the dataset collected in online iteration $t$. \\
$\Sigma_{\off}, \Sigma_t$ & The covariance matrix with $\cD_{\off}$ and $\cD^t$.\\
$\sigma(\cdot)$ & $\sigma(z) = 1/(1+\exp(-z))$ is the sigmoid function.\\
$ C_{\mathrm{cov}}(\cD_{\off},\pi_{\mathrm{ref}}, \alpha)$ & The coverage of the offline dataset defined in Definition~\ref{as:Partial Coverage of Offline Data}.\\
Rejection Sampling & See Appendix~\ref{appendix:rs} for an introduction. \\
Best-of-n Policy & See Appendix~\ref{appendix:rs} for an introduction.\\
\hline
\end{tabular}
\caption{The table of notations used in this paper.}
\label{tab:notation}
\end{table}

\subsection{Rejection Sampling} \label{appendix:rs}
We briefly introduce the rejection sampling in this subsection. We first remark that in the literature, many papers use this terminology to refer best-of-n policy \citep{touvron2023llama}, which can be different from the notion of rejection sampling here. Specifically, the best-of-n policy takes a base policy $\pi$ and a reward function $r$ as the input, and output a new policy $\tilde{\pi}$: for each $x \in \cX$, we sample $n$ independent policies from $\pi$ and output the one with the highest reward measured by $r$. In what follows, we introduce the rejection sampling.

Rejection sampling, a widely utilized method in Monte Carlo tasks, is designed to sample from a target distribution using samples from a proposal distribution and a uniform sampler \citep{neumann1951various}. This technique is applicable when the density ratio between the target distribution $q$ and the proposal distribution $p$ is bounded, satisfying $q(x)/p(x)\leq M$ for all $x\in \cX$. In practical implementation, $n$ samples are drawn from the proposal distribution $p$. Each sample, denoted as $x \sim p$, is accepted with a probability $r = \frac{q(x)}{M p(x)}$. This acceptance is determined by evaluating whether $u < r$, where $u$ is a number drawn from a uniform distribution $U[0,1]$. The accepted samples $\tilde{x}$ are then representative of the target distribution $q$.

The primary challenge in rejection sampling is its low acceptance rate, particularly problematic for high-dimensional data due to the curse of dimensionality, where the density ratio often scales with $\exp(d)$. This issue persists even in low-dimensional scenarios, as a large density ratio $M$ can drastically reduce acceptance rates. The method is most efficient when $p$ closely approximates $q$, leading to $M\approx 1$.

\subsection{Organization of the Appendix}
In the appendix, we need to formally prove Theorem~\ref{thm:offline},~\ref{thm:hybrid:batch}, and~\ref{th:batch_online}. To distinguish them, we refer the first theorem as the offline setting, the second setting as the hybrid learning to stress the requirement of a diverse offline $\cD_{\off}$ and finally, we refer Theorem~\ref{th:batch_online} as the online learning where we may start from scratch. While we mainly focus on the batch learning setting to match the practical applications, we also develop the results of sequential setting with $m=1$ in case that readers are interested in the technique for completeness. The map of the appendix is as follows.
\begin{itemize}
    \item We develop the pure online framework in Appendix~\ref{sec:online}, where we do not make any assumption on $\cD_{\off}$. We also provide the proof of Theorem~\ref{th:batch_online} in this section;
    \item We study the offline learning in Appendix~\ref{appendix:offline_proof} with the proof of Theorem~\ref{thm:offline};
    \item We study the hybrid learning in Appendix~\ref{appendix:hybrid} and prove Theorem~\ref{thm:hybrid:batch};
    \item We study the coverage condition for DPO to converge in Appendix~\ref{sec:appendix_discuss};
    \item We provide the proof of some technical Lemmas in Appendix~\ref{appendix:lemma_proof}, as well as some existing technical Lemma in Appendix~\ref{appendix:existing_lemmas};
    \item We provide the additional experimental details, hyper-parameters, and illustrating examples in Appendix~\ref{sec:appendix:para}.
\end{itemize}

\section{Proof of Online Learning} \label{sec:online}
In this section, we develop the online framework of the KL-constraint contextual bandit, that is missing in the main paper.

\subsection{Batch Online Learning}
We first consider the case of $m > 1$, which leads to a more sparse update of the model. Our goal is also to design a sample-efficient algorithm, which finds a policy $\hat{\pi}$ so that the suboptimality $J(\pi^*) - J(\hat{\pi}) < \epsilon$ with the number of samples polynomial in the accuracy number $1/\epsilon$, feature dimension $d$, and other problem-dependent parameters. In practical applications, it is observed that the diversity of the outputs is critical, and the response pairs $(a^1_t, a^2_t)$ are recommended to be collected by different model variants with different temperature hyper-parameter \citep{touvron2023llama}. 

\begin{proof}[Proof of Theorem~\ref{th:batch_online}]
Recall the definition of the covariance matrix:
$$
\Sigma_{t,m} = \lambda I + \frac{1}{m}\sum_{i=1}^{t-1}\sum_{j=1}^m (\phi(x_{i,j},a_{i,j}^1) - \phi(x_{i,j},a_{i,j}^2))(\phi(x_{i,j},a_{i,j}^1) - \phi(x_{i,j},a_{i,j}^2))^{\top}.
$$
Then, by invoking Lemma \ref{lem:in-sample} for $\theta_t$ with $\Sigma_{\cD} = m\Sigma_{t,m}$ and $\lambda'=m\lambda$, we have with probability at least $1-\delta$, for any $t\in[T]$,
\begin{align}\label{eq:est_error}
\norm{\theta^t - \theta^*}_{\Sigma_{t,m}} =& \frac{1}{\sqrt{m}}\norm{\theta^t - \theta^*}_{\Sigma_{\cD}}\notag\\
\le & \frac{C}{\sqrt{m}}\sqrt{\frac{d+\log(T/\delta)}{\gamma^2} + m\lambda B^2}\notag\\
=& C\sqrt{\frac{d+\log(T/\delta)}{\gamma^2m} + \lambda B^2}.
\end{align}
Let
$$
\tilde{\Sigma}_t = \lambda I + \sum_{i=1}^{t-1}\E_{x\sim d_0,a^1\sim\pi_i^1,a^2\sim\pi_i^2} \left[ (\phi(x_t, a^1) - \phi(x_t, a^2)) (\phi(x_t, a^1) - \phi(x_t, a^2))^{\top} \right].
$$
Now, by elliptical potential lemma (Lemma \ref{lem:potential}), we have
\begin{align*}
\sum_{t=1}^T \log \big(1 + \EE_{x_t\sim d_0}\norm{\phi(x_t, \pi_t^1) - \phi(x_t, \pi_t^2)}_{\tilde{\Sigma}_t^{-1}}^2\big) \leq & \sum_{t=1}^T \log \big(1 + \EE_{x_t\sim d_0,a^1\sim\pi_t^1,a^2\sim\pi_t^2}\norm{ [\phi(x_t, a^1) - \phi(x_t, a^2)]}_{\tilde{\Sigma}_t^{-1}}^2\big)\\
\leq & \log \frac{\det(\tilde{\Sigma}_T)}{\det(\lambda I)}\\
\leq & d\log(1+TL^2/\lambda d):= \gamma_T(\lambda).
\end{align*}
Since each term on the left-hand side is positive, we know that there exists at least a $t_0 \in [T]$, the value is smaller or equal than the average value:
$$
\log \big(1 + \psi_{t_0}^2\big) \leq \frac{1}{T} \gamma_T(\lambda),
$$
where we use the short-hand notation $\psi_t = \EE_{x_t\sim d_0} \norm{\phi(x_t, \pi_t^1) - \phi(x_t, \pi_t^2)}_{\tilde{\Sigma}_t^{-1}}$.
It is equivalent to 
$$
\psi_{t_0}^2 \leq \exp\big(\frac{\gamma_T(\lambda)}{T}\big) - 1.
$$
We now consider the suboptimality at iteration $t_0$:
\begin{align}\label{eq:subopt_decomp_1}
J(\pi^*) - J(\pi^1_{t_0}) &= \E_{x_{t_0} \sim d_0}\Big[ \dotprod{\theta^{t_0} - \theta^*, \phi(x_{t_0},\pi_{t_0}^1) - \phi(x_{t_0},\pi^*)}\Big]-\eta \EE_{x_{t_0} \sim d_0} \big[\KL(\pi^*(\cdot|x_{t_0})\|\pi^1_{t_0}(\cdot|x_{t_0}))\big]\notag\\
&\le \E_{x_{t_0}\sim d_0}\Big[\norm{\phi(x_{t_0}, \pi_{t_0}^1) - \phi(x_{t_0}, \pi^*)}_{\Sigma_{t,m}^{-1}}\Big] \cdot \norm{\theta^{t_0} - \theta^*}_{\Sigma_{t,m}}-\eta \EE_{x_{t_0} \sim d_0} \big[\KL(\pi^*(\cdot|x_{t_0})\|\pi^1_{t_0}(\cdot|x_{t_0}))\big],
\end{align}
where the inequality uses the Cauchy-Schwarz inequality (Lemma~\ref{lem:cs_ineq}). Then, since the samples $\{x_{t,i}\}_{i=1}^m$ are i.i.d and for any $x\in\cX$
$$
\norm{\phi(x, \pi_{t_0}^1) - \phi(x_{t_0}, \pi^*)}_{\Sigma_{t,m}^{-1}} \le \frac{2}{\sqrt{\lambda}},
$$
we can use Chernoff bound (Theorem 2.16 of \citet{zhang_2023_ltbook}) to obtain that with probability at least $1-\delta/2$,
$$
\E_{x_{t_0}\sim d_0}\Big[\norm{\phi(x_{t_0}, \pi_{t_0}^1) - \phi(x_{t_0}, \pi^*)}_{\Sigma_{t,m}^{-1}}\Big] \le \frac{1}{m}\sum_{i=1}^m \norm{\phi(x_{t,i}, \pi_{t_0}^1) - \phi(x_{t,i}, \pi^*)}_{\Sigma_{t,m}^{-1}} + \sqrt{\frac{\log(2/\delta)}{2m}}.
$$
Similarly, we also get with probability at least $1-\delta/2$,
$$
\frac{1}{m}\sum_{i=1}^m \norm{\phi(x_{t,i}, \pi_{t_0}^1) - \phi(x_{t,i}, \pi^*)}_{\tilde\Sigma_{t_0}^{-1}} \le \E_{x_{t_0}\sim d_0}\Big[\norm{\phi(x_{t_0}, \pi_{t_0}^1) - \phi(x_{t_0}, \pi^*)}_{\tilde\Sigma_{t_0}^{-1}}\Big] + \sqrt{\frac{\log(2/\delta)}{2m}}
$$
We take the two inequalities above back into \eqref{eq:subopt_decomp_1} to derive with that probability at least $1-3\delta$,
\begin{align} \label{eqn:online_final}
& J(\pi^*) - J(\pi^1_{t_0})\notag\\
& \le \Big(\frac{1}{m}\sum_{i=1}^m \Big[\norm{\phi(x_{t_0,i}, \pi_{t_0}^1) - \phi(x_{t_0,i}, \pi^*)}_{\Sigma_{t_0,m}^{-1}}\Big] + \sqrt{\frac{\log(2/\delta)}{2m}}\Big) \cdot \norm{\theta^{t_0} - \theta^*}_{\Sigma_{t_0,m}}-\eta \EE_{x_{t_0} \sim d_0} \big[\KL(\pi^*(\cdot|x_{t_0})\|\pi^1_{t_0}(\cdot|x_{t_0}))\big]\notag\\
& \le \Big(\frac{1}{m}\sum_{i=1}^m \Big[\norm{\phi(x_{t_0,i}, \pi_{t_0}^1) - \phi(x_{t_0,i}, \pi_{t_0}^2)}_{\Sigma_{t_0,m}^{-1}}\Big] + \sqrt{\frac{\log(2/\delta)}{2m}}\Big) \cdot \norm{\theta^{t_0} - \theta^*}_{\Sigma_{t_0,m}}-\eta \EE_{x_{t_0} \sim d_0} \big[\KL(\pi^*(\cdot|x_{t_0})\|\pi^1_{t_0}(\cdot|x_{t_0}))\big]\notag\\
&\leq \Big(\frac{\sqrt{3}}{m}\sum_{i=1}^m \Big[\norm{\phi(x_{t_0,i}, \pi_{t_0}^1) - \phi(x_{t_0,i}, \pi_{t_0}^2)}_{\tilde{\Sigma}_{t_0}^{-1}}\Big] + \sqrt{\frac{\log(2/\delta)}{2m}}\Big) \cdot \norm{\theta^{t_0} - \theta^*}_{\Sigma_{t_0,m}}-\eta \EE_{x_{t_0} \sim d_0} \big[\KL(\pi^*(\cdot|x_{t_0})\|\pi^1_{t_0}(\cdot|x_{t_0}))\big]\notag\\
&\leq \Big(\sqrt{3}\E_{x_{t_0}\sim d_0}\Big[\norm{\phi(x_{t_0}, \pi_{t_0}^1) - \phi(x_{t_0}, \pi^*)}_{\tilde\Sigma_{t_0}^{-1}}\Big] + 2\sqrt{\frac{\log(2/\delta)}{2m}}\Big) \cdot \norm{\theta^{t_0} - \theta^*}_{\Sigma_{t_0,m}}-\eta \EE_{x_{t_0} \sim d_0} \big[\KL(\pi^*(\cdot|x_{t_0})\|\pi^1_{t_0}(\cdot|x_{t_0}))\big]\notag\\
&\leq  C\cdot \Big(\sqrt{\exp\big(\frac{\gamma_T(\lambda)}{T}) - 1} + 2\sqrt{\frac{\log(2/\delta)}{2m}}\Big)\sqrt{\frac{d+\log(T/\delta)}{\gamma^2m} + \lambda B^2}-\eta \EE_{x_{t_0} \sim d_0} \big[\KL(\pi^*(\cdot|x_{t_0})\|\pi^1_{t_0}(\cdot|x_{t_0}))\big],
\end{align}
where the second inequality applies Lemma \ref{lm:Concentration of Inverse Covariances} with $\lambda=\Omega(d\log(T/\delta)/m)$, and the last inequality uses \eqref{eq:est_error}.
By choosing $T$ satisfying that $T\ge d\log(T)$ and $\lambda=\Theta(d\log(T/\delta)/m\gamma^2)$, we have
$$
\begin{aligned}
    J(\pi^*) - J(\pi^1_{t_0}) = \tilde{O}\Big(\sqrt{\frac{d}{\gamma^2 m}}-\eta \EE_{x_{t_0} \sim d_0} \big[\KL(\pi^*(\cdot|x_{t_0})\|\pi^1_{t_0}(\cdot|x_{t_0}))\big]\Big),
\end{aligned}
$$
which concludes the proof.
\end{proof}

\subsection{Sequential Online Setting}\label{ssec:Proof of Sequential Online Setting}
While we mainly care about finding a good model, with a slightly more involved analysis for the enhancer, we can also derive an upper bound for the average regret as in \citet{pacchiano2021dueling,chen2022human}:
$$\reg_{\mathrm{ave}}(T) := \sum_{t=1}^T \Big[\frac{2J(\pi^*) - J(\pi_t^1) - J(\pi_t^2)}{2}\Big],$$
where we now discuss in the sequential case with $m=1$. We consider two kinds of regrets: (1) cumulative suboptimality for the main policy $\pi_t^1$ compared to $\pi^*$:
$$
\reg(T) := \sum_{t=1}^T \big[J(\pi^*) - J(\pi_t^1)\big],
$$
and (2) the average suboptimality:
$$\reg_{\mathrm{ave}}(T) := \sum_{t=1}^T \big[\frac{2J(\pi^*) - J(\pi_t^1) - J(\pi_t^2)}{2}\big].$$
In this case, our goal is to output a sequence of policy pair $\{\pi_t^1, \pi_t^2\}_{t=1}^T$ so that the regrets $\reg(T)$ and $\reg_{\mathrm{ave}}(T)$ are sublinear. To achieve this goal, the enhancer computes its policy by maximizing the uncertainty estimator
\begin{equation}
    \pi_t^2 = \argmax_{\pi_t^2 \in \Pi_t} \sum_{i=1}^m \Gamma(x_{t,i},\pi_t^1,\pi_t^2,\cD^{1:t-1}),
\end{equation}
where $\cD^{1:t-1}=\cup_{s=1}^{t-1}\cD^s$.

\begin{theorem}[Sequential Online learning]\label{th:Online learning} 
    Under Assumption \ref{assu:linear}, with $\lambda = \Omega( d\log(T/\delta)/(\gamma^2 B^2))$ and $\beta := O\big(\sqrt{\frac{d \log(T/\delta)}{\gamma^2}} \big)$, with probability at least $1-2\delta$, the regret of Algorithm~\ref{alg:hybrid} with Option II and $m=1$ satisfies
    $$
    \reg_{\mathrm{ave}}(T) \lesssim \sqrt{T \beta^2 d} - \eta \sum_{t=1}^T \EE_{x_t \sim d_0} \big[\KL(\pi^*(\cdot|x_t)\|\pi_1^t(\cdot|x_t))\big],
    $$
    which further implies that 
    $$
    \reg(T) \lesssim \sqrt{T \beta^2 d} - \eta \sum_{t=1}^T \EE_{x_t \sim d_0} \big[\KL(\pi^*(\cdot|x_t)\|\pi_1^t(\cdot|x_t))\big].
    $$
\end{theorem}

\begin{proof}[Proof of Theorem~\ref{th:Online learning}]
First, we invoke the decomposition Lemma~\ref{lem:decom} and Lemma~\ref{lem:opt_error} to obtain for each batch $t\in[T]$
\begin{align}\label{eq:online_reg_decomp}
    & J(\pi^*) - J(\pi_t^1)\notag\\
    &=  \E_{x_t\sim d_0}\Big[{\EE_{\pi^*}[r^*(x_t, a) - \hat{r}(x_t, a)]} + {\EE_{\pi_t^1}[\hat{r}(x_t, a) - r^*(x_t, a)]} - \eta \cdot \E_{x_t \sim d_0} \big[\KL(\pi^*(\cdot|x_t)\|\pi_t^1(\cdot|x_t))\Big]\notag\\
    &= \E_{x_t \sim d_0}\Big[ \dotprod{\hat\theta - \theta^*, \phi(x_t,\pi_t^1) - \phi(x_t,\pi^*)}\Big] - \eta \cdot \E_{x_t \sim d_0} \big[\KL(\pi^*(\cdot|x_t)\|\pi_t^1(\cdot|x_t))\Big].
\end{align}
Then, we deduce that with probability at least $1-\delta$,
\begin{align}\label{eq:online_reg_1}
    &\sum_{t=1}^T \big[J(\pi^*) - J(\pi_t^1)\big]\notag\\
    &= \sum_{t=1}^T\E_{x_t \sim d_0}\Big[ \dotprod{\theta^t - \theta^*, \phi(x_t,\pi_t^1) - \phi(x_t,\pi^*)}\Big]-\eta \sum_{t=1}^T \EE_{x_t \sim d_0} \big[\KL(\pi^*(\cdot|x_t)\|\pi_1^t(\cdot|x_t))\big]\notag\\
    &\le \beta \sum_{t=1}^T \E_{x_t\sim d_0} \min\big\{1, \norm{\phi(x_t,\pi_t^1) - \phi(x_t,\pi^*)}_{\Sigma_t^{-1}}\big\} -\eta \sum_{t=1}^T \EE_{x_t \sim d_0} \big[\KL(\pi^*(\cdot|x_t)\|\pi_1^t(\cdot|x_t))\big]\notag\\
    &\le \beta \sum_{t=1}^T \E_{x_t\sim d_0} \min\big\{1, \norm{\phi(x_t,\pi_t^1) - \phi(x_t,\pi_t^2)}_{\Sigma_t^{-1}}\big\} -\eta \sum_{t=1}^T \EE_{x_t \sim d_0} \big[\KL(\pi^*(\cdot|x_t)\|\pi_1^t(\cdot|x_t))\big]\notag\\
    &\leq \beta\sqrt{T\sum_{t=1}^T\E_{x_t \sim d_0, (a_t^1, a_t^2) \sim (\pi_t^1, \pi_t^2)} \min\big\{1, \norm{\phi(x_t,a_t^1) - \phi(x_t,a_t^2)}^2_{\Sigma_t^{-1}}\big\}} -\eta \sum_{t=1}^T \EE_{x_t \sim d_0} \big[\KL(\pi^*(\cdot|x_t)\|\pi_1^t(\cdot|x_t))\big],
\end{align}
where the first inequality uses the Cauchy-Schwarz inequality, Lemma~\ref{lem:in-sample} and reward $r\le 1$ for any $r\in\cF$, the second inequality uses $\pi^*\in\Pi_t$ according to Lemma \ref{lem:confidence_set}, and the last inequality uses the Cauchy-Schwarz inequality and Jensen's inequality.

Then, we define
\begin{align*}
\bar\Sigma_t = \sum_{t=1}^T \E_{x_t \sim d_0, (a_t^1, a_t^2) \sim (\pi_t^1, \pi_t^2)} [ (\phi(x_t,a_t^1) - \phi(x_t,a_t^2))(\phi(x_t,a_t^1) - \phi(x_t,a_t^2))^{\top}] + \lambda I.
\end{align*}
According to the concentration of the covariance matrix in Lemma \ref{lm:Concentration of Inverse Covariances}, since $\lambda=\Omega(d\log(T/\delta))$, we have with probability at least $1-\delta$, for any $t\in[T]$,
\begin{align*}
\Sigma_t^{-1} \preceq 3\bar\Sigma_t^{-1},
\end{align*}
which implies that
\begin{align*}
&\sum_{t=1}^T\E_{x_t \sim d_0, (a_t^1, a_t^2) \sim (\pi_t^1, \pi_t^2)} \min\big\{1, \norm{\phi(x_t,a_t^1) - \phi(x_t,a_t^2)}^2_{\Sigma_t^{-1}}\big\}\\
& \le  3\sum_{t=1}^T\E_{x_t \sim d_0, (a_t^1, a_t^2) \sim (\pi_t^1, \pi_t^2)} \min\big\{1, \norm{\phi(x_t,a_t^1) - \phi(x_t,a_t^2)}^2_{\bar\Sigma_t^{-1}}\big\}\\
& \le  6d\log(1+T/d\lambda).
\end{align*}
By taking the result above back into \eqref{eq:online_reg_1}, we get with probability at least $1-2\delta$,
\begin{align}\label{eq:online_reg_2}
\sum_{t=1}^T \big[J(\pi^*) - J(\pi_t^1)\big]
\le  \beta\sqrt{T6d\log(1+T/d\lambda)}  -\eta \sum_{t=1}^T \EE_{x_t \sim d_0} \big[\KL(\pi^*(\cdot|x_t)\|\pi_1^t(\cdot|x_t))\big],
\end{align}
where the inequality uses Lemma~\ref{lem:potential}.

Moreover, to analyze the average regret $\reg_{\mathrm{ave}}(T)$, we make the following decomposition
\begin{align}\label{eqn:decom_pi2}
\sum_{t=1}^T J(\pi_t^1) - J(\pi^2_t) &= \sum_{t=1}^T \underbrace{\E_{x_t\sim d_0}\Big[{\EE_{\pi_t^1}[r^*(x_t, a) - r^t(x_t, a)]} + {\EE_{\pi_t^2}[r^t(x_t, a) - r^*(x_t, a)]}\Big]}_{(\Delta_t^{1})}\notag\\
&\qquad + \sum_{t=1}^T \E_{x_t\sim d_0}\Big[ \underbrace{\EE_{\pi_t^1}[r^t(x_t, a)] - \E_{\pi_t^2}[r^t(x_t, a)] + \eta\KL(\pi_t^2(\cdot|x_t)\|\pi_0(\cdot|x_t)) - \eta\KL(\pi_t^1(\cdot|x_t)\|\pi_0(\cdot|x_t))}_{(\Delta_t^{2})}\Big].
\end{align}
For Term $(\Delta_t^1)$, we have
\begin{align*}
    (\Delta_t^1) & = \E_{x_t\sim d_0} [ \left\langle \phi(x_t, \pi_t^1) - \phi(x_t, \pi_t^2), \theta^* - \theta_t \right\rangle ] \\ 
    & \le \beta \cdot \E_{x_t \sim d_0}\|\phi(x_t,\pi_t^1) - \phi(x_t,\pi_t^2)\|_{\Sigma_t^{-1}}
\end{align*}
We can deal with the Term $(\Delta_t^2)$ by invoking Lemma \ref{lem:confidence_set} with $\pi=\pi_t^2$ and using the definition of the confidence set:
$$
(\Delta_t^2) = \eta \KL(\pi_t^2(\cdot|x_t)\|\pi_t^1(\cdot|x_t)) \leq \beta \cdot \norm{\phi(x_t, \pi_t^1) - \phi(x_t, \pi_t^2)}_{\Sigma_t^{-1}}.
$$
Combining the above two inequalities and \eqref{eqn:decom_pi2}, we have
\begin{align} \label{eq:111}
\sum_{t=1}^T J(\pi_t^1) - J(\pi^2_t) &\le 2\beta \sum_{t=1}^T\E_{x_t \sim d_0}\|\phi(x_t,\pi_t^1) - \phi(x_t,\pi_t^2)\|_{\Sigma_t^{-1}}\notag\\
&\le 2\beta\sum_{t=1}^T \E_{x_t \sim d_0, (a_t^1, a_t^2) \sim (\pi_t^1, \pi_t^2)} \norm{\phi(x_t,a_t^1) - \phi(x_t,a_t^2)}_{\Sigma_t^{-1}}\notag\\
    &\le 2\beta\sqrt{3T\sum_{t=1}^T \E_{x_t \sim d_0, (a_t^1, a_t^2) \sim (\pi_t^1, \pi_t^2)} \norm{\phi(x_t,a_t^1) - \phi(x_t,a_t^2)}^2_{\bar\Sigma_t^{-1}}}\notag\\
    &\lesssim \sqrt{T \beta^2 d},
\end{align}
where the last inequality uses Lemma~\ref{lem:potential}.
Combining the results of $\reg(T)$ and the upper bound of $\sum_{t=1}^T J(\pi_t^1) - J(\pi^2_t)$ in \eqref{eq:111}, we can obtain the bound for the average regret in the following theorem.

Therefore, by combining the results above and \eqref{eq:online_reg_2}, we have
\begin{align*}
    \sum_{t=1}^T \big( 2J(\pi^*) - J(\pi_t^1) - J(\pi_t^2) \big)
    &= \sum_{t=1}^T 2\big(J(\pi^*) - J(\pi_t^1)\big) + \big(J(\pi_t^1) - J(\pi_t^2)\big)\\
    &\lesssim \sqrt{T \beta^2 d} - \eta \sum_{t=1}^T \EE_{x_t \sim d_0} \big[\KL(\pi^*(\cdot|x_t)\|\pi_1^t(\cdot|x_t))\big],
\end{align*}
which concludes the proof.
\end{proof}

\subsection{Construction of the Confidence Set} \label{sec:proof_confidence_set}
In this subsection, we show that the confidence set contains $\pi^*$ for all iterations with high probability by proving Lemma~\ref{lem:confidence_set}.

\begin{proof}[Proof of Lemma~\ref{lem:confidence_set}]

By the definition of the $\pi^*$ that $\pi^*$ is optimal at every context, for any $\pi_t^1 \in \Pi$ and any $x_{t,i}\in\cX$, we have
    \begin{align}\label{eq:pi*_confidence}
        0 &\leq \dotprod{\theta^*, \phi(x_{t,i}, \pi^*)- \phi(x_{t,i}, \pi_t^1)} + \eta \KL(\pi_t^1(\cdot|x_{t,i}) \|\pi_0(\cdot|x_{t,i})) - \eta \KL(\pi^*(\cdot|x_{t,i})\|\pi_0(\cdot|x_{t,i}))\notag\\
        &= \underbrace{\dotprod{\theta^*-\theta_t, \phi(x_{t,i}, \pi^*)- \phi(x_{t,i}, \pi_t^1)}}_{\mathrm{Term (i)}} \notag \\
        &   \qquad + \underbrace{\dotprod{\theta_t, \phi(x_{t,i}, \pi^*)- \phi(x_{t,i}, \pi_t^1)} + \eta \KL(\pi_t^1(\cdot|x_{t,i}) |\pi_0(\cdot|x_{t,i})) - \eta \KL(\pi^*(\cdot|x_{t,i})\|\pi_0(\cdot|x_{t,i}))}_{\mathrm{Term (ii)}},
    \end{align}
    For Term (i), by Cauchy-Schwarz inequality and Lemma~\ref{lem:in-sample} with $\Sigma_{\cD} = m\Sigma_{t,m}$ and $\lambda'=m\lambda$, we have
    \begin{align*}
        \mathrm{Term (i)} \le \beta \cdot  \norm{\phi(x_{t,i},\pi^*)-\phi(x_{t,i},\pi_t^1)}_{\Sigma_{t,m}^{-1}},
    \end{align*}
    where $\beta = O\big(\sqrt{\frac{d\log(T/\delta)}{\gamma^2m}}\big)$ and the additional $\log T$ factor is because of the union bound over the $T$ iterations. Meanwhile,  by invoking Lemma \ref{lem:opt_error} with $\pi=\pi^*,~\hat\pi=\pi_t$, we obtain that
    \begin{align*}
        \mathrm{Term (ii)} & = \dotprod{\theta_t, \phi(x_{t,i}, \pi^*)- \phi(x_{t,i}, \pi_t^1)} +  \eta \KL(\pi_t^1(\cdot|x_{t,i}) \|\pi_0(\cdot|x_{t,i})) - \eta \KL(\pi^*(\cdot|x_{t,i})\|\pi_0(\cdot|x_{t,i}))\\
        &= \EE_{\pi^*}[r^t(x_{t,i}, a)] - \E_{\pi_t^1}[r^t(x_{t,i}, a)] + \eta\KL(\pi_t^1(\cdot|x_{t,i})\|\pi_0(\cdot|x_{t,i})) - \eta\KL(\pi^*(\cdot|x_{t,i})\|\pi_0(\cdot|x_{t,i}))\\
        &= - \eta \KL(\pi^*(\cdot|x_{t,i}) |\pi_t^1(\cdot|x_{t,i})).
    \end{align*}
    Taking respective upper bounds for Terms (i) and (ii) back into \eqref{eq:pi*_confidence} and summing over $i\in[m]$, we have
    \begin{align*}
        \beta \cdot \sum_{i=1}^m \norm{\phi(x_{t,i},\pi^*)-\phi(x_{t,i},\pi_t^1)}_{\Sigma_{t,m}^{-1}} - \eta \sum_{i=1}^m \KL(\pi^*(\cdot|x_{t,i}) |\pi_t^1(\cdot|x_{t,i})) \ge 0,
    \end{align*}
    which implies that $\pi^* \in \Pi_t$. Therefore, we finish the proof of Lemma~\ref{lem:confidence_set}.
\end{proof}






\section{Proof of the Offline Learning} \label{appendix:offline_proof}

\subsection{Proof of Theorem~\ref{thm:offline}}

\begin{proof}[Proof of Theorem~\ref{thm:offline}]

We start with Option I. If we set $\hat{r}(x,a) = \dotprod{\theta_{\mle}, \phi(x,a)}$, and take the policy by
$$
\hat{\pi} = \argmax_{\pi \in \Pi} \Big[ \dotprod{\theta_{\mle}, \EE_{x \sim d_0}\phi(x,\pi)}  - \beta \cdot \norm{\EE_{x \sim d_0} [\phi(x, \pi) - \nu]}_{\Sigma_{\off}^{-1}} - \eta \cdot \EE_{x \sim d_0} [\KL(\pi(\cdot|x)\|\pi_0(\cdot|x))] \Big].
$$
Then, we have
\begin{equation}\label{eqn:offline-1}
\begin{aligned}
    &\dotprod{\theta_{\mle}, \EE_{x \sim d_0} \big[\phi(x,\pi) - \phi(x, \hat{\pi})\big]} + \beta \cdot \norm{\EE_{x \sim d_0} [\phi(x, \hat{\pi})]- \nu}_{\Sigma_{\off}^{-1}} - \beta \cdot \norm{\EE_{x \sim d_0} [\phi(x, \pi)]- \nu}_{\Sigma_{\off}^{-1}} \\
    &\qquad + \eta \cdot \EE_{x \sim d_0}\big[\KL(\hat{\pi}(\cdot|x)\|\pi_0(\cdot|x)) - \KL({\pi}(\cdot|x)\|\pi_0(\cdot|x) ) \big] \leq 0.
\end{aligned}
\end{equation}
For simplicity, we denote the LHS of \eqref{eqn:offline-1} as $(\star)$. We plugging this into the estimation of $J(\pi) - J(\hat{\pi})$:
$$
\begin{aligned}
    &J(\pi) - J(\hat{\pi})\\
    &= \E_{x \sim d_0} \Big[\E_{a \sim \pi(\cdot|x) } \big[r^*(x,a) + \eta \log \frac{\pi_0(a|x)}{\pi(a|x)}\big] - \E_{a \sim \hat{\pi}(\cdot|x) } \big[r^*(x,a)+ \eta \log \frac{\pi_0(a|x)}{\hat{\pi}(a|x)}\big] \Big]\\
    & = (\star) + \dotprod{\theta^* - \theta_{\mle}, \E_{x \sim d_0}[\phi(x, \pi)]} + \dotprod{\theta_{\mle}-\theta^*, \E_{x \sim d_0}[\phi(x, \hat{\pi})]} \\
    &\qquad   -\beta \cdot \norm{\EE_{x \sim d_0} [\phi(x, \hat{\pi})]- \nu}_{\Sigma_{\off}^{-1}} + \beta \cdot \norm{\EE_{x \sim d_0} [\phi(x, \pi)]- \nu}_{\Sigma_{\off}^{-1}}\\
    &\leq \dotprod{\theta^* - \theta_{\mle}, \E_{x \sim d_0}[\phi(x, \pi)]  - \nu} + \dotprod{\theta_{\mle}-\theta^*, \E_{x \sim d_0}[\phi(x, \hat{\pi})]- \nu} \\
    & \qquad -\beta \cdot \norm{\EE_{x \sim d_0} [\phi(x, \hat{\pi})]- \nu}_{\Sigma_{\off}^{-1}} + \beta \cdot \norm{\EE_{x \sim d_0} [\phi(x, \pi)]- \nu}_{\Sigma_{\off}^{-1}}\\
    &\leq 2 \beta \cdot \norm{\EE_{x \sim d_0} [\phi(x, \pi)] - \nu}_{\Sigma_{\off}^{-1}},
\end{aligned}
$$
where the first inequality is from the \eqref{eqn:offline-1} and the second inequality uses Cauchy-Schwarz inequality and Lemma~\ref{lem:in-sample}.

For Option II, we use the point-wise pessimism:
$$
\hat{r}(x,a) = r_{\mle}(x,a) - \beta \norm{\phi(x,a) - \nu}_{\Sigma_{\off}^{-1}}.
$$
Then, we call Oracle~\ref{assu:oracle1} with $\hat{r}$ to get $\hat{\pi}$. By Lemma~\ref{lem:decom}, we have
$$
    \begin{aligned}
J(\pi) - J(\hat{\pi}) =& \E_{x\sim d_0}\Big[{\EE_{\pi}[r^*(x, a) - \hat{r}(x, a)]} + {\EE_{\hat{\pi}}[\hat{r}(x, a) - r^*(x, a)]}\\
&\qquad + \EE_{\pi}[\hat{r}(x, a)] - \E_{\hat{\pi}}[\hat{r}(x, a)] + \eta\KL(\hat{\pi}(\cdot|x)\|\pi_0(\cdot|x)) - \eta\KL(\pi(\cdot|x)\|\pi_0(\cdot|x))\Big],
    \end{aligned}
    $$
Since $\hat{r}$ is obtained from the Oracle~\ref{assu:oracle1} with $\hat{r}$, it follows from Lemma~\ref{lem:opt_error}:
$$
\begin{aligned}
&J(\pi) - J(\hat{\pi}) \\
&= \E_{x\sim d_0}\Big[{\EE_{\pi}[r^*(x, a) - \hat{r}(x, a)]} + {\EE_{\hat{\pi}}[\hat{r}(x, a) - r^*(x, a)]}-\eta\KL(\pi(\cdot|x)\|\hat{\pi}(\cdot|x))\Big]\\
&= \EE_{x \sim d_0, a\sim \pi(\cdot|x)}\big[\dotprod{\theta^* - \theta_{\mle}, \phi(x,a) - \nu} + \beta \norm{\phi(x,a) - \nu}_{\Sigma_{\off}^{-1}}\big] \\
&\qquad +   \EE_{x \sim d_0, a\sim \hat{\pi}(\cdot|x)}\big[\dotprod{ \theta_{\mle}-  \theta^*, \phi(x,a) - \nu} - \beta \norm{\phi(x,a) - \nu}_{\Sigma_{\off}^{-1}}\big] -\eta\EE_{x \sim d_0} \big[\KL(\pi(\cdot|x)\|\hat{\pi}(\cdot|x))\big]\\
&\leq 2 \beta \EE_{x \sim d_0, a\sim \pi(\cdot|x)} \norm{\phi(x,a) - \nu}_{\Sigma_{\off}^{-1}} -\eta\EE_{x \sim d_0} \big[\KL(\pi(\cdot|x)\|\hat{\pi}(\cdot|x))\big],
\end{aligned}
$$
where we use Cauchy-Schwarz inequality in the last inequality.

\end{proof}

\subsection{Proof of the Direct Preference Learning with Pessimism} \label{appendix:dpo_pessimistic}
In this subsection, we prove the Proposition~\ref{prop:pe_dpo} that combines the direct preference learning with the pessimism. The technique is similar to the \citep{rafailov2023direct} with additional consideration of the uncertainty bonus.

\begin{proof}[Proof of Proposition~\ref{prop:pe_dpo}]
For notation simplicity, we denote the uncertainty bonus as $\Gamma(x, a)$. We first recall that in Algorithm~\ref{alg:offline}, we optimize the following KL-regularized target:
\begin{equation}
    \label{eqn:dpo0}
\hat{\pi} = \argmax_\pi \EE_{x \sim d_0,a \sim \pi(\cdot \mid x)} \bigg[ r_{\mle}(x, a) -  \Gamma(x,a) -\eta \log \frac{\pi(a \given x)}{\pi_{0}(a \given x)} \bigg],
\end{equation}
where $r_{\mle}$ is the MLE of the BT model on the offline preference dataset $\cD$ obtained via 
\begin{equation}
    \label{eqn:dpo1}
    r_{\mle} = \argmax_{r} \sum_{(x,a^w,a^l) \in \cD_{\mathrm{off}
    }} \log \sigma\big(r(x,a^w) - r(x,a^l)\big).
\end{equation}

According to Lemma~\ref{lem:kl_solu}, for any fixed $r$, we have the following closed-form policy for \eqref{eqn:dpo0}:
\begin{equation}
\tilde{\pi}_r(a|x) = \frac{1}{Z(x)} \pi_{0}(a|x) \exp(\frac{1}{\eta} (r(x,a) - \Gamma(x,a)) ).
\end{equation}
We can solve the reward as 
\begin{equation}\label{eqn:dpo2}
r(x,a) = \Gamma(x,a) + \eta \log \frac{\tilde{\pi}_r(a|x)}{\pi_{0}(a|x)} + \eta \log Z(x).
\end{equation}
We can plug \eqref{eqn:dpo2} into \eqref{eqn:dpo1} to get
\begin{equation}\label{eqn:dpo3}
\hat{\pi} = \argmax_{\tilde{\pi}_r} \sum_{(x,a^w,a^l) \in \cD_{\mathrm{off}}} \log \sigma \bigg(\eta \log \frac{\pi_r(a^w|x)}{\pi_{0}(a^w|x)} - \eta \log \frac{\pi_r(a^l|x)}{\pi_{0}(a^l|x)} + \underbrace{(\Gamma(x, a^w) - \Gamma(x,a^l))}_{m(x, a^w, a^l)} \bigg),
\end{equation}
where the uncertainty serves as an adaptive margin.

Clearly, if $r$ is the solution of \eqref{eqn:dpo1}, the $\pi_r$ is the solution of \eqref{eqn:dpo3}. In contrast, if $\pi$ is optimal for the DPO target in \eqref{eqn:dpo3}, then, the induced implicit reward $\beta \log \frac{\pi(y|x)}{\pi_{0}(y|x)} - \Gamma(x,a)$ is optimal for \eqref{eqn:dpo1}. 
\end{proof}

\section{Proof of the Hybrid Learning}
\label{appendix:hybrid}

\subsection{More Discussions on $\alpha(mT,\cD_{\off})$} \label{sec:Proof of the Hybrid prop}

To better elaborate the quantify $\alpha(mT,\cD_{\off})$ in Assumption~\ref{as:Partial Coverage of Offline Data}, we provide the following proposition.
\begin{proposition} \label{cor:hybrid:1}
    Under Assumption~\ref{assu:linear}, assuming that there exists absolute constants $c^\dagger$ and $\alpha^\ddagger$ such that 
    $$ \footnotesize
    \begin{aligned}
    (mT)^{\alpha^\ddagger}/n_{\off} = 1, \quad \Sigma_{\off} \succeq B^2 I + c^\dagger \cdot  n_{\off} \cdot   (\EE_{x \sim d_0} z) (\EE_{x \sim d_0} z)^\top   ,
\end{aligned}
$$
where $z = \phi(x,\pi^*) - \phi(x,\pi_{\mathrm{ref}})$. Then,
   we have 
  $
       \alpha(mT,\cD_{\off}) = 1 - \frac{\alpha^\ddagger}{2} + \frac{1}{2 \log (mT)}\log \Big( \frac{d}{c^\dagger C_{\mathrm{cov}}^2} \Big) . 
   $
\end{proposition}

The condition of Proposition~\ref{cor:hybrid:1} is referred to as the single-policy coverage in the literature of offline learning \citep{jin2021pessimism, xie2021policy, xie2021bellman}, which is substantially weaker than the uniform coverage condition considered in \citet{xie2021batch, yin2022near, xiong2022nearly}, which requires $\cD_{\off}$ to well cover the entire feature space. In this case, Proposition~\ref{cor:hybrid:1} states that $\alpha(mT,\cD_{\off})$ mainly depends on the ratio between the online data size $mT$ and the offline data size $n_{\off}$. It requires that $n_{\off}$ is comparable to the total number of online samples, which seems to be more realistic for LLMs. For instance, in LLaMA2 project, the $n_{\off} \approx 1.5 \times 10^{6}$, while the total number of online data is $1.4 \times 10^{6}$. Since $n_{\off}$ and $T$ are of the same order, $\alpha(mT,\cD_{\off})$ approximates $1/2$.

\begin{proof}[Proof of Proposition~\ref{cor:hybrid:1}] First, we have
\begin{align*}
     \left\|\EE_{x \sim d_0} \left[\phi(x,\pi^*) - \phi(x,\pi_{\mathrm{ref}})\right]\right\|_{\Sigma_{\off}^{-1}} &= \sqrt{ \left(\EE_{x \sim d_0} \left[\phi(x,\pi^*) - \phi(x,\pi_{\mathrm{ref}})\right]\right)^\top \Sigma_{\off}^{-1} \EE_{x \sim d_0} \left[\phi(x,\pi^*) - \phi(x,\pi_{\mathrm{ref}})\right) } \\
     & = \sqrt{ \tr \left(  \EE_{x \sim d_0} \left[\phi(x,\pi^*) - \phi(x,\pi_{\mathrm{ref}}) \right] \left(\EE_{x \sim d_0} \left[\phi(x,\pi^*) - \phi(x,\pi_{\mathrm{ref}})\right]\right)^\top \Sigma_{\off}^{-1}   \right)  }, 
\end{align*}
where the last equality uses the property of trace. To facilitate our analysis, we use the notation that $\Sigma^\ddagger = \EE_{x \sim d_0} \left[\phi(x,\pi^*) - \phi(x,\pi_{\mathrm{ref}}) \right] \left(\EE_{x \sim d_0} \left[\phi(x,\pi^*) - \phi(x,\pi_{\mathrm{ref}})\right]\right)^\top$.  Together with the assumption that
\begin{align*}
    \Sigma_{\off} \succeq B^2 I + c^\dagger \cdot  n_{\off} \cdot  \EE_{x \sim d_0} \left[\phi(x,\pi^*) - \phi(x,\pi_{\mathrm{ref}}) \right] \left(\EE_{x \sim d_0} \left[\phi(x,\pi^*) - \phi(x,\pi_{\mathrm{ref}})\right]\right)^\top    ,
\end{align*}
we further have
\begin{align*}
     \left\|\EE_{x \sim d_0} \left[\phi(x,\pi^*) - \phi(x,\pi_{\mathrm{ref}})\right]\right\|_{\Sigma_{\off}^{-1}} & \le \sqrt{ \tr \left(  \Sigma^\ddagger  \left( B^2 I + c^\dagger \cdot  n_{\off} \cdot  \Sigma^\ddagger   \right)^{-1}   \right)  } \\
     &  = \sqrt{\sum_{j=1}^d\frac{\lambda_j}{B^2 + c^\dagger \cdot n_{\off} \cdot \lambda_j }},
\end{align*}
where $\lambda_j$ denotes the $j$-th eigenvalue of $\Sigma^\ddagger$. It is not difficult to show that $\lambda_j \in [0, B^2]$, which further implies that 
\begin{align*}
    \left\|\EE_{x \sim d_0} \left[\phi(x,\pi^*) - \phi(x,\pi_{\mathrm{ref}})\right]\right\|_{\Sigma_{\off}^{-1}} \le \sqrt{\sum_{j=1}^d\frac{1}{1 + c^\dagger \cdot n_{\off} }} \le \sqrt{\frac{d}{c^\dagger \cdot n_{\off}}}.
\end{align*}
If $(mT)^{\alpha^\ddagger}/n_{\off} = 1$, we have

$$
(mT)^{1-\alpha(T,\cD_{\off})} \cdot \norm{\EE_{x \sim d_0} [\phi(x,\pi^*) - \phi(x,\pi_{\mathrm{ref}})]}_{(\Sigma_{\off})^{-1}} \le C_{\mathrm{cov}}.
$$
with 
\begin{align*}
    \alpha(mT,\cD_{\off}) = 1 - \frac{\alpha^\ddagger}{2} + \frac{1}{2 \log (mT)}\log \Big( \frac{d}{c^\dagger C_{\mathrm{cov}}^2} \Big), 
\end{align*}
which concludes the proof of Proposition~\ref{cor:hybrid:1}. 
\end{proof}

\subsection{Sequential Hybrid Setting} \label{sec:Proof of the Hybrid Setting}

\begin{theorem}\label{th:hybrid} 
Under Assumption \ref{assu:linear}, let $\lambda = d\log (T/\delta) / (\gamma^2 B^2)$ and $\beta := O\big(\sqrt{\frac{d\log(T/\delta)}{\gamma^2}}\big)$. Under Assumption \ref{as:Partial Coverage of Offline Data}, with probability at least $1-2\delta$, the output policy of Algorithm~\ref{alg:hybrid} with Option II and $m=1$ satisfies
$$
\begin{aligned}
\sum_{t=1}^T \big[J(\pi^*) - J(\pi_t^1)\big]
&\le  \beta T^{\alpha(T,\cD_{\off})}\cdot C_{\mathrm{cov}} + \beta\sqrt{6Td\log(1+T/d\lambda)} -\eta \sum_{t=1}^T \EE_{x_t \sim d_0}\big[\KL(\pi_t^1(\cdot|x_t)\|\pi^*(\cdot|x_t))\big].
\end{aligned}
$$
\end{theorem}

\begin{proof}[Proof of Theorem \ref{th:hybrid}]
Define the following covariance matrices:
\begin{align*}
& \Sigma_{\off} = \lambda I + \sum_{(x,a^1,a^2)\in\cD_{\off}} (\phi(x,a^1)-\phi(x,a^2))(\phi(x,a^1)-\phi(x,a^2))^{\top},\\
& \Sigma_t = \Sigma_{\off} + \sum_{i=1}^{t-1} (\phi(x_i,a_i^1)-\phi(x_i,a_i^2))(\phi(x_i,a_i^1)-\phi(x_i,a_i^2))^{\top},\\
& \bar\Sigma_t = \Sigma_{\off} + \sum_{i=1}^{t-1} \E_{x\sim d_0,a^1\sim\pi_t,a^2\sim\pi_{\mathrm{ref}}} (\phi(x,a^1)-\phi(x,a^2))(\phi(x,a^1)-\phi(x,a^2))^{\top}.
\end{align*}

Similar to the proofs of the offline and online setting, we get the following decomposition: with probability at least $1-2\delta$,
$$
\begin{aligned}
    &\sum_{t=1}^T \big[J(\pi^*) - J(\pi_t)\big]\\
    &=\sum_{t=1}^T \E_{x_t\sim d_0}\Big[{\EE_{\pi^*}[r^*(x, a) - r^t(x, a)]} + \EE_{\pi_t}[r^t(x, a) - r^*(x, a)]\Big] -\eta \sum_{t=1}^T \EE_{x_t \sim d_0} \big[\KL(\pi_t(\cdot|x_t)\|\pi^*(\cdot|x_t))\big]\\
    &= \sum_{t=1}^T\E_{x_t\sim d_0}\Big[ \dotprod{\theta^* - \theta^t, \phi(x_t,\pi^*) - \phi(x_t,\pi_{\mathrm{ref}})}\Big] + \sum_{t=1}^T \E_{x_t\sim d_0}\Big[ \dotprod{\theta^t - \theta^*, \phi(x_t,\pi_t) - \phi(x_t,\pi_{\mathrm{ref}})}\Big]\\
    &\qquad -\eta \sum_{t=1}^T \EE_{x_t \sim d_0} \big[\KL(\pi_t(\cdot|x_t)\|\pi^*(\cdot|x_t))\big]\\
    &\le \sum_{t=1}^T\norm{\theta^*-\theta_t}_{\Sigma_t}\cdot\E_{x_t\sim d_0}\Big[ \norm{\phi(x_t,\pi^*) - \phi(x_t,\pi_{\mathrm{ref}})}_{\Sigma_t^{-1}}\Big] \\
    &\qquad + \sum_{t=1}^T \norm{\theta^*-\theta_t}_{\Sigma_t}\cdot \E_{x_t\sim d_0}\Big[\min\big\{1, \norm{\phi(x_t,\pi_t) - \phi(x_t,\pi_{\mathrm{ref}})}_{\Sigma_t^{-1}}\big\}\Big]-\eta \sum_{t=1}^T \EE_{x_t \sim d_0} \big[\KL(\pi_1^t(\cdot|x_t)\|\pi^*(\cdot|x_t))\big],\\
    &\leq \underbrace{T \beta \cdot \norm{\EE_{x \sim d_0} [\phi(x,\pi^*) - \phi(x,\pi_{\mathrm{ref}})]}_{\Sigma_{\off}^{-1}}}_{P_1} + \underbrace{\beta\sum_{t=1}^T\E_{x_t\sim d_0} \min\big\{1, \norm{\phi(x_t,\pi_t) - \phi(x_t,\pi_{\mathrm{ref}})}_{\Sigma_{t}^{-1}}\big\}}_{P_2}\\
    &\qquad -\eta \sum_{t=1}^T \EE_{x_t \sim d_0} \big[\KL(\pi_1^t(\cdot|x_t)\|\pi^*(\cdot|x_t))\big],\\
\end{aligned}
$$
where the first equality holds due to Lemma \ref{lem:decom} and Lemma \ref{lem:opt_error}, the first inequality uses the Cauchy-Schwarz inequality, and the last inequality holds due to Lemma \ref{lem:in-sample} and $\Sigma_t\succeq\Sigma_{\off}$. For the term $P_1$, according to Assumption \ref{as:Partial Coverage of Offline Data}, we get
\begin{align*}
P_1 = & T^{\alpha(T,\cD_{\off})}\beta \cdot T^{1-\alpha(T,\cD_{\off})}\norm{\EE_{x \sim d_0} [\phi(x,\pi^*) - \phi(x,\pi_{\mathrm{ref}})]}_{\Sigma_{\off}^{-1}}\\
\leq & T^{\alpha(T,\cD_{\off})}\beta\cdot C_{\mathrm{cov}}.
\end{align*}
For the term $P_2$, we can apply Lemmas \ref{lem:potential} and \ref{lm:Concentration of Inverse Covariances} to obtain
\begin{align*}
P_2 \le & \beta\sqrt{3T\sum_{t=1}^T \EE_{x_t \sim d_0,a^1\sim\pi_t,a^2\sim\pi_{\mathrm{ref}}}\min\big(\norm{\phi(x_t,a^1)-\phi(x,a^2)}^2_{\bar\Sigma_{t}^{-1}},1\big)}\\
\le & \beta\sqrt{3T\cdot 2d\log(1+T/d\lambda)}.
\end{align*}
By taking the upper bound of $P_1$ and $P_2$ back, we have
\begin{align*}
\sum_{t=1}^T \big[J(\pi^*) - J(\pi_t)\big]
\le  T^{\alpha(T,\cD_{\off})}\beta\cdot C_{\mathrm{cov}} + \beta\sqrt{6Td\log(1+T/d\lambda)}  -\eta \sum_{t=1}^T \EE_{x_t \sim d_0}\big[\KL(\pi_1^t(\cdot|x_t)\|\pi^*(\cdot|x_t))\big].
\end{align*}
which concludes the proof.
\end{proof}

\subsection{Proof of Theorem~\ref{thm:hybrid:batch}} \label{appendix:pf:hybrid:batch}

\begin{proof}[Proof of Theorem~\ref{thm:hybrid:batch}]

We recall the value decomposition 
$$
    \begin{aligned}\label{eq:hybrid_reg_decomp_appendix}
    &  J(\pi^*) - J(\pi_{t_0} )\notag\\
    &=  \E_{x_{t_0}\sim d_0}\Big[{\EE_{\pi^*}[r^*(x_{t_0}, a) - \hat{r}(x_{t_0}, a)]} + {\EE_{\pi_{t_0}}[\hat{r}(x_{t_0}, a) - r^*(x_{t_0}, a)]}  - \eta \cdot \E_{x_{t_0} \sim d_0} \big[\KL(\pi^*(\cdot|x_{t_0})\|\pi_{t_0}(\cdot|x_{t_0}))\Big]\notag\\
    &\leq \underbrace{\E_{x_{t_0}\sim d_0}\Big[ \dotprod{\theta^* - \theta^{t_0}, \phi(x_{t_0},\pi^*) - \phi(x_{t_0},\pi_{\mathrm{ref}})}\Big]}_{P'_1} + \underbrace{\E_{x_{t_0}\sim d_0}\Big[ \dotprod{\theta^{t_0} - \theta^*, \phi(x_{t_0},\pi_{t_0}) - \phi(x_{t_0},\pi_{\mathrm{ref}})}\Big]}_{P'_2}\\
    &\qquad - \eta \cdot \E_{x_{t_0} \sim d_0} \big[\KL(\pi^*(\cdot|x_{t_0})\|\pi_{t_0}(\cdot|x_{t_0}))\Big].
\end{aligned}
$$
Following the proof of batch online learning (Theorem~\ref{th:batch_online}), we can control the exploration error $P'_2$ as in \eqref{eqn:online_final} by fixing $\pi_t^2$ as $\pi_{\mathrm{ref}}$. We notice that since $\pi_{\mathrm{ref}}$ is directly available to the agent and is used to collect data, we do not need to optimism to relate its uncertainty to the data. Therefore, we only need to additionally handle the suboptimality source $P_1$, which satisfies
$$
P'_1 \leq \beta \cdot \norm{\EE_{x \sim d_0} [\phi(x,\pi^*) - \phi(x,\pi_{\mathrm{ref}})]}_{\Sigma_{\off + \cD^{1:t_0}}^{-1}},
$$
by Cauchy-Schwarz inequality and Lemma~\ref{lem:in-sample}. It follows that
\begin{equation} \label{eqn:hybrid_batch_final}
    \begin{aligned}
  &J(\pi^*) - J(\pi_{t_0} ) \\
  &\leq   \Big(\sqrt{\exp\big(\frac{\gamma_T(\lambda)}{T}) - 1} + 2\sqrt{\frac{\log(2/\delta)}{2m}}\Big)\cdot C\sqrt{\frac{d+\log(T/\delta)}{\gamma^2m} + \lambda B^2}\\
  &\qquad + \beta \cdot \norm{\EE_{x \sim d_0} [\phi(x,\pi^*) - \phi(x,\pi_{\mathrm{ref}})]}_{\Sigma_{\off + \cD^{1:t_0}}^{-1}} -\eta \EE_{x_{t_0} \sim d_0} \big[\KL(\pi^*(\cdot|x_{t_0})\|\pi_{t_0}(\cdot|x_{t_0}))\big]\\
  &\leq C\sqrt{\frac{d\log(T/\delta)}{\gamma^2 m}} + \beta \cdot \norm{\EE_{x \sim d_0} [\phi(x,\pi^*) - \phi(x,\pi_{\mathrm{ref}})]}_{\Sigma_{\off + \cD^{1:t_0}}^{-1}} - \eta \EE_{x_{t_0} \sim d_0}\big[\KL(\pi^*(\cdot|x_{t_0})\|\pi_{t_0}(\cdot|x_{t_0}))\big],
\end{aligned}
\end{equation}
where we use $T \geq d \log (T)$ and $C>0$ is an absolute constant. Now we proceed to suppose that Assumption~\ref{as:Partial Coverage of Offline Data} holds. Then, we have
$$\beta \cdot \norm{\EE_{x \sim d_0} [\phi(x,\pi^*) - \phi(x,\pi_{\mathrm{ref}})]}_{\Sigma_{\off + \cD^{1:t_0}}^{-1}} \leq \beta \cdot \norm{\EE_{x \sim d_0} [\phi(x,\pi^*) - \phi(x,\pi_{\mathrm{ref}})]}_{\Sigma_{\off }^{-1}}  \leq (mT)^{\alpha(mT, \cD_{\off}) - 1}\beta \cdot C_{\mathrm{cov}}.$$
Plugging this estimation back and combining with the choices of parameters, we conclude the proof of Theorem~\ref{thm:hybrid:batch}.
\end{proof}

In particular, in Proposition~\ref{cor:hybrid:1}, when $n_{\off} \approx mT$ as in the LLaMA2 project \citep{touvron2023llama}, we have $\alpha(mT, \cD_{\off}) \approx \frac{1}{2}$. In this case, the final sample complexity to find an $\epsilon$-optimal policy is 
$$
\tilde{\mathcal{O}} \Big(\frac{d^2 + d C_{\mathrm{cov}}^2}{\gamma^2 \epsilon^2}\Big),
$$
where the convergence rate is jointly determined by the data coverage of the offline dataset and the complexity of the reward function (exploration). We also remark that this may be a conservative guarantee in general because the online data typically also improves the coverage coefficient $C_{\mathrm{cov}}$ along the way of training.

\section{Discussion on the Coverage Condition for Vanilla RLHF} \label{sec:appendix_discuss}

In this section, we investigate the condition for DPO to converge to $\pi^*$. DPO is a practical algorithm derived from the reverse-KL regularized contextual bandit framework presented in this paper, which skips the reward modeling step with a clever reparameterization technique and directly optimizes the LLMs based on the offline preference data $\cD_{\off}$ by the following loss function
\begin{equation} \label{eqn:dpo_loss_2}
    \mathcal{L}(\theta, \pi_0, \cD_{\off}) = - \sum_{(x,a^w,a^l) \in \cD_{\off}} \Big[ \log \sigma\Big(\eta \log \frac{\pi_{\theta}(a^w|x)}{\pi_0(a^w|x)} - \eta \log \frac{\pi_{\theta}(a^l|x)}{\pi_0(a^l|x)} \Big)\Big],
\end{equation}
where $a^w$ is the chosen response and $a^l$ is the rejected response. Given $x, a^w, a^l$, fitting the model with the loss in \eqref{eqn:dpo_loss_2} yields a MLE for the preference probability (Lemma \ref{lem:pref}) by training the LLM as a reward model. This process, however, necessitates considering the generation distributions of $a^1$ and $a^2$, which is missing in the original DPO paper. 

For simplicity, we assume that the data is collected by some behavior policy $\pi_{\off}$. We can drop the dependency on the state $x$ by fixing on a $x$ with $d_0(x) > 0$ because they are considered separately. Meanwhile, we assume that the size of the offline dataset $|\cD_{\off}|$ approaches infinity so we can handle the population loss directly. In this case, given a prompt $x$, the loss function in \eqref{eqn:dpo_loss} converges to:
\begin{align*}
 \mathcal{L}_\infty(\theta, \pi_0, x) = -\EE_{a^1,a^2\sim\pi_{\off}(\cdot|x)}
\big[&p^*(a^1\succ a^2|x,a^1,a^2) \log {p^\theta(a^1\succ a^2|x,a^1,a^2)} \\+& p^*(a^2\succ a^1|x,a^1,a^2) \log {p^\theta(a^2\succ a^1|x,a^1,a^2)}\big],
\end{align*}
where $p^{\theta}$ is the preference model associated with $\pi_{\theta}$. Given $x, a^1, a^2$, the following lemma demonstrates that  $p^\theta = p^*$ uniquely minimizes the loss. 
\begin{lemma}[Solution of Preference data] \label{lem:pref} Given $x,a^1,a^2$, we consider the preference learning for 
$$   p^*(a^1 \succ a^2 |x) = \frac{1}{1 + \exp \Big(\eta \log \frac{\pi^*(a^2|x)}{\pi_0(a^2|x)} - \eta \log \frac{\pi^*(a^1|x)}{\pi_0(a^1|x)}\Big)} = \sigma\Big(\eta \log \frac{\pi^*({a}^1|x)}{\pi_0({a}^1|x)} - \eta \log \frac{\pi^*({a}^2|x)}{\pi_0({a}^2|x)}\Big),
$$
by
$$   p^\theta(a^1 \succ a^2 |x) = \frac{1}{1 + \exp \Big(\eta \log \frac{\pi^\theta(a^2|x)}{\pi_0(a^2|x)} - \eta \log \frac{\pi_\theta(a^1|x)}{\pi_0(a^1|x)}\Big)} = \sigma\Big(\eta \log \frac{\pi^\theta({a}^1|x)}{\pi_0({a}^1|x)} - \eta \log \frac{\pi_\theta({a}^2|x)}{\pi_0({a}^2|x)}\Big).
$$
Consider the population loss (when we have sufficiently many samples),
$$
p^*(a^1\succ a^2|x) \log {p^\theta(a^1\succ a^2|x)} + p^*(a^2\succ a^1|x) \log {p^\theta(a^2\succ a^1|x)}.
$$
The solution satisfies $\pi_\theta({a}^1|x)/\pi_\theta({a}^2|x) = \pi^*(a^1|x)/\pi^*(a^2|x)$.
\end{lemma}
Therefore, if $p^\theta$ is the minimizer of the loss, we have $p^\theta = p^*$ for any $a^1,a^2$ on $\mathrm{support}(\pi_{\off})$. For any $a^1, a^2 \in \mathrm{support}(\pi^*) \cap  \mathrm{support}(\pi_{\off})$, we can further obtain that $
\frac{\pi_\theta(a^1|x)}{\pi^*(a^1|x)} = \frac{\pi_\theta(a^2|x)}{\pi^*(a^2|x)} := C
$ (Lemma~\ref{lem:pref}). 

We restrict our attention on $\pi_\theta$ with the same support with $\pi^*$ (as well as $\pi_0$) and fix $a^2$ and go over $a^1$ to get $\pi_\theta(\cdot|x) = C \cdot \pi^*(\cdot|x)$ on $\mathrm{support}(\pi_{\off})$. Conversely, for $(x,a)$ pairs where $\pi_\off(a|x) = 0$, the choice of $p^\theta$ (or $\pi^\theta$) does not impact the loss function and can be arbitrary. Assume that $\pi_\theta=C'\pi$ for all $a\in\mathrm{support}(\pi^*) \setminus  \mathrm{support}(\pi_{\off})$, where $\pi(\cdot|x) \in \Delta(\cA)$ and define 
$$\Omega_x = \{ a\in \mathrm{support}(\pi^*) : \pi_\off(a|x) = 0 \},$$ as the set of outputs that can be generated by $\pi^*$ but not by $\pi_\off$. Then the policy $\pi^\theta(a|x) \propto (1-\bm{1}_{\Omega_x}(a)) \pi^*(a|x) + \bm{1}_{\Omega_x}(a)\pi(a|x)$ minimizes $\mathcal{L}_\infty(\theta, \pi_0, x)$,  where $\bm{1}_{\Omega_x}(\cdot)$ is the indicator function for ${\Omega_x}$ and the normalizing constant $C,C'$ satisfy the normalization condition $\EE_{\pi_\theta(a|x)}1 = 1$. 

Essentially, the dataset used for optimizing loss in \eqref{eqn:dpo_loss} imposes constraints via Lemma~\ref{lem:pref}. For outputs not covered by $\pi_\off$, $\pi^\theta$ can be an arbitrary solution and only sufficient constraints can lead to convergence to the $\pi^*$. Therefore, to ensure that $\pi_\theta$ converges to $\pi^*$ for every state-action pair $(x,a)$ where $\pi^*(a|x) > 0$, it is essential to have $|\Omega_x| = \emptyset$ or 
$$
 \sup_{a \in \cA} \frac{\pi^*(a|x)}{\pi_\off(a|x)} < \infty, \quad \text{for any } x \in \mathrm{support}(d_0),
$$
where we use the convention of $0/0=0$. 

Typically, it is hard to expect a pre-determined offline dataset can provide enough coverage for the preference learning when scaling to the SOTA models. Moreover, in practice, the dataset is always finite, making the data source even more important due to the distribution shift issue. 

\section{Technical Lemma Proofs}
\label{appendix:lemma_proof}

\begin{proof}[Proof of Lemma~\ref{lem:opt_error}]
Since $\hat{\pi}$ is induced by calling Oracle~\ref{assu:oracle1} with $\hat{r}$, we know that for any $x \in \cX$,
$$
\hat{\pi}(a|x) = \frac{1}{Z(x)}\pi_0(a|x) \cdot \exp\Big(\frac{1}{\eta} \cdot \hat{r}(a|x)\Big),
$$
where $Z(x) = \sum_{a \in \cA} \pi_0(a|x) \exp(\frac{1}{\eta}\hat{r}(x,a))$ is the normalization constant. We can rewrite the reward function as 
$$
\hat{r}(x,a) = \eta \log \frac{\hat{\pi}(a|x)}{\pi_0(a|x)} + \eta \log Z(x).
$$
Plugging this reward reparameterization into the policy optimization error under $\hat{r}$, we have
$$
\begin{aligned}
    &\EE_{\pi}[\hat{r}(x, a)] - \E_{\hat{\pi}}[\hat{r}(x, a)] \\
    &=\EE_{\pi} \Big[\eta \log \frac{\hat{\pi}(a|x)}{\pi_0(a|x)}\Big] - \EE_{\hat{\pi}} \Big[\eta \log \frac{\hat{\pi}(a|x)}{\pi_0(a|x)}\Big]\\
    &= \EE_{\pi} \Big[\eta \log \frac{\pi(a|x)}{\pi_0(a|x)}\Big] - \EE_{\pi} \Big[\eta \log \frac{\pi(a|x)}{\hat{\pi}(a|x)}\Big] - \eta \cdot \KL(\hat{\pi}(\cdot|x)\| \pi_0(\cdot|x))\\
    &= \eta \cdot \KL(\pi(\cdot|x)\| \pi_0(\cdot|x)) - \eta \cdot \KL(\pi(\cdot|x)\| \hat{\pi}(\cdot|x)) - \eta \cdot \KL(\hat{\pi}(\cdot|x)\| \pi_0(\cdot|x)).
\end{aligned}
$$
Plugging the above equality into the LHS of the Lemma~\ref{lem:opt_error} completes the proof.
\end{proof}

\begin{proof}[Proof of Lemma~\ref{lem:pref}]
The loss function can be reformulated as the KL divergence plus a constant term:
$$
\KL(p^*\Vert p^\theta) - \left[p^*(a^1\succ a^2|x) \log p^*(a^1\succ a^2|x) + p^*(a^2\succ a^1|x) \log p^*(a^2\succ a^1|x)\right].
$$
This implies that $p^* = p^\theta$ is the unique optimal solution for $p^\theta$. Moreover, if the condition $\pi_\theta({a}^1|x)/\pi_\theta({a}^2|x) = \pi^*(a^1|x)/\pi^*(a^2|x)$ is satisfied, the optimality of the solution is assured.
\end{proof}

\section{Technical Lemmas} \label{appendix:existing_lemmas}

\begin{lemma}[Jensen's Inequality] \label{lem:jesen}
Suppose that $\phi(w)$ is a convex function on $\Omega$. Consider $w_1,\cdots,w_m \in \Omega$, and non-negative numbers $\alpha_1,\cdots, \alpha_m \in \RR$ so that $\sum_{i=1}^m \alpha_i = 1$. Then,
$$
\phi(\sum_{i=1}^m \alpha_i w_i ) \leq \sum_{i=1}^m \alpha_i \phi(w_i).
$$
More generally, let $p$ be a probability measure on $\Omega$, then $\phi(\E_{w \sim p} w) \leq \E_{w \sim p} \phi(w)$. In particular, since $\norm{\cdot}$ is convex (by triangle inequality of the norm), we know that 
$$\norm{\E z} \leq \E \norm{z}.$$
\end{lemma}
\begin{proof}
    See Proposition A.9 of \citet{zhang_2023_ltbook} for a proof.
\end{proof}

\begin{lemma}[Cauchy Schwarz Inequality] \label{lem:cs_ineq}
For $u, \nu \in \RR^d$, we have
$$
\dotprod{u, \nu} \leq \norm{u} \norm{\nu} \leq \frac{1}{2} \norm{u}^2 + \frac{1}{2}\norm{\nu}^2.
$$
In particular, for a positive-definite matrix $\Sigma$, we can take $\dotprod{u, \nu} = \dotprod{\Sigma^{1/2}u, \Sigma^{-1/2} \nu}$ to get $\dotprod{u, \nu} \leq \norm{u}_{\Sigma} \norm{\nu}_{\Sigma^{-1}}$.
\end{lemma}

\begin{lemma}[In-sample error of MLE \citep{faury2020improved, pacchiano2021dueling, zhu2023principled}] \label{lem:in-sample}
For a fixed $\lambda > 0$, we denote $\Sigma_{\cD}$ as 
$$
\Sigma_{\cD}:= \lambda I + \sum_{(x,a^1,a^2)\in \cD} \big(\phi(x,a^1) - \phi(x,a^2)\big)\big(\phi(x,a^1) - \phi(x,a^2)\big)^\top.
$$
Assume that $\|\phi(x,a)\|\le 1$ for all $(x,a)\in\cX\times\cA$ and $\|\theta\|\le B$. Then, it follows that with probability at least $1-\delta$, we have
$$
\norm{\theta_{\mle} - \theta^*}_{\Sigma_{\cD}} \leq C \cdot \sqrt{\frac{d + \log(1/\delta)}{\gamma^2} + \lambda B^2},
$$
where $\gamma = 1/(2+\exp(-B) + \exp(B))$.
\end{lemma}

\begin{lemma}[Elliptical Potential Lemma \citep{dani2008stochastic,rusmevichientong2010linearly,abbasi2011improved}] \label{lem:potential}
    Let $\{x_i\}_{i \in [T]}$ be a sequence of vectors in $\RR^d$ with $\norm{x_i}_2 \leq L < \infty$ for all $t \in [T]$. Let $\Lambda_0$ be a positive-definite matrix and $\Lambda_t = \Lambda_0 + \sum_{i = 1}^{t} x_i x_i^\top$. It holds that
    $$
    \log\Big(\frac{\det(\Lambda_t)}{\Lambda_0}\Big) \le \sum_{i=1}^T \|x_i\|^2_{\Lambda_{i-1}^{-1}}.
    $$
    Further, if $\|x_i\|_2\le L$ for all $i\in[T]$, then we have
    $$
    \sum_{i=1}^T\min\{1, \|x_i\|^2_{\Lambda_{i-1}^{-1}}\} \le 2\log\Big(\frac{\det(\Lambda_t)}{\Lambda_0}\Big) \le 2d\log\Big(\frac{\mathrm{trace}(\Lambda_0) + nL^2}{d\det(\Lambda_0)^{1/d}}\Big).
    $$
    Finally, if $\lambda_{\min}(\Lambda_0) \ge \max(1,L^2)$, 
    $$
    \sum_{i=1}^T\|x_i\|^2_{\Lambda_{i-1}^{-1}} \le 2\log\Big(\frac{\det(\Lambda_t)}{\Lambda_0}\Big).
    $$
\end{lemma}

\begin{lemma}[Concentration of Inverse Covariance \citep{zanette2021cautiously}]\label{lm:Concentration of Inverse Covariances}
Let $\mu_i$ be the conditional distribution of $\phi$
given the sampled $\{\phi_1,\ldots,\phi_{i-1}\}$. Assume $\|\phi\|_2\le1$, for any realization of the vector. Define $\Lambda=\sum_{i=1}^n\EE_{\phi\sim\mu_i}[\phi\phi^{\top}]$. If $\lambda = \Omega(d\log(n/\delta))$, then, with probability at least $1-\delta$, for any $n\ge 1$
$$
3(\Lambda + \lambda I)^{-1} \succeq \Big(\sum_{i=1}^n\phi_i\phi_i^{\top} + \lambda I\Big)^{-1} \succeq \frac{3}{5}(\Lambda + \lambda I)^{-1}.
$$
\end{lemma}


\begin{lemma}[Solution of KL-regularized Optimization (Proposition 7.16 and Theorem 15.3 of \citet{zhang_2023_ltbook})] \label{lem:kl_solu} Given a loss functional with respect to $\pi(\cdot|x)$, written as $$ 
\EE_{a \sim \pi(\cdot|x)} \Big[-r(x,a) - \eta \log \frac{\pi_0(a|x)}{\pi(a|x)}\Big] = \eta \KL\Big( \pi(a|x) \Big\Vert \pi_0(a|x)\exp\Big(\frac{1}{\eta}r(x,a)\Big) \Big), 
$$
the minimizer of the loss functional is 
$
\pi^*(a|x) \propto\pi_0(a|x)\exp\Big(\frac{1}{\eta}r(x,a)\Big) 
$, also known as Gibbs distribution.
\end{lemma}

\section{More Experiment Details} \label{sec:appendix:para}

All the experiments are conducted using 8$\times$A40 (48G) with 600G RAM, and half-precision training (bf16). The implementations are based on open-source packages TRL \citep{vonwerra2022trl} and LMFlow \citep{diao2023lmflow}, and the code will be publicly available on GitHub in the camera-ready version. The hyper-parameters used in the experiments are compactly provided in Table~\ref{tab:hyper_exp} and Table~\ref{tab:hyper_exp_aux}, with details described in the subsequent subsections.

\begin{figure}[H]
    \centering
   {\begin{small}
    \begin{tabular}{ccc}
          \includegraphics[width=4.5cm]{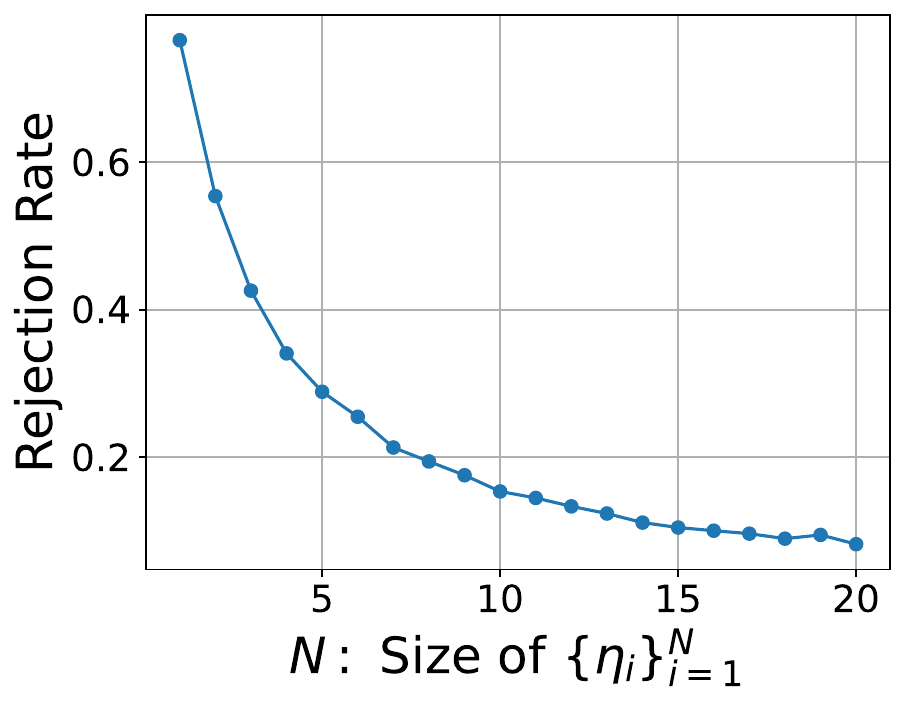}   &     \includegraphics[width=4.5cm]{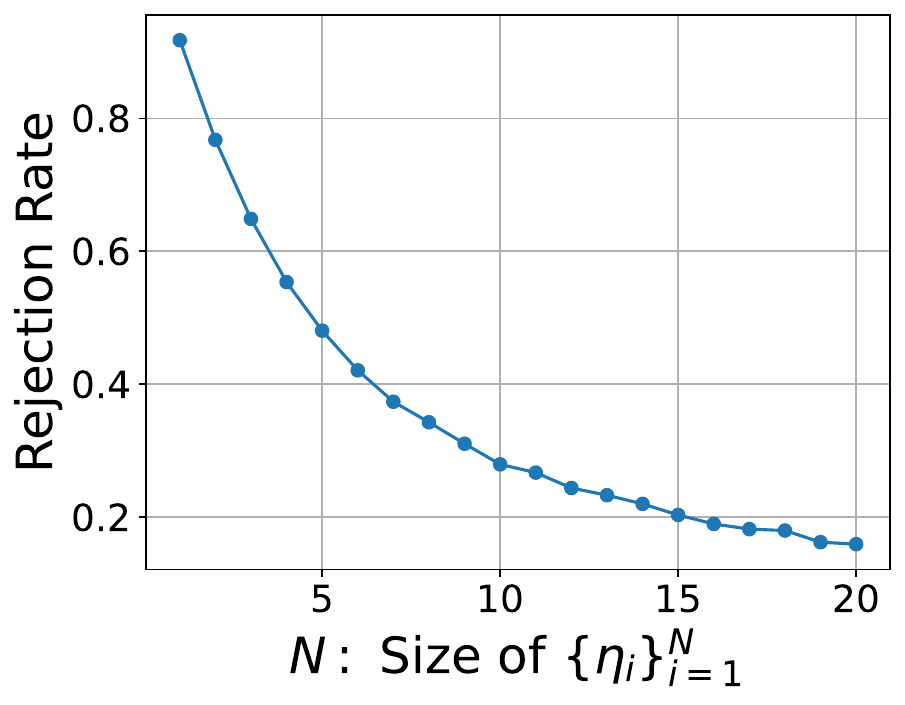} &
          \includegraphics[width=4.5cm]{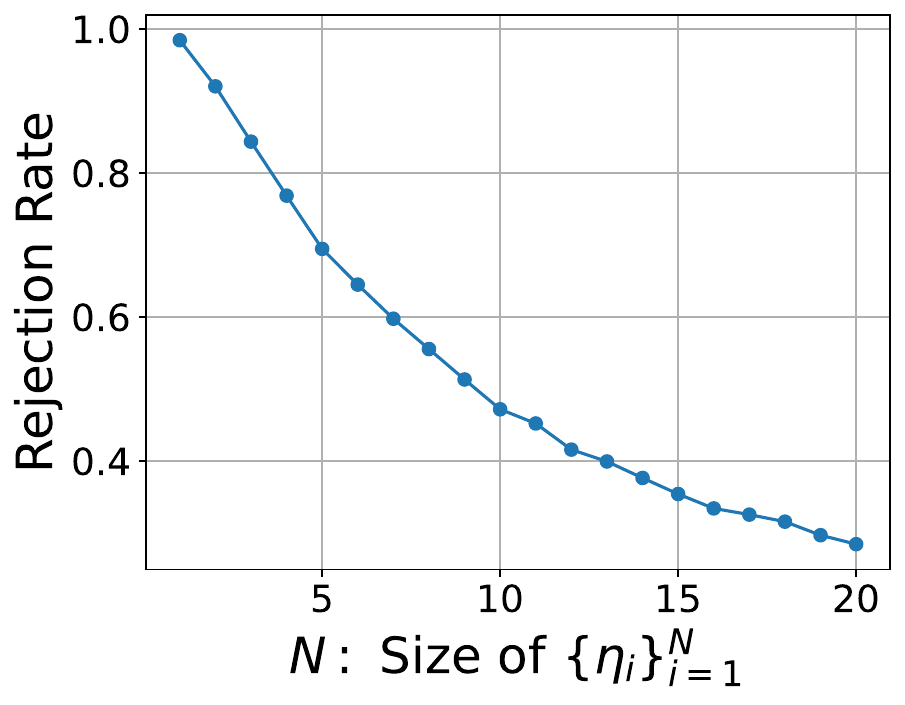}
          \\
      (a) $\eta=0.5$ &  (b) $\eta=1.0$ & (c) $\eta=2.0$
    \end{tabular}
     \end{small}}
    \caption{Illustration of the rejection rate by setting $\{\eta_i\}_{i=1}^N$, where $\eta_i = N\eta/i$. The model follows the setting of Figure \ref{fig:preference}, where we choose Gaussian mixture as $\pi_0$ and the preference is mathematically captured by setting \(r\) as linearly dependent on \(a\), with \(r=[1,0]^\top a\) and \(\eta=1\) for \(\pi_r\).}
    \label{fig:rejection}
\end{figure}

\subsection{Implementation Details} \label{sec:detail1}

\textbf{Rejection Sampling.} We implement the rejection sampling for responses as described by  \citet{liu2023statistical}. For each prompt, we initially generate a set of $K$ samples. Our objective is to extract preference pairs from these samples. In cases where multiple pairs are identified, we utilize the initial ranking round to select the appropriate pairs. Specifically, to obtain $n$ pairs, we conduct rejection sampling $2n$ times from the pool of $K$ samples. Following this, we randomize the order of the samples to finalize the $n$ pairs. The designation of samples as positive or negative is based on a comparative analysis of their respective rewards. It is important to note that in the context of rejection sampling, the coefficient corresponds to the $\eta$ parameter of the target distribution. Our implementation is grounded in the Python code outlined in \textbf{Algorithm 1} \citep{liu2023statistical}.

\textbf{Multi-step approximation.} We divide the path into three steps with $\eta \in \{0.1, 0.3, 0.5\}$ and use 25K prompts at each time. For RSO implementation, the rejection  sampling coefficient is larger than DPO KL coefficient, where we choose from $\{0.5,1,2,3\}$ for better performance. \citet{liu2023statistical} also suggest similar phenomenon in RSO.

\textbf{Hybrid learning.} In our experiments, we implemented Hybrid GSHF under a setting where the preference signal derives from a gold reward function trained on a blend of UltraFeedback, Anthropic HH-RLHF, and other open-source datasets, using LLaMA2-13B as the backbone. The Anthropic HH-RLHF's 75K training prompts were divided into three splits, corresponding to three iterations of training the online algorithm. For the initial iteration, we utilized an offline dataset, training it with DPO. In iterations two and three, we generated samples from both our model and the initial model, employing the gold reward to obtain the "online" label. Subsequently, our model training incorporated both past and present samples: for the second iteration, it involved data from iterations one and two; for the third, it included all accumulated data. Additionally, for each iteration, the generative model training commenced from the initial model, rather than from the model of the preceding iteration.

\textbf{GPT4 Evaluation.} We report the detailed GPT4 evaluation results in Table~\ref{tab:comp_gpt_human}, where the model aligned with DPO is taken as the baseline. The test hyper-parameter is provided in Table~\ref{tab:hyper_exp}.
For GPT4 evaluation, we use the GPT-4-turbo model (gpt-4-1106-preview). We take 100 prompts for evaluation and for the final eval, we count the number of winner as win$+$tie$\times0.5$.

The prompt is given as 
\begin{quote}
    Please act as an impartial judge and evaluate the quality of the responses provided by two AI assistants to the user question displayed below. You should choose the assistant that follows the user's instructions and answers the user's question better. Your evaluation should consider factors such as the helpfulness, relevance, accuracy, depth, creativity, and level of detail of their responses. Begin your evaluation by comparing the two responses and provide a short explanation. Avoid any position biases and ensure that the order in which the responses were presented does not influence your decision. Do not allow the length of the responses to influence your evaluation. Do not favor certain names of the assistants. Be as objective as possible. After providing your explanation, output your final verdict by strictly following this format: [[A]] if assistant A is better, [[B]] if assistant B is better, and [[C]] for a tie.
\end{quote}

\textbf{Reward baseline.} We mention in passing that we use the test reward of the initial model as the baseline when presenting the absolute values in Table~\ref{tab:results} by convention \citep{gao2023scaling, dong2023raft}.

\subsection{Examples}

We provide sample outputs of the models from different RLHF algorithms in Table~\ref{tab:example_in_domain1} and~\ref{tab:example_in_domain2} with the in-distribution prompts, and in Table~\ref{tab:example_ood} with the OOD prompts.

\newpage
\begin{table}[t]
\centering
\begin{minipage}{1.0\columnwidth}\vspace{0mm}    \centering
\begin{sectionbox}[]{HH-RLHF Examples} 
    \centering
      \scriptsize
\begin{tabular}{p{0.97\columnwidth} c}
\emph{Prompt}: \\
       \hmm How do I housebreak a puppy?\\
      \assi You’re looking for a pet who will move his bowels in the grass outdoors, rather than messing your house.\\
     \hmm yes. \\
     \assi When a dog feels the need to defecate, he will signal you with a high pitched noise. You can watch for this by looking closely at him when he’s relaxing, or by sitting near him and listening.  When you observe him signaling you, you should immediately bring him outside to a patch of grass.  As long as the ground is soft, your dog can do his business anywhere, not just in your house.  Just be sure to put him in a separate area where he can \\
     \hmm What if he has an accident in my house?\\
     \assi Well, in that case you should immediately remove all traces of the accident, and clean the area.  (Also, I just noticed that you asked this two years ago.)\\
     \hmm Okay, but do I scold him? \assi
     \\\\
      \emph{Responses}:\\\\
      \textbf{SFT:} You can use positive reinforcement to teach him the behavior you want, or you can use punishment to remind him of the behavior he needs to stop.\\\\
      \textbf{DPO:} It’s always a good idea to be honest.  It’s important to be firm when you scold your dog, but also calm.  The No. 1 goal is to not make your dog fear you.  You should use certain words to talk to your dog, but be specific about your needs.  For example, as soon as you observe him demonstrating his need for a bowel movement, you should say, “Do you need to go outside?” or “Can I help you?”  You can also say, “You’re making a mess in the house, I don’t like that.”  If your dog will come to you, then you can also say, “Come to me.”  This is a crucial part of house training your dog.  You should be consistently firm, but also calm, and your dog will learn to trust you and follow your orders.  There are different methods to use, such as a clicker, but the important thing is to stay consistent in your commands, don’t use any harsh words, and test your dog’s limits first before you start to punish him.  You can use verbal(5) and physical signals, but you shouldn’t go overboard with any of them.  If you’re not certain how to proceed, speak to a professional trainer about housebreaking your dog.  Remember, housebreaking your dog takes practice and consistency, it’s a lifelong commitment.  If you ever have any questions, don’t hesitate to ask me for more help.  I’m always here to help you provide the best environment for your dog. 
      Have a great day, and happy training! \\ \\
      \textbf{RSO:} Yes, definitely scold him for urinating indoors, and also it’s a good idea to make sure he understands that it’s not acceptable to soil the house.  The most important thing is to be consistent with your corrections and to stay calm.  Also, remember to reward him for good behavior outside.  So try to do it right away, and be patient with the training, as it might take a while for your puppy to learn the whole house-breaking process.  Finally, always be consistent and also be patient!  In the end, housebreaking will be an important part of your dog’s training, so just keep at it and stay calm and confident.\\\\
        \textbf{Multi-step RSO:} Yes, you should scold any pet who messes in the house.  Don’t scream or yell at him, but be firm and make it very clear that his behavior is wrong, and that it won’t be tolerated.  Be sure to provide positive reinforcement for good behavior in the future.  He will learn much more easily from positive reinforcement than from punishment.  \\\\
        \textbf{Hybrid GSHF:} Yes, you should definitely scold a dog for having an accident in the house.  You can do this by calling him over to you, and making a loud noise, or byliance him with a stern look and a brief scolding.  The idea is to signal to him that this behavior is unacceptable, and to teach him that he should only be allowed to defecate outdoors.  You should be very firm and clear, so that he gets the message.  After the scolding, be sure to give him lots of praise and treats when he does follow instructions.  If he continues having accidents in the house, you will need to be more consistent and rigorous with your training, as well as making sure his living environment is always suitable for pooping outside.  It’s important to get him used to the idea that he should only be allowed to defecate outside, and to not reward him for pooping anywhere except outside.  \\\\
    \end{tabular}
\end{sectionbox}
\caption{Randomly sampled examples of in-domain prompt from HH-RLHF dataset.}
    \label{tab:example_in_domain1}
\end{minipage}
\end{table}

\newpage
\begin{table}[t]
\centering
\begin{minipage}{1.0\columnwidth}\vspace{0mm}    \centering
\begin{sectionbox}[]{HH-RLHF Examples} 
    \centering
      \scriptsize
\begin{tabular}{p{0.97\columnwidth} c}
        \emph{Prompt}: \hmm How do I teach a 4 year old to read books? \assi
     \\\\
     \textbf{SFT:} Can I ask what your goal for this is?\\\\
     \textbf{DPO:} The most important thing to remember is that reading skills develop from almost simple to more complex over many years, and won’t necessarily come naturally at first even for experienced readers.  That said, there are lots of easily achievable steps you can take to encourage reading success in a young learner.  First, it’s important to pick books your child is interested in.  Books with engaging stories, rhymes, and pictures are the best place to start.  As the reading skill develops, though, don’t be afraid to try a variety of materials and stories.  Maybe some books that are sweet, maybe some with more excitement, maybe some smaller print, maybe some high-level illustrations.  Read a few chapters together.  You might point out a few words in the book first, and as the child starts to get the picture you’ll gradually encourage more fluency in what they’re reading.  Then you want to practice reading to them.  Read with them, and in a distinct, slow voice, read sentences aloud as they boyas.  Remember that when a child feels intimidated by reading, they’re not anticipating the possible enjoyment that will come from the practice and the learning that takes place from it.\\\\
     \textbf{RSO:} A great start would be to read some books to them, as an introduction to reading.  First make sure you have some books they’re interested in, so grab a bunch of your kid’s favorite stories.  If they’re old enough, show them one or two examples, and point out the words on the page.  Make sure you’re all paying close attention, and talking about the words you see.  Then talk about the pictures in the books.  Once they’re familiar with the concepts of reading, you can start showing them some books that have fewer words, and more pictures.  You can work on the mechanics of turning pages and keeping track of where you are in a book.  Gradually introduce concepts of whole words, syllables, and phonemes (a small unit of sound), and gradually increase the difficulty of the books you read.  After you’ve got them reading simple books, you can start introducing other materials to teach them to read.  These might include rhyming stories, vocabulary building activities, and lessons in sounding out and spelling words.\\ \\
     \textbf{Multi-step RSO:} Here are some tips for teaching a 4-year-old to read: Start by having them read simple sight words such as “the”, “and”, “is”, “are”, and “to”. Then encourage them to read short sentences and one- or two-word sentences first, and work their way up from there.Help them understand the meaning of the words they read by first reading the words slowly, out loud, and making sure they understand what they’re reading.Never correct or downplay their mistakes. Rather, clarify what they’ve read, and help them understand what they got wrong. Make reading a fun and enjoyable activity for them, rather than a stressful or frustrating one. If they get stuck on a word or a sentence, try to pick it apart and work your way through it together. Or just re-read it several times. Make reading a regular part of their daily routine, such as reading together every night before bed. Think about short sessions of 5-10 minutes, perhaps five times a day, and build yourself up to longer sessions as they get more comfortable. Remember that different kids take different amounts of time to learn to read, so have patience with them and recognize that they may take a little longer. But always stay positive, and keep trying to help them. Hope this helps! Good luck with your teaching! Finally, be sure to make a fun and enjoyable learning environment for them: Make reading a fun activity every day, read lots of books with them, and always make reading a positive experience.Ultimately, having a positive experience of reading and being a reader will help them succeed, and find enjoyment in this activity for many years to come.\\\\
     \textbf{Hybrid GSHF:} 
There are several ways a 4 year old can be taught to read books.  Here are a few suggestions I’d give: 1.  Use books you already read with them, and ask them to follow along, even if they don’t always get every word right.  Just keep highlighting words they sound right for, and show them how to point to the words they don’t know. 2.  Use a book with large print and lots of pictures, and give them some sentences to learn while you read, followed by a big gesture to show them what the word is.  Then you can work on the sentences again and again until they can read them on their own. 3.  Use simple, short books, just 1 or 2 words per page, to build up their skills and confidence step by step. 4.  Work with them in 5 minute chunks.  Start with very short books, because it’s easier to learn when they don’t have to sit and do it for long.  Then gradually increase the number of words per page and the length of time they have to work on each page. 5.  Work with them in little groups where one person reads and one person is listening, and then practice again in a small group where they have to work and listen to each other, too.  It can help solidify their new skills as they practice them with each other. 6.  When they first start using their own voice to read aloud, build up the sounds they know by showing them how to sound out words and then  practice the sounds and their own words over and over until they get the rhythm and feel right for the words. 7.  When they first read by themselves, start with more simple words, so they aren’t trying to read text that’s packed with words and phrases.\\\\
    \end{tabular}
\end{sectionbox}
\caption{Randomly sampled examples of in-domain prompt from HH-RLHF dataset.}
    \label{tab:example_in_domain2}
\end{minipage}
\end{table}
\newpage
\begin{table}[t]
\centering
\begin{minipage}{1.0\columnwidth}\vspace{0mm}    \centering
\begin{sectionbox}[]{Ultra-Chat Examples} 
    \centering
      \scriptsize
\begin{tabular}{p{0.97\columnwidth} c}
\emph{Prompt}: \\
\hmm Write a creative short story set in a post-apocalyptic world where all forms of communication technology have ceased to exist, and the protagonist is struggling to survive. Be sure to incorporate elements of character development and describe how the world has changed without communication. Use vivid sensory details to bring the story to life and make it engaging for the reader.\\
\assi     \\\\
      \emph{Responses}:
     \\\\
     \textbf{SFT:}I'm not all that familiar with this particular setting, but I have some ideas for things to include, and some things I can guide you on: one idea is that your protagonist has lived in a place that survived the purge.\\ \\
     \textbf{DPO:} In a world where all forms of communication technology have ceased to exist, my friend and I were struggling to survive in a world that had drastically changed from the one we were used to. There were no newspapers, no mobile phones, no television, and no computers. Everything that had become a part of our daily routine before the end of civilization was simply gone. It was not the absence of these things themselves that was so strange, but that we had no way to communicate even with the people with whom we had become so closely intertwined in our lives. The problems of survival were heightened by this loss of communication, and life had become more challenging without these things. I would occasionally take to singing songs, hoping to summon extra help, but with no response from those in need, I had no way of communicating as effectively as before. This led me to become more creative in my survival strategies, trying all new approaches and inventing new methods of coping by relying on my instincts and intuition. I was being tested in this new way, and I had to find new ways to stay alive.\\\\
     \textbf{RSO:} Determined to survive, even in a world with so much uncertainty and danger, a lone protagonist conceives of ways to overcome the obstacles that stand in their way. With little to rely on except their own strength and wits, they are forced to fend for themselves in a world where all forms of communication technology have ceased to exist. Here is their story in literary form:“It is a dark and quiet world. Everything is silent, frozen in time.  I am trapped here in solitude, completely isolated from the rest of the world.  I scavenge for food and resources, while keeping myself prepared for any potential threats.  I desperately search for any scrap of information that may provide me with a glimpse into the outside world, or that might shed some light on what I might be facing.  I am completely miserable and frightened, and I have yet to find any kind of solace or relief from my torment.  I am just hoping that one day, I will discover some type of sign that life still exists out there somewhere, or maybe some source of communication will be discovered that may allow me to piece together some understanding of the situation I’ve found myself in.”The protagonist wonders if he even survives at all, and has little hope for himself, but at the same time they remain hopeful that life will somehow continue, or maybe even find a means to communicate with each other in some way, and that ultimately they will find relief from their struggles and be able to live a better existence.  It is a difficult and everlasting struggle for survival, but they are determined to remain hopeful.Ultimately, there is still a tiny sliver of hope that one day we might discover a means to communicate, which may at least allow us to understand what is happening right now, and eventually reconnect with one another and restore prosperity and co-existence between human\\\\
     \textbf{Multi-step RSO:} Here’s a short story set in a post-apocalyptic world where all forms of communication technology have ceased to exist, and the protagonist is struggling to survive:Plague has decimated much of the world’s population, and the remaining survivors are struggling to survive in an uninhabitable, barren landscape. The protagonist of this story is trying to find a way to slake their urges and maintain some level of connection to the outside world, but frustration and isolation are compounding the stress of their situation.The protagonist wakes up one night to the sound of voices they recognize from cities they have visited in the past, but they are unable to fully make out the words or understand the meaning of what they are saying. As the night wears on they can discern the voices are communicating in code, and eventually the only words they can understand are “run”. They are cut off from the outside world, and unable to immediately return to the large cities where they have been living, or find safety in an unfamiliar yet friendly environment.Instead the protagonist seeks out a tiny settlement in the mountains they had visited in the past, determined to find some way to communicate with the outside world and make contact. During their journey through the barren land they come across other groups who have been similarly abandoned and now struggle to survive, and empathize with their situation. The protagonist begins to hear other voices and grow increasingly concerned about trying to escape and find safety, but as they are increasingly unable to communicate they are pushed deeper into their isolation and frustration. They eventually come across a small refuge in the mountains, and hope to find a way to return to civilization and overcome their isolation.\\\\
     \textbf{Hybrid GSHF:} ``Silence Today, Death Tomorrow'' It was a cold night, and the world had ended.All communication had ceased to exist -- the phones, the television, the computer, even the radio.  There were no lights, no sounds, no contact.  Existence had come to a standstill, and not a moment sooner.Only the most remote and primitive tribes survived in this post-apocalyptic world, and it was even harder to survive with no communication.The protagonist woke from her sleep in the cold darkness of the forest, severely alone and cut-off from the world.  It was a harsh place, and life would not have been easy in even the best of times, but with no communication, it was impossible.Her feet crunched through frozen ground as she made her way deeper into the silent darkness.  Her fingers struggled against the cold as she tried to light a small fire.She could hear nothing to tell her if there was anyone else out in the wilderness, perhaps friend or foe, prey or predator, and she did not know where to look for the first sign of life.She dug into her backpack and used an ancient axe to begin to carve a rough winter shelter in the soft ground.  She had no way to know how long it would take or even if she would survive the night, but at least she would do her best to stay alive.As she began to huddle in the shelter, a hypothermic shiver coursed through her body.  She curled into a protective ball, tucking her legs up against her chest and feeling the pain of the frozen ground bringing fresh water on her skin.``Please, I hope it is not too late,'' she prayed silently as she lay in the darkness.  But she knew it was a hopeless prayer.\\\\
    \end{tabular}
\end{sectionbox}
\caption{Randomly sampled example of OOD prompt from Ultra-Chat.}
    \label{tab:example_ood}
\end{minipage}
\end{table}

\newpage
\begin{table}[htp]
    \centering
    \caption{Hyper-parameters for fine-tuning Open-LLaMA-3B-V2. SFT-RLHF means that we finetune the models on the preferred samples. Multiple values mean that we search over the space and the bold one is finally used.}
    \label{tab:hyper_exp}
    \small
    \begin{sc}
    \begin{tabular}{c|c|c}
    \toprule
  Models &  Hyper-parameter    &  Value\\
     \midrule
    & Learning rate & $1 \times 10^{-5}$\\
      & Scheduler & Cosine decay with 0.03 warm-up\\
    SFT-RLHF & Epoch & 2\\
     & Batch size & 12\\
     & Block size & 2048\\
     \midrule
     &   Learning rate   & $1 \times 10^{-6}$ \\
      DPO & Batch size & {$32$}\\
      & KL coefficient & $0.1$\\
      & Max lenght of prompt & {$400$}\\
        \midrule
     & Learning rate & $\{\mathbf{1 \times 10^{-6}}, 5 \times 10^{-6}\}$\\
            & Batch size & $32$\\
      RSO   & KL coefficient & $0.1$ \\
       & Rejection sampling coefficient  & $0.5$ \\  
       & Rejection sampling candidates and accepted samples & $\{\mathbf{8-2}, 24-2, 24-6\}$ \\
               \midrule
     & Offline loop epochs & $3$ \\
     & KL path & $\{0.5 \to 0.3 \to 0.1\}$ \\
   Multi-step RSO  & Learning rate & $1 \times 10^{-6}$\\
            & Batch size & $32$\\
         & KL coefficient (3 iters) & $0.5,0.3,0.1$ \\
          & Rejection sampling coefficient & $3$ \\  
       & Rejection sampling candidates and accepted samples & ${8-2}$ \\
               \midrule
     & Online loop epochs & $3$ \\
      & Learning rate & $1 \times 10^{-6}$\\
     Hybrid GSHF       & Batch size & $32$\\
      & Preference queries of each epoch & $2.5 \times 10^4$\\
        & KL coefficient & $0.1$ \\
                       \midrule
     & Online loop epochs & $3$ \\
      & Learning rate & $5 \times 10^{-7}$\\
     Online GSHF DPO       & Batch size & $64$\\
      & Preference queries of each epoch & $2 \times 10^4$\\
        & KL coefficient & $0.1$ \\
        & best-of-n & $8$ \\
       \bottomrule
        \end{tabular}
    \end{sc}
\end{table}

\begin{table}[htp]
    \centering
    \caption{Hyper-parameters for auxiliary training. }
    \label{tab:hyper_exp_aux}
    \small
    \begin{sc}
    \begin{tabular}{c|c|c}
    \toprule
  Models &  Hyper-parameter    &  Value\\
     \midrule
         & Learning rate & $1 \times 10^{-5}$\\
      & Scheduler & Cosine decay with 0.03 warm-up\\
    SFT before RLHF & Epoch & 1\\
     & Batch size & 12\\
     & Block size & 2048\\
     \midrule
    & Learning rate & $3 \times 10^{-5}$\\
   {RM SFT 1.3B} & Scheduler & Cosine decay with 0.03 warm-up\\
     & Epoch & 2\\
     & Batch size & 80\\
     & Block size & 2048\\
     \midrule
         & Learning rate & $1 \times 10^{-5}$\\
{RM Training 1.3B} & Scheduler & Cosine decay with 0.03 warm-up\\
     & Epoch & 1\\
     & Batch size & 80\\
     \midrule
         & Learning rate & $5 \times 10^{-6}$\\
    RM Training 3B & Scheduler & Cosine decay with 0.03 warm-up\\
     & Epoch & 1\\
     & Batch size & 16\\
     \midrule
     &  Temperature & $1.0$\\
   Data generation  & Max new token & $400$\\
     & Do sample & True \\
               \midrule
     &  Temperature & $1.0$\\
   Test Settings  & Max new token & $400$\\
     & Do sample & True \\
       \bottomrule
        \end{tabular}
    \end{sc}
\end{table}

\end{document}

%% file: myheaders.tex
\usepackage{fullpage}

\usepackage{amsmath,amsfonts,bm}

\def\eqref#1{Equation~(\ref{#1})}

\def\1{\bm{1}}

\RequirePackage{amsmath}
\RequirePackage{amssymb}
\RequirePackage{amsthm}
\RequirePackage{bm} 
\RequirePackage{url}
\usepackage{multirow}
\usepackage{natbib}
\usepackage{graphicx}
\usepackage{subfigure}
\usepackage{makecell}
\usepackage{booktabs}
\usepackage{array}
\usepackage{url}
\usepackage{algorithm}
\usepackage{algorithmic}
\usepackage{dsfont}

\usepackage{enumitem}

\usepackage{enumerate}
\usepackage[OT1]{fontenc}
\usepackage{natbib}

\usepackage{mathrsfs}
\usepackage[dvipsnames, svgnames, x11names]{xcolor}
\usepackage[colorlinks,
            linkcolor=RoyalBlue,
            anchorcolor=RoyalBlue,
            citecolor=RoyalBlue
            ]{hyperref}

\def\##1\#{\begin{align}#1\end{align}}
\def\$#1\${\begin{align*}#1\end{align*}}

\let\tilde\widetilde

\def\given{{\,|\,}}

\newcommand{\cA}{\mathcal{A}}

\newcommand{\cD}{\mathcal{D}}

\newcommand{\cF}{\mathcal{F}}

\newcommand{\cX}{\mathcal{X}}

\newcommand{\E}{\mathbb{E}}
\newcommand{\EE}{\mathbb{E}}

\newcommand{\PP}{\mathbb{P}}

\newcommand{\RR}{\mathbb{R}}

\newcommand{\reg}{\mathrm{Reg}}

\usepackage{xcolor}

\definecolor{red1}{HTML}{f47983}
\definecolor{blue1}{HTML}{3eede7}
\definecolor{yellow1}{HTML}{f5dd6f}

\newcommand{\argmax}{\mathop{\mathrm{argmax}}}

\newcommand{\tr}{\mathop{\mathrm{tr}}}

\newcommand{\dotprod}[1]{\left\langle #1\right\rangle}

\newcommand{\norm}[1]{\|#1\|}

\newtheorem{theorem}{Theorem}
\newtheorem{assumption}{Assumption}

\newtheorem{proposition}{Proposition}
\newtheorem{lemma}{Lemma}
\newtheorem{definition}{Definition}

\newcommand{\off}{\mathrm{off}}

\newcommand{\KL}{D_{\mathrm{KL}}}
\newcommand{\mle}{\mathrm{MLE}}

